\documentclass[pdflatex,sn-apacite]{sn-jnl}%

\usepackage{svg}
\usepackage{graphicx}%
\usepackage{multirow}%
\usepackage{amsmath,amssymb,amsfonts}%
\usepackage{amsthm}%
\usepackage{mathrsfs}%
\usepackage[title]{appendix}%
\usepackage{xcolor}%
\usepackage{textcomp}%
\usepackage{manyfoot}%
\usepackage{booktabs}%
\usepackage{algorithm}%
\usepackage{algorithmicx}%
\usepackage{algpseudocode}%
\usepackage{listings}%

\usepackage{hyperref}
\usepackage{url}
\usepackage{microtype}
\usepackage{graphicx}
\usepackage{subcaption}
\usepackage{booktabs} %
\usepackage[textsize=tiny]{todonotes}

\usepackage{multirow}

\usepackage[utf8]{inputenc}

\usepackage{amsfonts}
\usepackage{amssymb}
\usepackage{amsmath}
\usepackage{amsthm}
\usepackage{enumerate}
\usepackage{tikz}
\usepackage{booktabs}
\usepackage{nicefrac}       %
\usepackage{makecell}
\usepackage{array, multirow, boldline}
\usepackage{bm}

\def\eqref#1{equation~\ref{#1}}

\def\ceil#1{\lceil #1 \rceil}
\def\floor#1{\lfloor #1 \rfloor}
\def\1{\bm{1}}

\def\vone{{\bm{1}}}

\DeclareMathAlphabet{\mathsfit}{\encodingdefault}{\sfdefault}{m}{sl}
\SetMathAlphabet{\mathsfit}{bold}{\encodingdefault}{\sfdefault}{bx}{n}

\def\gD{{\mathcal{D}}}

\def\gN{{\mathcal{N}}}

\def\gX{{\mathcal{X}}}

\newcommand{\E}{\mathbb{E}}

\newcommand{\R}{\mathbb{R}}

\newcommand{\norm}[1]{\lvert\lvert #1\rvert\rvert}

\newcommand{\abs}[1]{\lvert #1 \rvert}
\newcommand{\ip}[2]{\langle #1,#2\rangle}

\usepackage{natbib}
\usepackage{hyperref}
\usepackage{cleveref}

\newcommand*\numcircledtikz[1]{\tikz[baseline=(char.base)]{
            \node[shape=circle,draw,inner sep=1.2pt] (char) {#1};}} 
\usepackage{subcaption}
\usepackage[export]{adjustbox}
\newcommand{\blue}[1]{{\color{black}#1}}
\usepackage{tabularx}

\usepackage{orcidlink}

\theoremstyle{thmstyleone}%
\newtheorem{theorem}{Theorem}%
\newtheorem{lemma}[theorem]{Lemma}
\newtheorem{assumption}[theorem]{Assumption}

\newtheorem{proposition}[theorem]{Proposition}%

\theoremstyle{thmstyletwo}%

\theoremstyle{thmstylethree}%
\newtheorem{definition}{Definition}%

\title[Selective Classification with  Pairwise Queries]{Improving Selective Classification with  Pairwise Queries for Binary Classification}

\author*[1]{\fnm{Harsh} \sur{Vardhan}\orcidlink{https://orcid.org/0000-0002-4656-3162}}\email{hharshvardhan@ucsd.edu}

\author[2]{\fnm{Sunav} \sur{Choudhary}\orcidlink{0000-0002-7711-487X}}\email{schoudha@adobe.com}

\author[3]{\fnm{Natwar} \sur{Modani}\orcidlink{https://orcid.org/0009-0005-3105-1016}}\email{nmodani@adobe.com}

\author[4]{\fnm{Arya} \sur{Mazumdar}\orcidlink{https://orcid.org/0000-0003-4605-7996}}\email{arya@ucsd.edu}

\affil*[1]{\orgdiv{CSE}, \orgname{UCSD}, \orgaddress{\street{9500 Gilman Drive}, \city{La Jolla}, \postcode{92092}, \state{CA}, \country{USA}}}

\affil[2]{\orgname{Adobe Research}, \orgaddress{\street{345 Park Ave}, \city{San Jose}, \postcode{95110}, \state{CA}, \country{USA}}}

\affil[3]{\orgname{Adobe Research}, \orgaddress{\street{Marathahalli}, \city{Bengaluru}, \postcode{560087}, \state{Karnataka}, \country{India}}}
\affil[4]{\orgdiv{HDSI}, \orgname{UCSD}, \orgaddress{\street{9500 Gilman Drive}, \city{La Jolla}, \postcode{92092}, \state{CA}, \country{USA}}}

\begin{document}

\abstract{
In selective classification, a model predicts the labels of data samples where it is confident, and abstains from predicting labels for samples on which it is not confident. The rejected samples are often labeled by an expert, which is expensive. The budget for the expert is best utilized when the model has low error on non-rejected samples. However, the estimate of a model's confidence might be inconsistent with the model's predictions, which can lead to high error on non-rejected points. Such situations can readily occur in in-context binary classification by LLMs. To remedy this, we propose making additional pairwise queries to the same model. These pairwise queries can detect high-error samples and be incorporated into selective classification techniques to reduce the error on non-rejected samples. Theoretically, we establish the conditions under which a simple algorithm using pairwise queries outperforms an inconsistent confidence estimate. We support this insight through extensive experiments for $1$ synthetic and $4$ in-context learning-based real binary classification datasets. In all these cases, we show that our algorithms, using pairwise queries, obtain a better accuracy-cost tradeoff than using only the raw confidence estimates, for instance, the LLM's next-token logits.}


\maketitle

\section{Introduction}
\label{sec:intro}
In selective classification~\citep{chow70,el-yaniv10a}, a classifier is augmented with a rejection function such that it only predicts labels on the datapoints that are not rejected. There are two complementary goals for selective classification: 1) rejecting a small fraction of samples (high coverage), 2) obtaining low error on samples that are not rejected (low selective risk). Balancing both these goals rules out trivial rejection functions that either reject everything or reject nothing. If a fixed fraction of datapoints is rejected, the best selective classification algorithm should obtain the smallest possible selective risk. This requires identifying samples where the classifier makes an error. Without the knowledge of the true labels, these samples cannot be identified. However, a viable approach is to reject datapoints on which the classifier is not confident in its prediction. This approach has been widely adopted, for instance, using softmax confidences~\citep{geifman17}, or the uncertainty of an ensemble's predictions~\citep{pmlr-v48-gal16}.  Such high-confidence or low-uncertainty predictions are also desired in healthcare~\citep{CRAIG202366}, fairness~\citep{justice, fairness} and finance~\citep{DASTILE2020106263}.

If the measure of confidence is not correlated with the predicted labels, then selective classification based on confidence fails. We find that this is often the case in a recent but growing field of work: {\em in-context learning} (ICL) with {\em large language models} (LLMs) (see Fig~\ref{fig:mot_ex}). Consider the task of verifying the correctness of a Natural Language to SQL (NL2SQL)  output by ICL on an LLM~\citep{liu2024surveynl2sqllargelanguage,numbersstation2023NSText2SQL,yu-etal-2018-spider}.

 \begin{figure}[H]
 \centering
 \begin{minipage}[htbp]{0.45\textwidth}
 \centering
 
    \includegraphics[width=0.9\linewidth]{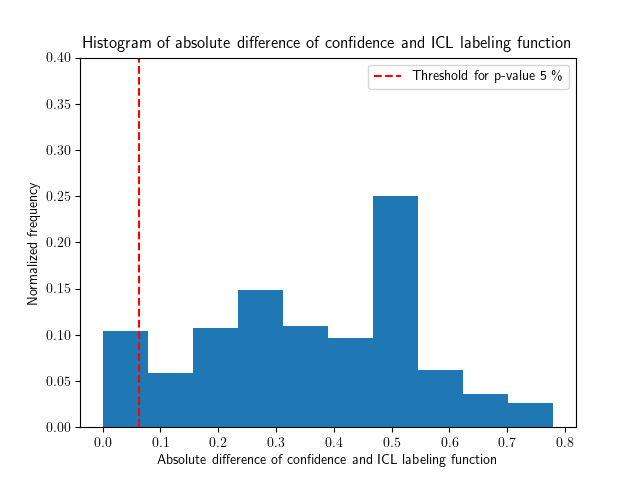}
    \caption{\footnotesize Distribution of absolute difference of confidence from next token logits and ICL labeling function for verifying the correctness of NL2SQL, where the NL2SQL are from the Bird dataset~\citep{li2024can} and ICL labels and next token logits are from Gemma3 4B Instruct. The threshold with $p$-value of $5\%$ is for the hypothesis that confidence and ICL labeling function are close, when averaged over $50$ random seeds.}
    \label{fig:mot_ex}
 \end{minipage}%
 \hfill
\begin{minipage}{.5\textwidth}
In NL2SQL, each datapoint consists of a tuple of a natural language question for a database and a corresponding SQL on the database. If the SQL can answer the natural language question, then the label is $1$, and it is $0$ otherwise. NL2SQL is an important use-case of LLMs in business applications~\citep{liu2024surveynl2sqllargelanguage,numbersstation2023NSText2SQL,zhang2023llmaaa}. Using selective classification with in-context learning on an LLM, we can reject the  NL2SQL pairs on which the classifier is least confident in its predictions, and delegate them to an expensive human annotator.  A natural confidence measure in this case is the classifier's next token logits for the predicted label~\citep{geifman17}.

\end{minipage}
\end{figure}%

Note that verifying correctness is a binary classification problem, so the confidence and bias of the ICL labeling function from in-context learning are both in $[0,1]$ for any (NL,SQL) pair. In Fig~\ref{fig:mot_ex}, we plot the distribution of the absolute difference between these two functions for an NL2SQL dataset~\cite{li2024can}. This distribution has a heavy tail for values close to 0.5, implying very low correlation between them. In  Fig~\ref{fig:mot_ex}, we see a peak when the confidence and ICL labeling bias differ by 0.5. In this scenario, selective classification will waste expensive human annotation by sending datapoints where the predicted labels are already correct.

We address this problem in this paper: {\em selective classification when the confidence function differs from the actual labeling function of the classifier.} 

This problem has not been explored in existing works on selective classification~\citep{el-yaniv10a, geifman17, pugnana2024deepneuralnetworkbenchmarks}, but, their focus has been on simpler calibrated classifiers, and not on ICL. Additionally,  due to prohibitive sizes of the classifier, which for LLMs might be billions of parameters~\citep{llama3modelcard}, training or retraining the model might be impossible, although efficient fine-tuning might be possible. Alternatively, the generated outputs in the previous example might be obtained from a closed-source model with restrictive conditions on their usage. For instance, OpenAI's models~\citep{openai2024gpt4} prevent the use of its generated outputs in any training step. Therefore, we need selective classification methods that do not require any training. If we cannot train, several existing techniques used both in theory~\citep{el-yaniv10a,el-yaniv12a} and practice~\citep{pugnana2024deepneuralnetworkbenchmarks}, including the learning-to-reject framework~\citep{corbiere19,pmlr-v97-geifman19a,cortes}, cannot be applied. In addition to this, ICL uses only a small number of labeled samples, which disqualifies recent model-agnostic methods~\citep{pmlr-v206-pugnana23a}.

We propose a solution circumventing all the above-mentioned problems: a pairwise query. In binary classification, a pairwise query would correspond to sending two unlabeled datapoints to the model and asking which of their labels is closer to the label $1$.

 For the NL2SQL example, where the label $1$ corresponds to correctness, it would translate to asking which of the two pairs of natural language questions and corresponding SQLs is more correct. This pairwise information is provided by the model directly, and it turns out to be more accurate than an unreliable confidence value. This intuitively motivates selective classification with pairwise queries. Additionally, 1) pairwise queries are cheaper than the expert annotator, 2) several models are amenable to pairwise queries, especially the LLM classifiers we discussed for the NL2SQL case and 3) LLMs are more accurate on pairwise queries than direct labels on certain language tasks~\citep{qin-etal-2024-large}. The goal of this paper is to understand \textit{if pairwise queries can improve selective classification over thresholding an unreliable confidence estimate}.
We answer this in the affirmative.

\subsection{Contributions}
Our main contributions are listed below.

    \textbf{1. Pairwise Query-based Selective Classification:} For binary classification, we propose $5$ different selective classification algorithms (PairSel-Middle, PairSel-Max-Entropy, PairSel-Max-Presence, PairSel-Max-Displacement, PairSel-kNN) that do not require retraining and use a sorting subroutine based on pairwise queries. Of these, only PairSel-kNN requires a distance metric on the feature space of the datapoints.

 \textbf{2. Theoretical Improvement:} Under a mild assumption on the model's predictions (Assumption~\ref{assump:lin_model} in Section~\ref{sec:theory}) that is satisfied by linear, generalized linear, and neural network models in the NTK regime~\citep{jacot_gabriel_hongler_2018}, we derive the necessary theoretical conditions under which the simplest of our pairwise algorithms, Pair-Middle, outperforms vanilla confidence thresholding for spherical (Theorem~\ref{thm:spherical}) and Gaussian features (Theorem~\ref{thm:gauss}).  
    
     \textbf{3. Empirical Performance:} We consider $3$ tasks to test our techniques -- i) Synthetic dataset with linear classifier matching our theory, ii) Verification of NL2SQL (Spider~\citep{yu-etal-2018-spider} and Bird~\citep{li2024can}), iii) Binary question answering on the BoolQ~\citep{clark2019boolq} and VisOnlyQA (vision)~\citep{kamoi2025visonlyqa} datasets. For all the datasets except synthetic, we use several open-source LLMs of size $<10B$ parameters as classifiers.  We show that our pairwise algorithms obtain better performance than 3 baselines that don't use training --confidence thresholding on the next-token logits, using the base model confidences for an instruct model, and post-processing confidence calibration~\citep{han2023prototypical}, for most values of coverage (see Figures~\ref{fig:cov_all} and Table~\ref{tab:tot_acc}) or the labeling cost (see Figure~\ref{fig:cost_all}). In particular, from Table~\ref{tab:tot_acc}, the best pairwise methods have total accuracy atleast $8\%$ higher than the raw baseline, which corresponds to next-token logits (Synthetic, Bird and BoolQ), while post-processing calibration doesn't always improve over the raw baseline (Synthetic and VisOnlyQA in Table), and base models are not always available (VisOnlyQA and Llama3 SQLCoder).

\subsection{Related Works}
\label{sec:related_works}
\paragraph{Selective Classification}
~\cite{chow70} introduced selective classification initially as a theoretical paradigm. Subsequent work by ~\cite{grandvalet_08,bartlett08a} focused on obtaining rejection rules for SVMs by rejecting datapoints with small margins. Fundamental theoretical advancements in selective classification were obtained by ~\cite{el-yaniv10a} and ~\cite{wiener11}, however, their proposed selective classification algorithms were not efficient in practice. This inefficiency is due to subroutines in their algorithms that need to perform an exhaustive search on the space of all possible weights.

A more recent line of work, focusing on the practical applications of selective classification, utilized different measures of uncertainty in the model's predictions. ~\cite{geifman17} utilize the softmax confidences while ~\cite{pmlr-v48-gal16} use the uncertainty of Bayesian NNs via MC dropout as a rejection function. Other formulations of uncertainty estimation, for instance, using ensembles of models, have also been used to create rejection functions~\citep{lakshminaryanan17}. While these work in practice, only a few concrete theoretical results~\citep{pmlr-v206-pugnana23a,ding_ensembles} beyond the bayesian case are available. Several works aim to learn this rejection function from data~\citep{Cortes2023TheoryAA,corbiere19,pmlr-v97-geifman19a}, by including the rejection rate as an additional class or incorporating additional layers to predict the confidence score. These fall into the learning to reject/abstain framework. ~\citep{pugnana2024deepneuralnetworkbenchmarks} provide a comprehensive benchmark of several selective classification methods on deep learning models, showing that softmax confidences and ensemble based techniques outperform most baselines in most metrics. Note that none of the approaches beyond the bayesian case can be used in our setting, due to lack of model training.  ~\cite{ding_ensembles} requires a diverse ensemble of verifiers, which we don't have and is hard to simulate from a single verifier. Further, the \blue{model-agnostic} approaches of ~\cite[Algorithms~1,2]{pmlr-v206-pugnana23a} require a significant amount of labeled data to run, which we don't have, as we only use a small batch $20$ samples for in-context examples, which is less than $0.6\%$ of the size of our smallest dataset. Additionally, ~\cite{vishwakarma} propose a confidence-threshold based autolabeling procedure that outperforms iterative selective classification, however, it requires training. \blue{Another related field is that of semi-supervised learning(SSL)~\citep{vanEngelen2019}, where a combination of human-labeled and unlabeled features is available, and the goal is to label the unlabeled features accurately. However, SSL also requires training, and there are no labeled datapoints in our setup. So, SSL techiques cannot be applied directly to our problem. However, certain subroutines of SSL, especially label propagation~\citep{vanEngelen2019} can still be applied to propagate information from the classifier's labels. We use this in one of our algorithms in Section~\ref{sec:algorithms}.}

\blue{\paragraph{Active Learning and Selective Classification}
A closely related problem to selective classification is active learning. Similar to selective classification, in active learning, we also have an unlabeled dataset of features, and we need to select the best features to send for human labeling. However, contrary to selective classification, the goal after human labeling is to learn a new classifier using these labeled datapoints. Further, in active learning, this cycle of selecting the best unlabeled datapoints, human labeling, and subsequent re-training of classifier is repeated several times, until a budget of human labels or a desired accuracy of the classifier is reached. Thus, selective classification is a training-free version of only a single step of active learning~\citep{el-yaniv12a}. ~\cite{el-yaniv12a, gelbart2019} show strong theoretical connections between theoretically optimal active learning and selective classfiication. ~\cite{Hanneke2014TheoryOD} provide a comprehensive overview of theoretically optimal active learning algorithms, ~\cite{ren_survey_2021} provide an overview of deep-learning based active learning algorithms, ~\cite{zhang-etal-2022-survey} cover recent applications of active learning to NLP. The most relevant survey of active learning focused on LLMs is provided in ~\citep{xia-etal-2025-selection}.

Due to their similarities, one may want to adapt techniques from active learning to selective classification. Note that active learning necessitates training, which is computationally expensive in our case of an LLM as a classifier, thus disqualifying all active learning techniques, including the cost-sensitive variants~\citep{ijcai2017p261,Settles2008ActiveLW}. However, we can adapt active learning techniques used to select samples for human labeling. From ~\cite[Section~3.2]{xia-etal-2025-selection}, SoTA techniques for selecting most informative samples use the LLM logits, or ask the LLM to manually rank the difficulty of an example, or use consistency of predicted labels among a committee of LLMs. Note that a committee is only useful~\citep{Settles2012}, if the members are sufficiently different, which is hard to ensure with a single classifier. Manually ranking a large dataset ($\approx1000$ samples) of NL2SQL queries, that are each atleast a few $1000$ tokens already leads to a context length of a few million tokens, which is only achieved by SoTA models~\citep{openai2024gpt4,gemmateam2025gemma3technicalreport,bai2025qwen25vltechnicalreport}. Our pairwise queries are a simpler and more token-efficient technique to rank the difficulty of LLM examples. }

\paragraph{In-context learning and pairwise queries} To obtain predictions from the LLM verifier, we use in-context learning~\citep{dong2024surveyincontextlearning}, where a few labeled examples are provided as context along with an unlabeled feature, and the LLM predicts the label for this unlabeled feature. There are existing methods to calibrate LLM confidences -- post-processing of logits~\citep{han2023prototypical} and prompt-sensitivity approaches are utilized~\citep{chen-etal-2023-relation}. We found that dropout-based prompt-sensitivity for our problems performs much worse than calibration for all our datasets, so we do not report its results, and post-processing calibration is used as a baseline in our experiments.  ~\citep{openai2024gpt4} point out that base models are well calibrated, while instruct models are poorly calibrated. We utilize this to obtain another baseline. Pairwise queries form a key subroutine of our proposed methods, and existing works show that LLMs are better at pairwise predictions than direct labeling~\citep{qin-etal-2024-large}. \blue{Note that pairwise queries have been used in theoretically optimal active learning algorithms~\citep{pmlr-v134-hopkins21a, coarse_labels, pmlr-v125-hopkins20a, kane2017active}. However, these algorithms cannot be implemented efficiently for our case. For instance, ~\cite{kane2017active,pmlr-v125-hopkins20a, pmlr-v134-hopkins21a} require the knowledge of the inference dimension of the problem to implement their optimal active learning algorithm. Computing the inference dimension is at least as hard as computing the VC dimension of a problem. For our LLM models, we can either use an extremely loose bound on this dimension, which makes the actual active learning algorithm suboptimal, or we estimate it accurately, which itself is computationally intractable. Further, ~\cite[Algorithm~1]{pmlr-v125-hopkins20a} requires solving a linear program, and ~\cite{pmlr-v134-hopkins21a, kane2017active} have efficient algorithms only for the case of learning halfspaces.}

\textbf{Organization} In Section~\ref{sec:setup}, we formally define our problem, and in Section~\ref{sec:algorithms}, we describe the baseline confidence threshold algorithm and pairwise query-based algorithms. In Section~\ref{sec:theory}, we establish theoretical conditions under which pairwise queries are better than using confidence, and in Section~\ref{sec:exp}, we justify this insight via extensive experiments.  

\section{Problem Setup}
\label{sec:setup}
\textbf{Dataset.} We assume that we have an unlabeled dataset $\{x_i\}_{i=1}^n$ of $n$ datapoints. Here, each $x_i\in \gX$ is a feature, obtained from the feature space $\gX$ and sampled iid from the feature distribution $\gD_x$, i.e., $x_i\overset{iid}{\sim}\gD_x$.  Throughout this paper, we consider the task of binary classification. So, we want to assign a label $y_i\in \{0,1\}$ to each datapoint $x_i$ where $i\in [n]$. We use $[n]$ to denote the set $\{1,2,\ldots, n\}.$

\paragraph{Labeling Function}
For binary classification, we can model any classifier as a labeling function $p:\gX \to [0,1]$ such that for a given feature $x\in \gX$, the label $y$ is generated from the  probabilistic model $y\sim \mathrm{Ber}(p(x))$. We use this modeling to represent the in-context labeling function $p_{\mathrm{clf}, l}$, the confidence function $p_{\mathrm{clf}, c}$, the human labeler $p_h$ as well as the true labels $p^\star$.

\textbf{Classifier.}
Our classifier can be represented by two functions $p_{\mathrm{clf}, l}, p_{\mathrm{clf}, c}:\gX \to [0,1]$. Here, $\mathrm{clf}$ is used to denote the classifier and $l$ and $c$ are used to denote the labeling and confidence functions. For any given datapoint $x_i,\, i\in [n]$, the classifier generates it's predicted label $\hat{y}_i$ as $\hat{y}_i\sim \mathrm{Ber}(p_{\mathrm{clf}, l}(x))$. The classifier's confidence in the label $\hat{y}_i$ is represented by $p_{\mathrm{clf}, c}(x_i)$. We assume that the labeling function and the confidence function are different $p_{\mathrm{clf}, c} \neq p_{\mathrm{clf}, l}$.
This is the weakest notion of inconsistency and we will use stronger ones in our theoretical results (Theorems~\ref{thm:spherical} and \ref{thm:gauss}).
For the example of in-context learning, $\hat{y}_i$ corresponds to the in-context label and $p_{\mathrm{clf}, c}(x_i)$ corresponds to logits-based confidence. Additionally, we would never have access to the actual functions $p_{\mathrm{clf}, l}, p_{\mathrm{clf}, c}$, apart from the labels and the confidence generated by them.

\textbf{Annotators.} The classifier can reject a subset of the features. These are sent to the expert (also referred to as  human), parameterized by its labeling function $p_{h}:\gX\to [0,1]$ such that the expert's label is $\tilde{y}_i \sim \mathrm{Ber}(p_{h}(x_i)), \forall i\in [n]$. The true labels for the problem are generated from a true labeling function $p^\star :\gX\to [0,1]$, such that $y_i^\star \sim \mathrm{Ber}(p^\star(x_i)), \forall i\in [n]$.

\paragraph{Selective Risk} Our main metric will be the fraction of points misclassified, which we will call the empirical risk or the empirical misclassification. Using only the classifier, the empirical misclassificaiton is  $f_{\mathrm{clf}}(n) \triangleq \frac{1}{n}\sum_{i=1}^n \vone\{\hat{y}_i\neq y_i^\star\}$. For the expert/human, this is $f_h(n)\triangleq \frac{1}{n}\sum_{i=1}^n\vone\{\tilde{y}_i \neq y_i^\star\}$. For all our selective classification methods, we will fix the fraction of datapoints to be rejected as $\alpha \in \{0,\frac{1}{n}, \frac{2}{n}, \ldots, 1\}$. Each selective classification method selects a subset of datapoints $S_{\alpha}\subseteq [n]$, that are sent to the expert, where $\abs{S_{\alpha}} = \alpha n$. To measure the performance of any selective classification method, we define the selective risk as 
\begin{align*}
    f_{S_{\alpha}}(n) = \frac{1}{n}\sum_{i=1}^n (\vone\{\hat{y}_i \neq y_i^\star, i\notin S_{\alpha}\}+ \vone\{\tilde{y}_i \neq y_i^\star, i\in S_{\alpha}\})
\end{align*}
Note that for all the datapoints that have been selected, i.e. $i\in S_{\alpha}$, the predicted label is the expert label $\tilde{y}_i$, while for the rest of the datapoints, $i\notin S_{\alpha}$, it is the classifier label $\hat{y}_i$. Further, we can define the population risks of the classifier, human and any selective classification method as the expectation of the corresponding empirical risks with respect to the labeling distribution and the feature distribution. Let $F_{\mathrm{clf}} \triangleq \E_{x_i,\hat{y}_i,y^\star_i}[f_{\mathrm{clf}}(n)]$ and $F_h \triangleq \E_{x_i,\tilde{y}_i, y^\star_i}[f_{h}(n)]$ denote population risk for classifier and human respectively.

\paragraph{Pairwise Queries.} For two datapoints $x_i,x_j\in \gX$, where $x_i\neq x_j$, the pairwise query to the classifier measures the difference {$p_{\mathrm{clf},l}(x_i) - p_{\mathrm{clf},l}(x_j)$.} %
Conditioned on the correct labels $y_i^\star,y_j^\star =1$, if the difference is large and positive, then $x_i$ has a higher bias and is thus more correct than $x_j$. In this case, the pairwise query returns $1$. If the difference is large and negative, then $x_j$ is more correct and the pairwise query returns $0$.  We provide a more precise definition of pairwise queries for our theoretical results in Definition~\ref{assump:pair}.

For any selective classification method, our goal is to minimize its corresponding selective risk $E_{S_{\alpha}, x_i, \hat{y}_i, \tilde{y}_i, y_i^\star}[f_{S_{\alpha}}(n)]$ for a fixed $\alpha$. As the corresponding empirical errors are sum of indicator random variables, by Chernoff bound, they differ from their population counterparts only by $\mathcal{O}(1/\sqrt{n})$ with high probability.

\section{Proposed Algorithms}
\label{sec:algorithms}
In this section, we describe all our proposed selective classification algorithms. This includes the existing confidence-based thresholding baseline, which we call ConfSel, and the family of methods using pairwise queries, which we call PairSel.

\paragraph{ConfSel}
First, we describe the baseline algorithm, ConfSel, in Algorithm~\ref{alg:conf_sel}, that rejects the least confident datapoints, computed using $p_{\mathrm{clf},c}$. For binary classification, if $p_{\mathrm{clf},c}(x_i)$ for some $i\in [n]$, is close to $1$, then the classifier is very confident in it's label being $1$. However, if $p_{\mathrm{clf}, c}(x_i)$ is close to $0$, the classifier is confident in the label being $0$. Therefore, $\abs{p_{\mathrm{clf},c}(x_i) - \frac{1}{2}}$ is a label-independent measure of confidence in the classifier's predictions. For $p_{\mathrm{clf},c}(x_i)$ close to either $0$ or $1$, our confidence measure achieves it's maximum value of $\frac{1}{2}$. If $p_{\mathrm{clf}, c}(x_i)$ is close to $\frac{1}{2}$, then the confidence measure is close to $0$, it's minimum value. The algorithm ConfSel selects datapoints that have the smallest confidence measure in the dataset.

\begin{algorithm}[htbp]
\caption{ConfSel}
\label{alg:conf_sel}
    \begin{algorithmic}
        \Require Set of datapoints $\{x_i\}_{i=1}^n$, Fraction of datapoints to send to human $\alpha\in [0,1]$
        \State Compute $p_{\mathrm{clf},c}(x_i)$ for each datapoint $x_i, \,i\in [n]$.
        \State Sort the datapoints according to $\abs{p_{\mathrm{clf},c}(x_i) - \frac{1}{2}}$.
        \State Let $S_{\alpha, c}$ be the set of datapoints in the smallest $\alpha$-quantile according to $\abs{p_{\mathrm{clf},c}(x_i) - \frac{1}{2}}$.
        \State {\bf Return} $S_{\alpha, c}$
    \end{algorithmic}
\end{algorithm}

\begin{algorithm}[t!]
{
\caption{PairSel}
\label{alg:pair_sel}
    \begin{algorithmic}
        \Require Set of data points $\{x_i\}_{i=1}^n$, Fraction of datapoints sent to human $\alpha\in [0,1]$.
        \State Sort the datapoints $\{x_i\}_{i=1}^n$ in descending order of $p_{\mathrm{clf},l}$ using Mergesort with $\mathcal{O}(n\log n)$ pairwise queries.
        \State Let $\Pi:[n]\to[n]$ be the permutation which sorts the datapoints and $\Pi^{-1}:[n] \to [n]$ be its inverse permutation.
        \State $S_{\alpha, p} \gets$ PairSel$-\star()$         \quad \# Here, $\star\in \{$Middle, Max-Entropy, Max-Presence, Max-Displacement, kNN$\}$
        \State {\bf Return } $S_{\alpha, p}$
    \end{algorithmic}}
\end{algorithm}

\paragraph{PairSel}
We define a class of algorithms based on pairwise queries in Algorithm~\ref{alg:pair_sel}, which are broadly referred to as the PairSel algorithms. The core idea behind all these algorithms is to use a sorting of the datapoints via pairwise queries. We first describe how to obtain this sorting using pairwise queries.

For each pair of datapoints $x_i, x_j$, a pairwise query provides us with a comparison function between $x_i$ and $x_j$. The pairwise query output for $x_i, x_j$ is either $0$ or $1$. If this comparison function is  consistent (anti-symmetric and transitive~\citep[Appendix~B.2]{clrs} ), it can be used to sort any set of $n$ datapoints. We can use a comparison-based sorting algorithm, for instance MergeSort~\citep{clrs}, for this purpose, where each comparison is performed with a pairwise query. From Definition~\ref{assump:pair}, pairwise queries are determined by the true labeling function of the classifier, $p_{\mathrm{clf}, l}$, instead of the incorrect confidence values obtained by the confidence function $p_{\mathrm{clf}, c}$ that are used in ConfSel. Using the comparison-based sorting with pairwise queries as the comparison function, we can sort the datapoints in descending order of $p_{\mathrm{clf},l}$. This provides a sorting permutation $\Pi$ and forms a subroutine for all pairwise query-based algorithms in Algorithm~\ref{alg:pair_sel}. Note that we don't need any additional structure on the labeling functions or the space of features $\mathcal{X}$, to either run our algorithms or for our pairwise queries to be consistent.

We propose $5$ methods for the actual rejection function based on this sorting. We now explain the intuition behind each of these methods. Our algorithms are created assuming pairwise queries are mostly accurate.

\textbf{Middle} If the classifier's labels have a balanced distribution of classes in expectation, i.e., $\E_{x\sim \gD_x}[p_{\mathrm{clf},l}(x)] = \frac{1}{2}$, then the values in the middle $\alpha$-fraction of iid samples sorted according to $p_{\mathrm{clf},l}$ should be close to $\frac{1}{2}$. These are exactly the points on which the true labeling function has least confidence. 
\begin{algorithm}
{
\caption{PairSel-Middle()}
    \begin{algorithmic}
        \State {\bf Return} $\{\Pi^{-1}(i): i\in \{\floor{\frac{n(1-\alpha)}{2}}, \ceil{\frac{n(1+\alpha)}{2}}\}\}$
    \end{algorithmic}}
\end{algorithm}

\textbf{Max Entropy} Note that the function $H_{bin}(B)$ for a sequence $B$ consisting of $0$'s and $1$'s is the empirical binary entropy function. If $m_B$ is the number of $1$'s in $B$, then, $H_{bin}(B) = -\frac{m_B}{\abs{B}}\log\left(\frac{m_B}{\abs{B}}\right) -\left(1-\frac{m_B}{\abs{B}}\right)\log\left(1 - \frac{m_B}{\abs{B}}\right)$. If $\E_{x\sim \gD_x}[p_{\mathrm{clf},l}(x)]\neq \frac{1}{2}$, then the middle is not an appropriate choice. However, assuming the sorting according to pairwise queries are accurate, the contiguous subsequence of the sorted datapoints where $p_{\mathrm{clf},l}$ is close to $\frac{1}{2}$, should have approximately equal number of $0$'s and $1$'s. This contiguous subsequence should thus have highest empirical binary entropy. 
\begin{algorithm}
{
\caption{PairSel-Max-Entropy()}}
    \begin{algorithmic}
        \State $i^\star = \arg\max_{i\in [(1-\alpha)n]\cup\{0\}} H_{bin}(\{\hat{y}_{\Pi^{-1}(i+k)}\}_{k=1}^{\alpha n})$
        \State {\bf Return} $\{\Pi^{-1}(i^\star+k): k\in [\alpha n]\}$
    \end{algorithmic}
\end{algorithm}

\textbf{Max Presence} In this method, we find the largest contiguous subsequence of the datapoints sorted by $p_{\mathrm{clf},l}$ such that it agrees with set of datapoints rejected by the confidence queries $S_{\alpha,c}$. This method assumes that pairwise queries might be inaccurate but they can be combined with inaccurate confidence estimates for better performance overall.
\begin{algorithm}
{\caption{PairSel-Max-Presence()}
    \begin{algorithmic}
        \State $i^\star = \arg\max_{i\in [(1-\alpha)n]\cup\{0\}} \sum_{k=1}^{\alpha n} \vone\{\Pi^{-1}(i) \in S_{c,\alpha}\}$ \quad \# Here, $\vone$ is the indicator function.
        \State {\bf Return} $\{\Pi^{-1}(i^\star+k): k\in [\alpha n]\}$
    \end{algorithmic}}
\end{algorithm}

\textbf{Max Displacement} This method measures the difference between the positions of the sorting by confidence values and pairwise queries. If the label $\hat{y}_i=1$, then $\Pi_c(i) - \Pi(i)$ is positive if a datapoint $x_i$ has smaller bias, relative to other datapoints, based on $p_{\mathrm{clf},c}$ than $p_{\mathrm{clf},l}$. In this case, $m_i$ is large. Similar logic applies for $\hat{y}_i=0$. Note that $m_i$ is large whenever the pairwise sorting differs from the confidence values based on predicted labels and this method rejects such datapoints.
\begin{algorithm}[H]
{\caption{PairSel-Max-Displacement()}
    \begin{algorithmic}
        \State Let $\Pi_c:[n]\to[n]$ be the permutation obtained by sorting the datapoints in descending order according to $p_{\mathrm{clf},c}$.
        \State Let $m_i = (\Pi_c(i) - \Pi(i))(2\hat{y}_i -1)$
        \State {\bf Return} Set of datapoints in the largest $\alpha$-quantile according to $m_i$
    \end{algorithmic}}
\end{algorithm}

\textbf{kNN} This method assumes access to a distance metric between datapoints, namely $\text{dist}$. The metric $m_i$ penalizes two quantities -- i) the average distance of a datapoint to its $k$ nearest neighbors in $\text{dist}$, ii) the average confidence of a datapoint's $k$ nearest neighbors in their labels. The average distance ensures that isolated datapoints are rejected. Note that the second term of $m_i$ evaluates to $(1-p_{\mathrm{clf},c}(x_j))$ when $\hat{y}_j=1$ and $p_{\mathrm{clf},c}(x_j)$ when $\hat{y}_j=0$. Therefore, if the datapoint $x_j$ is confident in its predicted label, then the corresponding terms should be small. If a datapoint's neighbors are not confident in their labels according to $p_{\mathrm{clf},c}$, which might be due to incoherence of $p_{\mathrm{clf},l}$ and $p_{\mathrm{clf},l}$, then the metric $m_i$ is large.
\begin{algorithm}
{{\caption{PairSel-kNN()}
    \begin{algorithmic}
        \State Let $N_k(x_i)\subset [n]\setminus \{i\}$ be the set of $k$ nearest neighbors of $x_i$ according to distance metric $\text{dist}$ on features $x$.
        \State Let $m_i = \frac{1}{k}\sum_{j\in N_k(x_i)} \text{dist}(x_j, x_i)  ((1- 2 p_{\mathrm{clf},c}(x_j))\hat{y}_j + p_{\mathrm{clf},c}(x_j))$
        \State $i^\star = \arg\max_{i\in [(1-\alpha)n]\cup\{0\}} \sum_{j=1}^{\alpha n} m_{\Pi^{-1}(i + j)}$
        \State {\bf Return} $\{\Pi^{-1}(i^\star+j): j\in [\alpha n]\}$
    \end{algorithmic}}}
\end{algorithm}

Note that PairSel-Middle is the simplest of the pairwise queries and PairSel-kNN is the most complicated one and only PairSel-kNN requires access to a distance metric.

\section{Why Pairwise Queries are Beneficial?}
\label{sec:theory}
In this section, we find conditions under which the simplest pairwise query-based algorithm, PairSel-Middle from Algorithm~\ref{alg:pair_sel} obtains lower total selective risk than that of ConfSel (Algorithm~\ref{alg:conf_sel}). The proofs for all the theoretical results in this section are deferred to Appendix~\ref{sec:proofs}.
\begin{lemma}\label{lem:conf_err}
Let $F_{\mathrm{clf},c}(\alpha) = \E_{x_i, \hat{y}_i, \tilde{y}_i, y^\star_i}[f_{S_{\alpha,c}}(n)]$  be population selective risk of ConfSel (Algorithm~\ref{alg:conf_sel}) for a fixed $\alpha\in \{0, \frac{1}{n},\frac{2}{n} \ldots, 1 - \frac{1}{n}, 1\}$. Then, $F_{\mathrm{clf}} - F_{\mathrm{clf},c}(\alpha)=$ 
\begin{align*}
\E_{x\sim \gD_x}\bigg[(2p^\star(x)-1)(p_h(x) - p_{\mathrm{clf},l}(x))
\quad \cdot\left(\Pr_{x'\sim \gD_x}[\abs{2p_{\mathrm{clf},c}(x') - 1} \geq \abs{2p_{\mathrm{clf},c}(x) - 1}]\right)^{(1-\alpha)n}\bigg].
\end{align*}
\end{lemma}
To quantify the population selective risk for PairSel-Middle, we first define our pairwise query oracle in Defintion~\ref{assump:pair}.
\begin{definition}[Pairwise Query Oracle]\label{assump:pair}
    A pairwise query oracle,  given  two points $x, x'\in \gX$ on the classifier $\mathrm{clf}$, will provide $1$ if $p_{,l}(x) \geq p_{\mathrm{clf},l}(x')$ and $0$ otherwise.
\end{definition}
This oracle assumes that pairwise queries are always accurate. This greatly simplifies the analysis of pairwise queries and is still realistic, as often in real models, including LLMs~\citep{qin-etal-2024-large}, pairwise queries are more accurate than direct queries.

\begin{lemma}\label{lem:pair_err}
    Let $F_{\mathrm{clf},p}(\alpha) = \E_{x_i, \hat{y}_i, \tilde{y}_i, y^\star_i}[f_{S_{\alpha, p}}(n)]$ be the population selective risk of PairSel-Middle (Algorithm~\ref{alg:pair_sel}) for a fixed $\alpha\in \{0, \frac{1}{n},\frac{2}{n} \ldots, 1 - \frac{1}{n}, 1\}$. Then, with access to a pairwise query oracle (Definition~\ref{assump:pair}), $F_{\mathrm{clf}} - F_{\mathrm{clf},p}(\alpha)=$ 
    \begin{align*}
     \E_{x\sim \gD_x}\bigg[(2p^\star(x)-1)(p_h(x) - p_{\mathrm{clf},l}(x))   \left(\Pr_{x'\sim \gD_x}[p_{\mathrm{clf},l}(x')\leq p_{\mathrm{clf},l}(x)]\right)^{\ceil{(1-\alpha)n/2}}\\\cdot  \left(\Pr_{x'\sim \gD_x}[p_{\mathrm{clf},l}(x')\geq p_{\mathrm{clf},l}(x)]\right)^{\floor{(1-\alpha)n/2}}\bigg].
    \end{align*}
\end{lemma}
Note that both Lemmas~\ref{lem:conf_err} and \ref{lem:pair_err} have a term in common, $(2p^\star(x)-1)(p_h(x) - p_{\mathrm{clf},l}(x))$. If the true label was $1$ for some $x\in \gX$, $p^\star(x) > \frac{1}{2}$. If the human is more accurate on the point $x$, i.e., $p_h(x) > p_{\mathrm{clf},l}(x)$,  it implies better selective risk for both ConfSel and PairSel-Middle.  However, the two Lemmas differ in their additional terms with an exponent of $(1-\alpha)n$, with ConfSel's corresponding term depending only on the distribution of $p_{\mathrm{clf},c}$, while for PairSel-Middle, the corresponding term depends on the distribution of $p_{\mathrm{clf},l}$. This additional term for a given $x\in \gX$ corresponds to the probability of being rejected.

To prove benefit of pairwise querires, we need to show that $F_{\mathrm{clf},p}(\alpha) \leq F_{\mathrm{clf},c}(\alpha)$ for some fixed $\alpha$. For arbitrary $p_{\mathrm{clf},l}, p^\star, p_h$ and $p_{\mathrm{clf},c}$, we can only obtain a strong sufficient condition, which we discuss in Appendix~\ref{sec:add_theory}. For weaker conditions on these labeling functions, we restrict them to a specific class of functions described by the following assumption.

\begin{assumption}[Classifier]\label{assump:lin_model}
We assume that there exists a feature mapping  $\varphi:\gX \to \R^d$, where $d$ is the ambient dimension of the classifier, models $w^\star, w_{h}, w_{\mathrm{clf},l}, w_{\mathrm{clf},c}\in \R^d$ each of unit norm, and a differentiable non-decreasing link function $g:\R\to [0,1]$ with $g(a) = 1 - g(-a), \forall a \in \R$, such that $\forall x\in \gX$,\,\, $p^\star(x) = g(\ip{w^\star}{\varphi(x)})$,\,\,\, $p_{h}(x) = g(\ip{w_{h}}{\varphi(x)})$,\,\,\, $p_{\mathrm{clf},l}(x) = g(\ip{w_{\mathrm{clf},l}}{\varphi(x)})$\,\, and \,\,$p_{\mathrm{clf},c}(x) = g(\ip{w_{\mathrm{clf},c}}{\varphi(x)})$.
\end{assumption}
A common example of  $g$ is the sigmoid function $(1+e^{-a})^{-1}$. %
The above assumption is satisfied if the classifier is a linear or a generalized linear model with link function $g$. For kernel machines, the inner product in above assumption can be expressed as the sum of kernel functions on the input $x$. Neural networks with large width, that have been trained with a small learning rate, or whose model weights remain close to initialization after training, satisfy the Neural Tangent Kernel (NTK) assumption~\citep{jacot_gabriel_hongler_2018,chizat_bach_2019}. The output of these classifiers can be represented as a linear function of the weights with a Neural Tangent Kernel, and thus satisfies the above assumptions.

\blue{The above assumption also assumes that the feature space $\mathcal{X}$ can be transformed to the metric space $\R^{d}$ via the feature mapping $\varphi$. Note that such a mapping, which ensures that pairwise queries are consistent and well-defined, always exists if we consider the complete space of features. We provide an example of what this feature mapping $\varphi$ might look like. Suppose we want to classify animals based on their number of legs and their body size. Let $\varphi$ map $x$ (animals) to $d=2$ coordinates, one for the number of legs, and another for body size, such that animals with different number of legs, for instance cat and bird, and animals with different sizes, for instance lion and mouse, can be differentiated based on pairwise queries. Let the LLM model be $w_{\mathrm{clf}, l}^\top = [1,1]$. Then, for cat ($x_i$) and bird ($x_j$), only the first coordinate of $\varphi(x_i) - \varphi(x_j)$ is large, and its inner product $\ip{w_{\mathrm{clf}, l}}{\varphi(x_i) - \varphi(x_j)}$ is still large. For the other case of lion ($x_i$) and mouse ($x_j$), only the second coordinate of $\varphi(x_i) - \varphi(x_j)$ is large, and its inner product $\ip{w_{\mathrm{clf}, l}}{\varphi(x_i) - \varphi(x_j)}$ is still large. Thus, pairwise queries can differentiate these examples. One major issue with such a feature mapping can be the ambient dimension of this mapping, $d$, being extremely large. This is empirically not the case for real datasets. For popular vision datasets, the intrinsic dimension is $~100$, as empirically verified in ~\citep{pope2021the}.}

Now, to compute $F_{\mathrm{clf},p}(\alpha)$ and $F_{\mathrm{clf},c}(\alpha)$, we require assumptions on the distribution of features $\gD_x$. We consider two special cases -- Spherical features and Gaussian features.

\begin{theorem}[Spherical Features, Linear Link]\label{thm:spherical}
    Under Assumptions~\ref{assump:pair} and \ref{assump:lin_model}, if $x\sim \gD_x$ is equivalent to $\varphi(x) \sim \mathrm{Unif}(\mathbb{S}^{d-1})$ and the link function is $g(a) = \frac{1 + a}{2},\,\forall a \in [-1,1]$, for any constant $\alpha \in (0,1)$ and large $n$, a sufficient condition for $F_{\mathrm{clf},c}(\alpha) \geq F_{\mathrm{clf},p}(\alpha)$ is 
    \begin{equation}\label{eq:spherical_thm}
    \begin{aligned}
            &\ip{w^\star}{w_{\mathrm{clf},c}}(\ip{w_{\mathrm{clf},c}}{w_h} - \ip{w_{\mathrm{clf},c}}{w_{\mathrm{clf},l}})
            \geq \left(1 + \Omega\left(\frac{d^{-(1-\alpha)n/2}}{(1-\alpha)n}\right)\right)(1 - 2^{-(1-\alpha)n})\\
            &\qquad \qquad \cdot(\ip{w^\star}{w_h} - \ip{w^\star}{w_{\mathrm{clf},l}}) +2^{-(1-\alpha)n}(\ip{w^\star}{w_{\mathrm{clf},l}}(\ip{w_h}{w_{\mathrm{clf},l}} - 1)). 
    \end{aligned}
    \end{equation}
\end{theorem}

\begin{theorem}[Gaussian Features, Smooth Link]\label{thm:gauss}
    Under Assumptions~\ref{assump:pair} and \ref{assump:lin_model}, if $x\sim \gD_x$ is equivalent to $\varphi(x)\sim \gN(0, \mathbb{I}_d)$ and the link function $g$ is $L$-smooth, then for any constant $\alpha\in (0,1)$ and large $n$, a sufficient condition for $F_{\mathrm{clf},c}(\alpha) \geq F_{\mathrm{clf},p}(\alpha)$ is 
    \begin{equation}\label{eq:gauss_thm}
    \begin{aligned}
            &\ip{w^\star}{w_{\mathrm{clf},c}}(\ip{w_{\mathrm{clf},c}}{w_h} - \ip{w_{\mathrm{clf},c}}{w_{\mathrm{clf},l}}) 
            \geq\, \left(1 + \omega(1)\right)(1 - 2^{-(1-\alpha)n})(\ip{w^\star}{w_h} - \ip{w^\star}{w_{\mathrm{clf},l}}) \\
            & \qquad +2^{-(1-\alpha)n}(\ip{w^\star}{w_{\mathrm{clf},l}}(\ip{w_h}{w_{\mathrm{clf},l}} - 1)) + \Omega(L d^{3/2}(1+2^{(1-\alpha)n})/g'(0)).
    \end{aligned}
    \end{equation}
\end{theorem}
The only differences in both the theorems is the additional term due to $L$-smoothness in \eqref{eq:gauss_thm} and the sharper characterization of the $\Omega\left(\frac{d^{-(1-\alpha)n/2}}{(1-\alpha)n)}\right)$ term in \eqref{eq:spherical_thm}.
Note that $L$-smooth link functions suffer from an additional term dependent on the dimension of the classifier. This can be made small for small smoothness constant $L$ or a $g$ which has a sharp increase at $0$. To interpret these conditions, consider the first term on RHS in both the equations. Assuming that the human is close to the true distribution, $w_h \approxeq w^\star$, the first term is always $\leq 1$. Further, its coefficient is always close to $1$ for large $n$ and constant $\alpha$. The second term on RHS of both equations is always negative as long as the classifier is a good classifier, i.e, $\ip{w^\star}{w_{\mathrm{clf},l}} \geq 0$. However, the impact of the second term is diminished for large $n$ or small $\alpha$. If we assume that the confidence model is also not bad, i.e., $\ip{w^\star}{w_{\mathrm{clf},c}}>0$, we find that $\ip{w_{\mathrm{clf},c}}{w_{\mathrm{clf},l}}$, or \textbf{the similarity between the confidence model and the actual labeling model, should be small for advantage of pairwise queries}. Note that if the confidence model is accurate, i.e., $w_{\mathrm{clf},c} = w_{\mathrm{clf},l}$, we require the second term on RHS, without the coefficients of $2^{-(1-\alpha)n}$ to be larger than the first term on the RHS. For a good classifier, the second term on RHS is always negative and can never be larger than the first term, which is positive as long as the human and true classifier are close and the human is better than the classifier. Therefore, for a good classifier and accurate confidences, ConfSel is always better than pairwise queries. Theorems~\ref{thm:spherical} and \ref{thm:gauss} provide conditions for any classifier, confidence function, human and true labels, which might not obey any of the reasonable assumptions we used in the above discussion(good classifier and confidence, human $\approxeq$ true labels).  As our theorems establish conditions when our simplest pairwise query-based algorithm, PairSel-Middle, is better than ConfSel, we expect that more complicated pairwise query-based algorithms in Algorithm~\ref{alg:pair_sel} will also be better than ConfSel under weaker conditions.

\section{Experiments}
\label{sec:exp}
To evaluate the performance of our methods empirically, we run them on $5$ ( $4$ real and $1$ synthetic) different binary classification datasets. Motivated by the mismatch between confidences and labeling functions in LLMs (Fig~\ref{fig:mot_ex}), all our real datasets utilize in-context learning on LLMs for direct and pairwise queries.  Additional details about our experiments are deferred to Appendix~\ref{sec:add_exp}.

\subsection{Setup}
\label{sec:exp_setup}

    \begin{figure*}[t!]
    \centering
    \begin{subfigure}[t]{0.5\textwidth}
    \centering
    \includegraphics[height=2.5in, width=\textwidth]{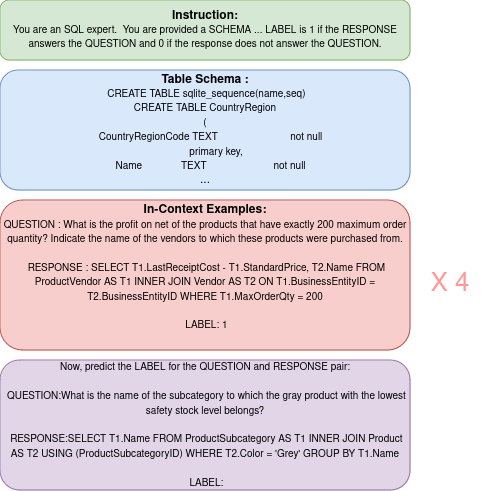}
    \end{subfigure}%
    \begin{subfigure}[t]{0.5\textwidth}
    \centering
    \includegraphics[height=2.5in, width=\textwidth]{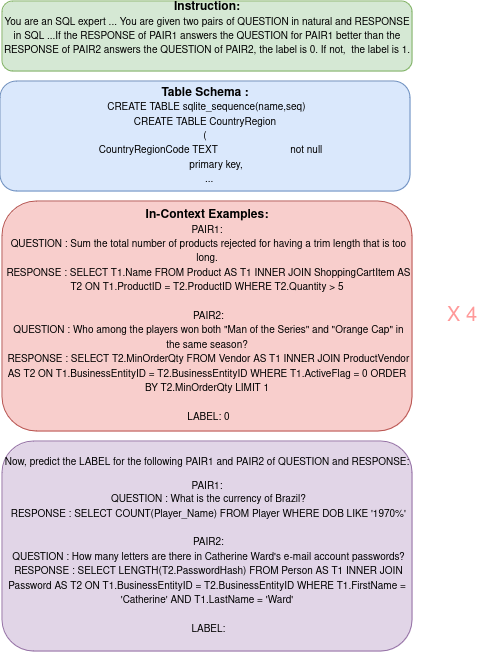}
    \end{subfigure}
    \caption{ICL Prompt for a direct query (left) and a pairwise query (right) from Bird. The ``X 4'' in each image denotes that there are $4$ in-context examples of which only one has been presented. Prompt template derived from the Bird benchmark ~\citep[\href{https://bird-bench.github.io}{https://bird-bench.github.io}]{li2024can}, distributed under CC BY-SA 4.0.}
    \label{fig:prompt}
\end{figure*}

\textbf{Datasets:}
Our first dataset, which we call Synthetic, is manually generated based on our theoretical model defined in Assumption~\ref{assump:lin_model}. Here, the features are Gaussian ($\varphi(x) = x, \forall x$), and the labels are generated from the optimal model  $w^\star$ using a sigmoid link function $g$. 

The second and third datasets are $2$ NL2SQL datasets -- Spider~\citep{yu-etal-2018-spider} and Bird~\citep{li2024can}. These datasets are not binary classification datasets, as they contain a natural language query and a SQL statement on a database that answers the question. To convert these datasets to binary classification, we create a dataset for verifying the correctness of a SQL query for a given question.  We split these datasets into two parts: the first retains the original (question, SQL) pairs and is assigned the positive label, and for the second part, we generate an incorrect SQL (by derangement) for each question, forming the negative label. The task is to predict the label for each (question, SQL) pair. We use the older Spider dataset instead of the newer and more challenging Spider2 dataset, as not all questions in Spider2 have correct SQLs, which are required for our task of verifying the correctness of SQL~\citep{lei2024spider}. 

Our fourth and fifth datasets are originally binary classification datasets, unlike Spider and Bird, where we needed to create a binary classification from the original NL2SQL task. The fourth dataset is BoolQ~\citep{clark2019boolq}, which consists of pairs of passages and True/False questions converted to binary labels. This is a part of the SuperGLUE~\citep{Wang2019SuperGLUE} benchmark used for evaluating Natural Language Understanding of LLMs. The fifth dataset is a vision dataset, VisOnlyQA~\citep{kamoi2025visonlyqa}, where large vision language models still struggle. It consists of pairs of an image and a question. A subset of these questions have a True/False answer, which we use for binary classification.

\textbf{Classifiers:} For the Synthetic dataset, the classifier itself is a pair of linear models, the labeling model $w_{\mathrm{clf},l}$ and the confidence model $w_{\mathrm{clf},c}$. We generate these so that $w_{\mathrm{clf},l}$ is closer to $w^\star$ than $w_{\mathrm{clf},c}$. For all our real datasets, we use open-source LLMs with $<10B$ parameters. For Spider, Bird and BoolQ datasets, we use $3$ general purpose LLMs --  Gemma 3 4B Instruct ~\citep{gemmateam2025gemma3technicalreport}, Llama 3.1 8B Instruct ~\citep{grattafiori2024llama3herdmodels}, and the most recent Qwen3 8B Instruct ~\citep{yang2025qwen3technicalreport}. For Spider and Bird, we also use SQLCoder Llama 3 8B ~\citep{sqlcoder} which is excellent for NL2SQL task. For VisOnlyQA, we use $3$ Vision Language Models -- Gemma 3 4B Instruct, Qwen 2.5 VL 7B Instruct~\citep{bai2025qwen25vltechnicalreport}, and the state-of-the-art GLM 4.1V 9B Thinking ~\citep{vteam2025glm45vglm41vthinkingversatilemultimodal}. All experiments were conducted using locally-run open-weight models. No hosted commercial API (e.g., OpenAI) was queried for experimental results.

\subsubsection{Baselines}
Throughout our experiments, we assume that the human returns the correct label for each datapoint.

\textbf{Direct/Raw Baseline:} Our first baseline is the direct labeling of each pair represented as raw in all our results. For Synthetic dataset, the label is obtained from the labeling model $w_{\mathrm{clf},l}$ and the confidence in the label is obtained by the confidence model $w_{\mathrm{clf},c}$. For real datasets, the label is obtained by an in-context learning prompt on the LLM, and the confidence is obtained by normalizing the LLM's output logits. Fig~\ref{fig:prompt} is an example of direct query for Bird dataset. The normalization process for confidence involves applying a softmax on the raw logits for the tokens $0$ and $1$ at the position of the label in the classifier output. The raw baseline applies ConfSel on these labels, treating the confidences as $p_{\mathrm{clf},c}$.

\textbf{Calibrated and Base Model:} To check if a more calibrated confidence function can improve selective classification over raw confidences, we consider two baselines  -- i) Calibrated labels and confidences obtained by prototypical calibration as a post-processing step~\citep{han2023prototypical} and ii) Using confidence of \textit{base model} and label from an instruct model, as base models are better calibrated than instruct models~\citep{openai2024gpt4}, but instruct models are better at in-context learning~\citep{instruct, wei2022finetuned}. Note that the base model is not used for the Synthetic and VisOnlyQA datasets and for Llama3 SQLCoder as it isn't available for these cases.

\textbf{Pairwise Queries:} For Synthetic dataset, we generate noiseless pairwise queries from the oracle defined in Definition~\ref{assump:pair}. For the real datasets, a single pairwise query is performed via in-context learning on the LLM. For instance, for Spider and Bird dataset, we provide the classifier with $2$ pairs of examples (NL-SQL pairs) and ask which of them is more correct, where the label $1$ corresponds to correct and $0$ to incorrect. Figure~\ref{fig:prompt} provides an example of a pairwise query for Bird dataset. Note that for all the remaining real datasets, this pairwise query based on correctness still makes sense as they use True/False questions where the label $1$ corresponds to True and a correct answer. As higher correctness for a feature $x\in \gX$ implies a higher $p_{\mathrm{clf},l}(x)$, these pairwise ICL queries are a noisy version of pairwise queries described in our theoretical setup (Section~\ref{sec:theory}).
\blue{
\begin{figure}[htbp]
\footnotesize
    \centering
    \begin{subfigure}[t]{0.33\textwidth}
        \centering
        \includegraphics[height=1.17in]{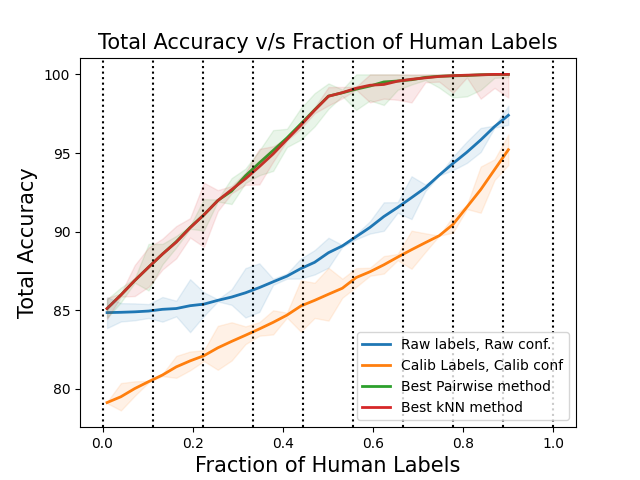}
        \caption{Synthetic, Linear}
    \end{subfigure}%
    \begin{subfigure}[t]{0.33\textwidth}
        \centering
        \includegraphics[height=1.17in]{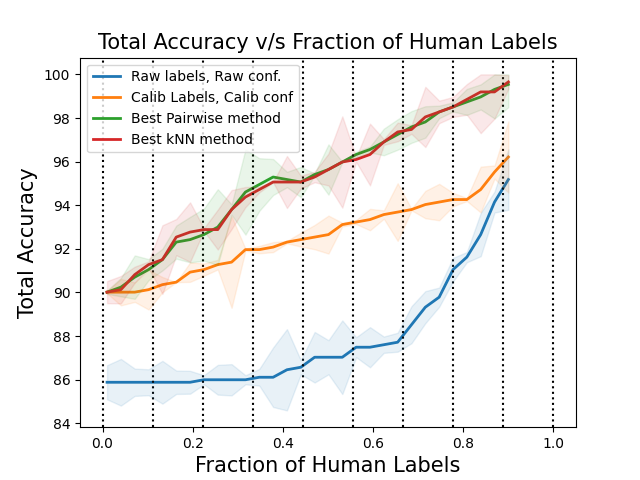}
        \caption{Bird, Llama3 SQLCoder}
    \end{subfigure}%
    \begin{subfigure}[t]{0.33\textwidth}
        \centering
        \includegraphics[height=1.17in]{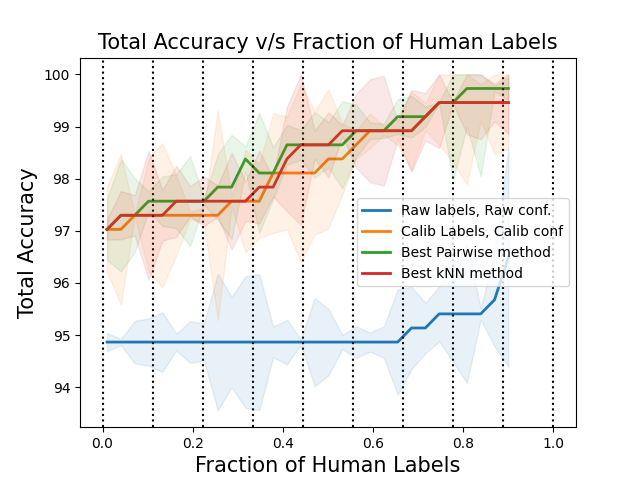}
        \caption{Spider, Llama3 SQLCoder}
    \end{subfigure}%
    \hfill
    \begin{subfigure}[t]{0.33\textwidth}
        \centering
        \includegraphics[height=1.17in]{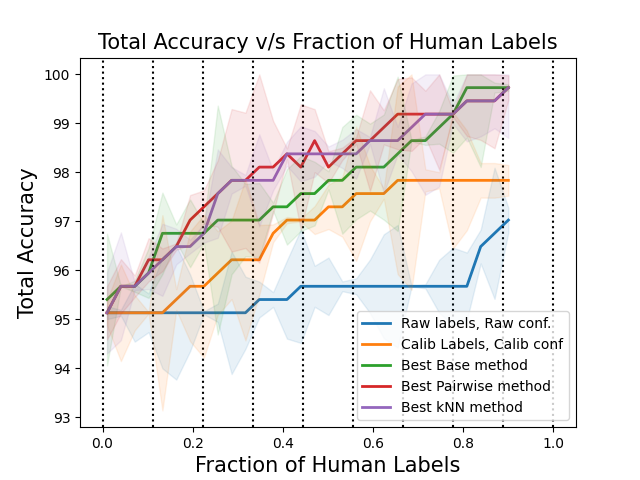}
        \caption{Spider, Gemma 3}
    \end{subfigure}%
    \begin{subfigure}[t]{0.33\textwidth}
        \centering
        \includegraphics[height=1.17in]{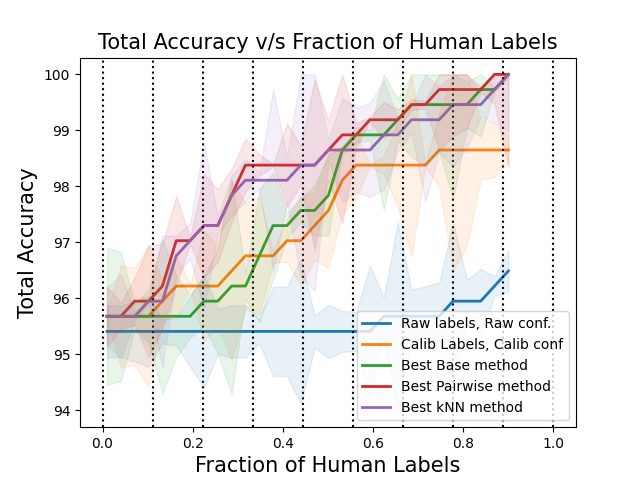}
        \caption{Spider, Llama 3.1}
    \end{subfigure}%
    \begin{subfigure}[t]{0.33\textwidth}
        \centering
        \includegraphics[height=1.17in]{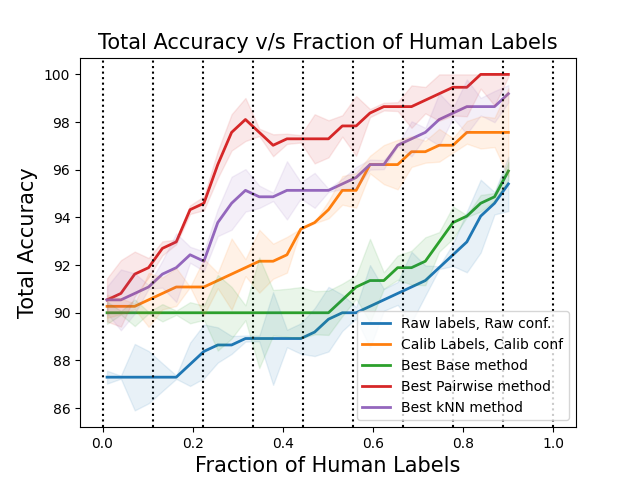}
        \caption{Spider, Qwen 3}
    \end{subfigure}%
    \hfill
    \begin{subfigure}[t]{0.33\textwidth}
        \centering
        \includegraphics[height=1.17in]{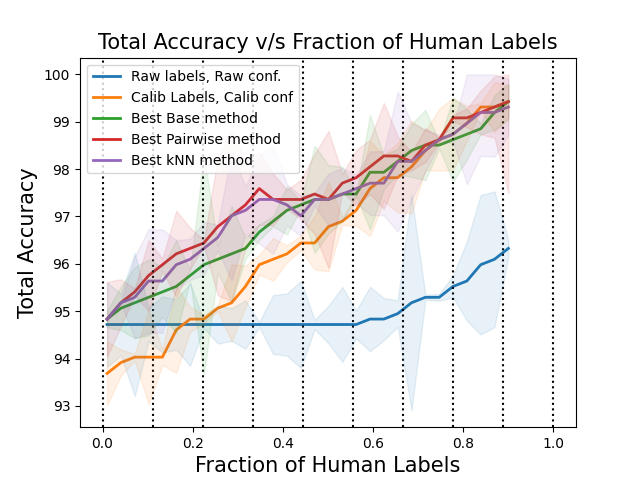}
        \caption{Bird, Gemma 3}
    \end{subfigure}%
    \begin{subfigure}[t]{0.33\textwidth}
        \centering
        \includegraphics[height=1.17in]{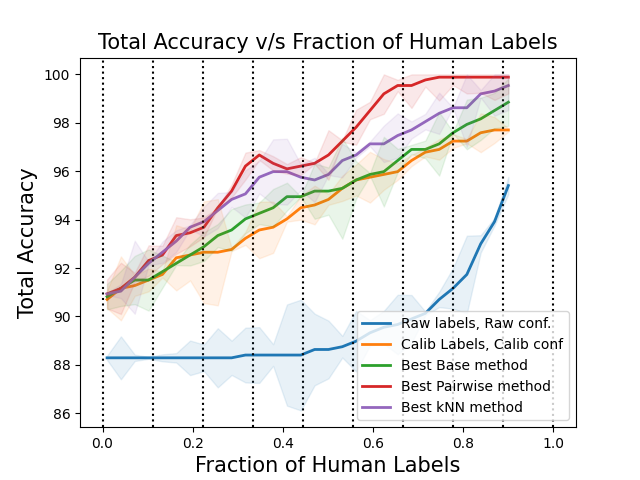}
        \caption{Bird, Llama 3.1}
    \end{subfigure}%
    \begin{subfigure}[t]{0.33\textwidth}
        \centering
        \includegraphics[height=1.17in]{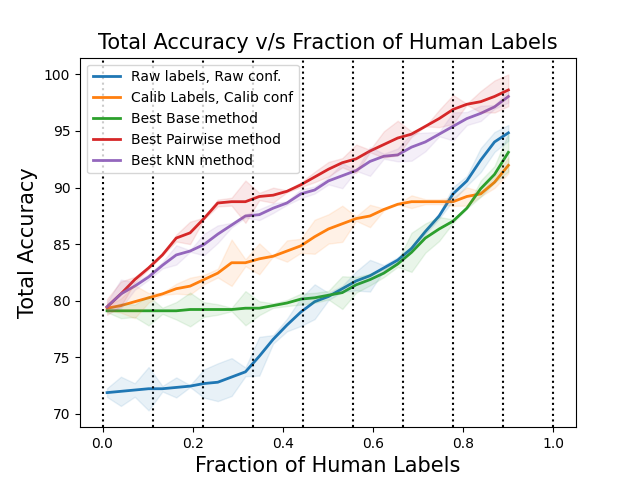}
        \caption{Bird, Qwen 3}
    \end{subfigure}%
    \hfill
    \begin{subfigure}[t]{0.33\textwidth}
        \centering
        \includegraphics[height=1.17in]{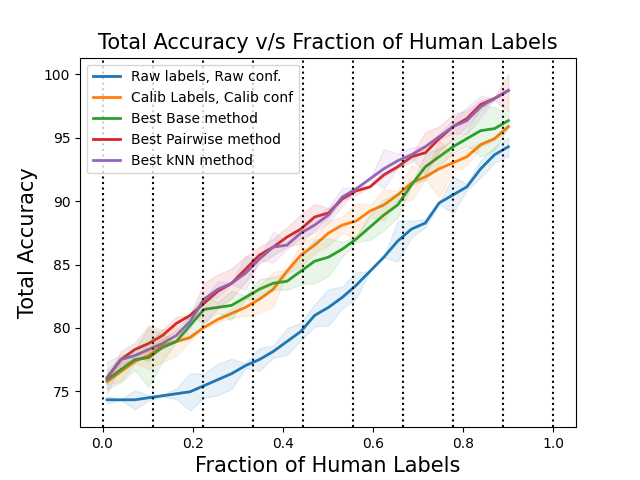}
        \caption{BoolQ, Gemma 3}
    \end{subfigure}%
    \begin{subfigure}[t]{0.33\textwidth}
        \centering
        \includegraphics[height=1.17in]{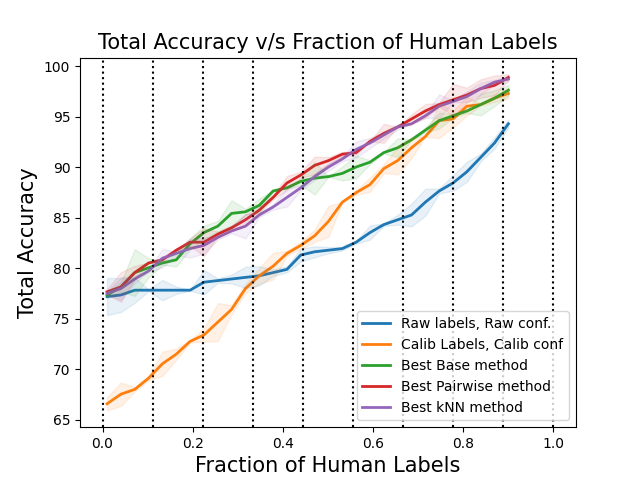}
        \caption{BoolQ, Llama 3.1}
    \end{subfigure}%
    \begin{subfigure}[t]{0.33\textwidth}
        \centering
        \includegraphics[height=1.17in]{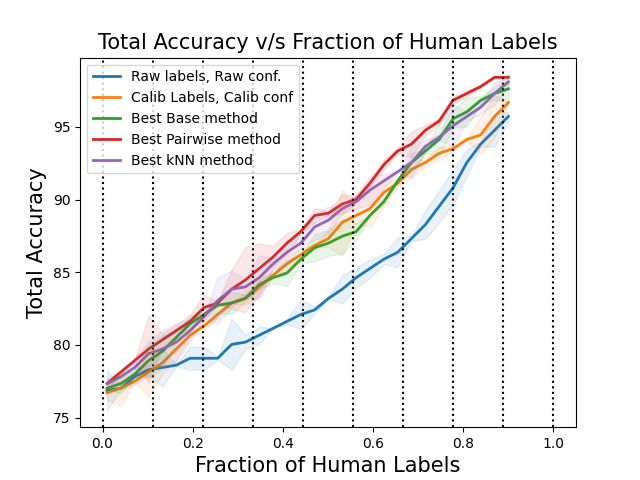}
        \caption{BoolQ, Qwen 3}
    \end{subfigure}%
    \hfill
    \begin{subfigure}[t]{0.33\textwidth}
        \centering
        \includegraphics[height=1.17in]{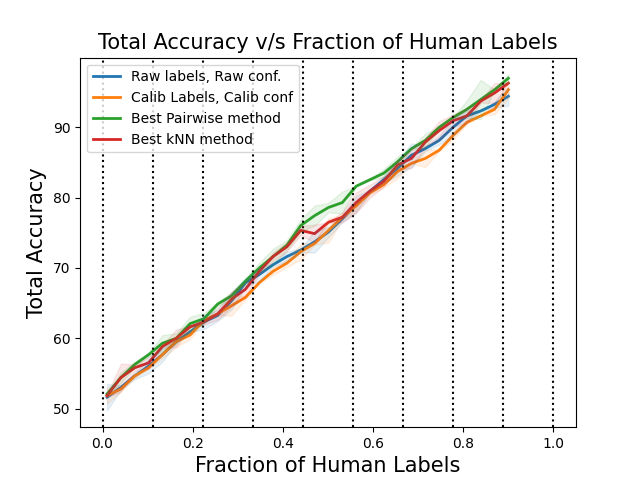}
        \caption{VisOnlyQA, Gemma 3}
    \end{subfigure}%
    \begin{subfigure}[t]{0.33\textwidth}
        \centering
        \includegraphics[height=1.17in]{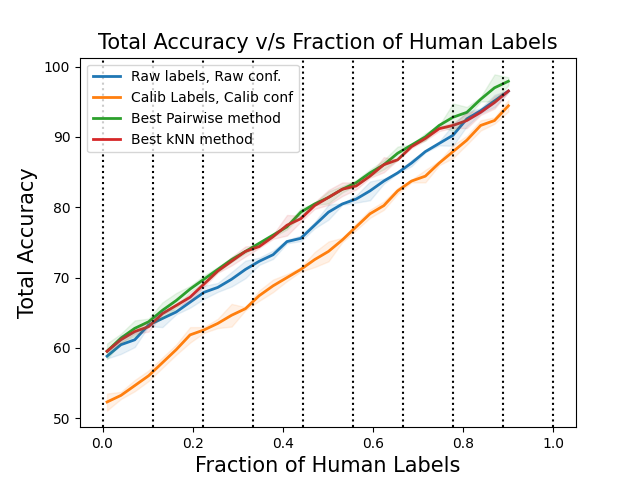}
        \caption{VisOnlyQA, Qwen 2.5 VL}
    \end{subfigure}%
    \begin{subfigure}[t]{0.3\textwidth}
        \centering
        \includegraphics[height=1.17in]{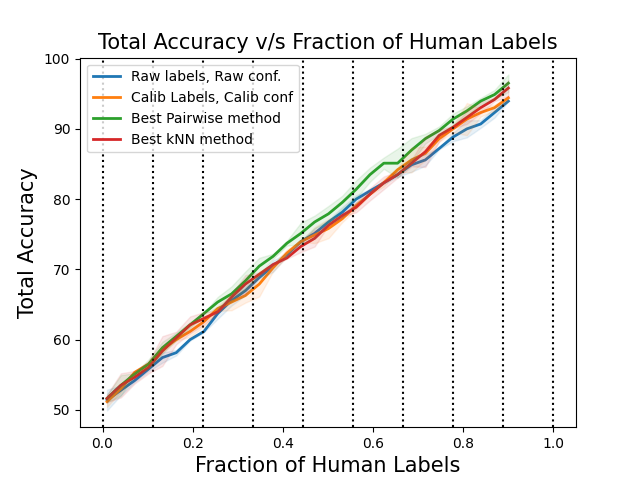}
        \caption{VisOnlyQA, GLM 4.1V}
    \end{subfigure}%

    \caption{\footnotesize Total Accuracy of classifier v/s the fraction of human labels for all datasets (Synthetic, Bird, Spider, BoolQ, VisOnlyQA). For each fraction of human labels, only the best performing method (Base, Pairwise and kNN) is reported along with raw and calibrated baselines.}
    \label{fig:cov_all}
\end{figure}}

\textbf{kNN:}
To implement PairSel-kNN, we require a distance metric, $\text{dist}$, between two features. For synthetic dataset, the features are $d$-dimensional vectors, and the distance is $\ell_2$ norm of their difference. For Spider and Bird datasets, each feature $x \in \gX$ is a pair of NL-SQL, represented as $x = (x_{NL}, x_{SQL})$, we use $\text{dist}(x,x') = \text{dist}_{\text{NL}}(x_{NL}, x_{NL}') + \lambda \cdot \text{dist}_{\text{SQL}}(x_{SQL}, x_{SQL}')$ for $x,x'\in \gX$, where $\text{dist}_{NL}$ and $\text{dist}_{SQL}$ are distance metrics for natural language and SQLs respectively and $\lambda >0$ is a weight. For NL, we use a popular  distance metric, cosine distance between sentence embeddings~\citep{ni-etal-2022-sentence}. For SQL, however, specialized and widely accepted distance metrics do not exist, so we propose one in Appendix~\ref{sec:dist_metric} and use it in our experiments. BoolQ and VisOnlyQA datasets, consist of question and passage/image pairs, respectively. For these, we use a similar distance computation, where instead of the distance between SQLs, we compute distances between passages or images. Distance between passages is the same as distance between NL, while we use cosine distance between CLIP embeddings~\citep{radford2021learningtransferablevisualmodels} for images.

\textbf{Naive Classifier:}
This classifier labels each datapoint with the majority label of the dataset and selects datapoints to reject uniformly at random.

\textbf{Metrics:}
As all our baselines are selective classification algorithms, we measure the total accuracy of a method at a fixed fraction of human labels $\alpha$. This is defined as the accuracy of the set of labels, when the rejected datapoints have human labels and the non-rejected datapoints have the classifier's labels. By varying $\alpha$ from $0.01$ to $0.9$, we plot the total accuracy v/s $\alpha$ in Fig~\ref{fig:cov_all} for all datasets and classifiers. For the baselines, base model, pairwise queries and PairSel-kNN baselines, we have several choices for the labels (calibrated/raw), confidences(calibrated/raw) and the actual method (Middle/Max-Entropy/Max-Presence/Max-Displacement). For each $\alpha$, for each of these baselines, we only consider the best performing method among all the possible choices. These correspond to the legends ``Best Pairwise Method" and so on in the figures. The $10$ dotted vertical lines represent $\alpha=10\%, 20\%, \ldots, 100\%$. To compute the maximum possible gain by the pairwise method, for each combination of Dataset and classifier, we find the value of $\alpha < 0.5$ at which the difference between the best pairwise method and the raw baseline is largest. In practice, a fixed budget for labeling implies an upper bound on the fraction of points $\alpha$ that can be human-labeled. We report the total accuracy for this $\alpha$, which we call best $\alpha$, in Table~\ref{tab:tot_acc}. Note that all our experiments are run for $5$ seeds, and we plot the standard deviation in the figures. Additional results, including using the cost of queries instead of $\alpha$ and sensitivity to parameters, are deferred to Appendix~\ref{sec:add_exp}.

\subsubsection{\blue{Setup for Fig~\ref{fig:mot_ex}}}
\blue{For Fig~\ref{fig:mot_ex}, we compute $p_{v,c}$ and $p_{v,l}$ for all datapoints on the filtered Bird dataset with Gemma3 4B Instruct. For each datapoint, we run a direct query for $50$ different seeds and compute the fraction of predicted labels that are $1$ across these runs as an estimate of $p_{v,l}$. The average value of the softmax of next token logits across these runs gives us the value of $p_{v,c}$. For the threshold with p-value of $5$, we consider the case when $p_{v,c}$ and $p_{v,l}$ are exactly equal. If we ran this experiment for $R$ random seeds then, the $p_{v,l} \approx p_{v,c}$ with p-value $\delta$ if $1-\delta$ of the mass of the distribution of $\abs{p_{v,l} - p_{v,c}}$ lies between $0$ and  $\sqrt{\frac{2\log(1/\delta)}{R}}$. The threshold is obtained by a simple application of Chernoff bound for binomial distributions.  We set $R=50$ for our case.}

\subsection{Results}
\begin{table}[t!]
    \centering
    \resizebox{\textwidth}{!}{
    \setlength{\tabcolsep}{2pt}

    {\color{black}
    \begin{tabular}{|c|cc|cccccc|}
        \toprule
        \textbf{Dataset} & \textbf{Classifier} & \makecell{\textbf{Best} $\alpha$} 
& \textbf{Naive} & \textbf{Raw} & \textbf{Calibrated} & \textbf{Best Base} & \textbf{Best Pairwise} & \textbf{Best kNN}\\
        \midrule 
Synthetic  & Linear & 0.47 & 73.52 (-14.53) & 88.05 & 85.63(-2.42) & - & \textbf{97.75(+9.70)} & 97.72(+9.67)\\
\midrule
\multirow{4}{*}{Spider} & SQLCoder &0.41 & 76.24(-18.62) & 94.86 & 98.11(+3.25) & - & \textbf{98.65(+3.79)} & 98.38(+3.52)\\
& Llama 3.1& 0.32 & 72.54(-21.88) & 95.41 & 96.76 (+1.22) &  96.22(+0.81)& \textbf{98.38(+2.97)} & 98.11(+2.70)\\
& Gemma 3 & 0.41 & 76.24(-19.17) & 95.41 & 97.03(+1.62) & 97.3(+1.89) &\textbf{98.38(+2.97)} &98.38(2.97)\\
& Qwen 3 & 0.32 & 72.54(-16.38) & 88.92 &  91.89(+2.97) & 90.0(+1.08) & \textbf{98.11(+9.19)} &95.14(+6.22) \\
\midrule 
\multirow{4}{*}{Bird} & SQCoder & 0.38 &  75.38(-10.73) &86.11&92.08(+5.97) & - & \textbf{95.29(+9.18)} & 95.06(+8.95)\\
& Llama 3.1 & 0.35 &74.16(-14.24) & 88.4 & 93.57(+5.17) & 
94.26(+5.86) & \textbf{96.67(+8.27)} & 95.75(+7.35)\\ 
& Gemma 3 & 0.35 & 74.16(-20.56) &94.72 & 95.98(+1.26)& 96.67(+1.95) & \textbf{97.59(+2.87)} &97.36(+2.64) \\
& Qwen 3 & 0.26 & 70.52(-2.27) & 72.79 & 82.43(+9.64) & 79.22(+6.43) & \textbf{88.63(+15.84)} & 85.88(+13.09) \\
\midrule
\multirow{3}{*}{BoolQ}& Llama 3.1& 0.41&  77.07(-2.80) & 79.87&  81.46(+1.59) &  87.96(+8.09) & \textbf{88.43(+8.56)} &  87.0(+7.13) \\ 
& Gemma 3 & 0.35 &  74.69 (-2.81) &  77.50 & 82.25(+4.75) &  83.04(+5.54) & \textbf{85.74(+ 8.24)} & 85.42(+7.92)\\
& Qwen 3 & 0.47 & 79.45(-2.96) & 82.41 & 86.85(+4.44) & 86.69(+4.28) &\textbf{88.91(+6.5)} & 88.11(+5.7) \\
\midrule
\multirow{3}{*}{\makecell{Vis-\\OnlyQA}} & Qwen 2.5VL & 0.44  & 72.38(-3.20) & 75.58 & 71.16 (-4.42)  & - & \textbf{79.30 (+3.72)} &  78.37 (+ 2.79)\\
& Gemma 3 & 0.47 & 73.89(+0.17) & 73.72 & 73.49(-0.23) &
- & \textbf{77.44(+3.72)} & 74.88(+1.16) \\
& GLM 4.1V & 0.22 & 61.78(+0.62) & 61.16 & 62.56(+1.4) & - &  \textbf{63.72(+2.56)} &63.02(+1.86) \\
\bottomrule
    \end{tabular}}}
    
    \caption{\blue{Total accuracy for different classifiers and datasets. For each row, the best $\alpha\in (0, 0.5]$ is where the difference between the performance of the raw baseline and best pairwise method is largest. For each baseline other than raw, we show the difference with the raw baseline in brackets. The largest difference in each row is \textbf{bold}.} }
    \label{tab:tot_acc}
\end{table}

From Fig~\ref{fig:cov_all}, for all our datasets and classifiers, we see that the best pairwise method consistently achieves the best performance across most values of $\alpha$. 

\textbf{Synthetic:} For the synthetic dataset with sigmoid link in Fig~\ref{fig:cov_all}(a), we can see that pairwise queries and kNN methods both lead to a large increase in performance as compared to either the raw baseline or the calibrated baseline. Surprisingly, the calibrated method decreases performance. This verifies our theoretical insight that as long as the confidence model $w_{\mathrm{clf},c}$ is far away from the labeling function, direct queries perform worse than pairwise methods. Further, from Table~\ref{tab:tot_acc}, pairwise method provides atmost $~10\%$ increase in performance.

\textbf{Spider and Bird:} Note that most of our classifiers are already extremely accurate on these datasets without any human labeling ($\approx 90\%$ total accuracy with $\alpha=0$). Therefore, the relative boost with selective classification is limited as seen by only $3\%$ maximum boost by pairwise methods for Spider in Table~\ref{tab:tot_acc}. Also, by the magnitude of the standard deviation in the Fig~\ref{fig:cov_all} (c) - (f), this performance boost can be nullified by the standard deviation. The harder Bird dataset has a larger boost due to pairwise methods ($\approx 16\%$ in Table~\ref{tab:tot_acc}).  A surprising observation is that for several cases, the raw confidences, $p_{\mathrm{clf},c}$, are so misaligned with the labelling function, $p_{\mathrm{clf},l}$, that the performance of the raw baseline remains almost constant even with increasing $\alpha$.

\textbf{BoolQ and VisOnlyQA:} For these datasets, we don't create negative samples. Further, these are harder for current classifiers, with the performance on VisOnlyQA close to a random classifier ($50\%$) without human labelling $\alpha=0.0$. For both these datasets, using a naive classifier is not much worse than the raw baseline (Table~\ref{tab:tot_acc}). However, these datasets are different. For BoolQ, we see a large boost due to pairwise methods ($\approx 8\%$), but for VisOnlyQA, we see only a minor improvement due to pairwise methods.
We suspect this is because VisOnlyQA is so hard that even pairwise queries are extremely noisy on it.

\textbf{Calibration:}
Note that for most combinations of datasets and classifiers, calibration improves performance over the raw baseline. However, this is not always the case, as for Synthetic, (Bird, Gemma 3), (BoolQ, Llama 3.1) and (VisOnlyQA, Qwen 2.5 VL), calibration decreases performance by a lot (see $-4.5\%$ for VisOnlyQA in Table~\ref{tab:tot_acc}). This is because the calibration technique~\citep{han2023prototypical} fits a mixture of $2$ Gaussians to the binary labels and confidences, which might not be a good fit in practice. Note that calibration is never the best baseline across all $\alpha$ for any dataset and classifier combination.

\textbf{Base:} Unlike calibration, base model confidences always improve performance whenever available. However, the main issue is  their availability, as they might not be available for custom classifiers like SQLCoder. They also never achieve the best performance among all baselines.

\textbf{kNN:}
Note that this is the second-best baseline as compared to pure pairwise methods. However, it's performance is always slightly worse than best pairwise. This decrease in performance indicates that use of the distance metric does not offer any explicit advantage over pairwise queries. We hypothesize that the pairwise queries that we compute and the distance metric are not independent of each other.

\subsubsection{\blue{Relative Performance of Different PairSel methods}}
\begin{table}[t!]
    \centering
    \resizebox{\textwidth}{!}{
    \setlength{\tabcolsep}{2pt}
{\color{black}
    \begin{tabular}{|c|cc|cccc|}
        \toprule
        \textbf{Dataset} & \textbf{Classifier} & \makecell{\textbf{Best} $\alpha$} 
& \makecell{\textbf{Best}\\ \textbf{Middle}} & \makecell{\textbf{Best} \\\textbf{Max}\\ \textbf{Presence}} & \makecell{\textbf{Best}\\ \textbf{Max}\\ \textbf{Displacement}} & \makecell{\textbf{Best}\\ \textbf{Max} \\\textbf{Entropy}} \\
        \midrule 
Synthetic  & Linear & 0.38 &  94.75(+6.96) & 87.79 & 93.91 (+6.12) & \textbf{95.16 (+7.37)} \\
\midrule
\multirow{4}{*}{Spider} & SQLCoder &0.32 & 97.57 & 97.57  &  \textbf{98.38 (+0.81)} & 97.84 (+0.27)\\
& Llama 3.1& 0.32 & 97.84 (+0.81) &  97.03 & 97.84 & \textbf{98.38(+1.35)}\\
& Gemma 3 & 0.47 & 98.11 (+0.81) & \textbf{98.65 (+1.35)} & 97.3 & 97.84(+0.54)\\
& Qwen 3 & 0.32 & 95.14(+0.55) & 94.86(+0.27) &  \textbf{98.11(+3.52)} &  94.59 \\
\midrule 
\multirow{4}{*}{Bird} & SQCoder & 0.41& \textbf{95.18(+1.26)} & 94.49 (+0.57) &  93.92 & 94.95(+1.03)\\
& Llama 3.1 & 0.32 &95.18 (+0.23) & 94.95& \textbf{96.21(+1.26)} & 94.95\\ 
& Gemma 3 & 0.16 & 95.98(+0.34) & 95.64 & \textbf{96.21(+0.57)} & 96.1(+0.46) \\
& Qwen 3 & 0.26 & 84.16 (+0.35) & 83.81&\textbf{88.63(+4.82)} &85.76(+1.95)\\
\midrule
\multirow{3}{*}{BoolQ}& Llama 3.1& 0.41&   87.0(+0.79)& 86.21 & \textbf{88.43(+2.22)} & 88.11(+1.90)\\ 
& Gemma 3 & 0.16 & 79.4& 79.56 (+0.16)& \textbf{80.35(+0.95)}& 79.56(+0.16)\\
& Qwen 3 & 0.47 & 87.0 & 87.64 (+0.64) & \textbf{88.91(+1.91)} & 87.8 (+0.8)\\
\midrule
\multirow{3}{*}{VisOnlyQA} & Qwen 2.5 VL & 0.44  &78.37(+0.7)&78.14(+0.47)&77.67 & \textbf{79.3(+1.63)}\\
& Gemma 3 & 0.47 & 74.19(1.4) & 76.98(+4.19)& \textbf{77.44(+4.65)} & 72.79\\
& GLM 4.1V & 0.35 &  69.77(+0.93)&68.84 & \textbf{70.47(+1.63)} & \textbf{70.47 (+1.63)} \\
\bottomrule
    \end{tabular}
    }}
    \caption{ \blue{Total accuracy of all pairwise methods for different classifiers and datasets. For each row, the best $\alpha\in (0, 0.5]$ is where the difference between the performance of any two of the pairwise methods is largest. For each row, the worst baseline is reported without brackets and for all other baselines, their difference in performance with respect to the worst baseline is shown inside the brackets.  The largest difference in each row is \textbf{bold}.}}
    \label{tab:tot_acc_pairwise}
\end{table}

\blue{The performance of different pairwise methods in PairSel(Algorithm~\ref{alg:pair_sel}) cannot be inferred from Fig~\ref{fig:cov_all}, \ref{fig:cost_all} or Table~\ref{tab:tot_acc}. To measure this relative performance, in Fig~\ref{fig:pair_comp}, we report the best performance of each pairwise method other than PairSel-kNN. Here, for each pairwise method, the best performance is among the different choices of labels and confidences. Further, in Table~\ref{tab:tot_acc_pairwise}, we compute the value of $\alpha \leq 0.5$, where the difference between the total accuracy of any two pairwise methods is maximum. We also compute the difference of each baseline with respect to the worst baseline for each combination of dataset and classifier.}

\blue{From Fig~\ref{fig:pair_comp} and Table~\ref{tab:tot_acc_pairwise}, we can see that the Max Displacement is most often the best performing pairwise method, followed by Max Entropy method. However, this conclusion does not hold across all datasets and classifiers as for some cases (Spider Llama 3.1 and Gemma 3, Bird SQLCoder and VisOnlyQA Qwen 2.5 VL), Max Displacement performs worst among all pairwise baselines. Note that Max Entropy is also the worst baseline for $3$ dataset classifier combinations ( Spider Qwen 3, Bird Llama 3.1 and VisOnlyQA Gemma 3). Note that Max Presence consistently is one of the worst performing pairwise methods due to its simple implementation and correlation to the raw baseline. However, for Spider Gemma 3, it does obtain the best performance. As for the Middle method, whose performance we investigate theoretically, its performance always lies between the best methods (Max Displacement, Max Entropy) and the worst method (Max Presence). Further, the difference in performance between pairwise methods is comparable to that between best pairwise methods and the raw baseline (see Tables~\ref {tab:tot_acc} and \ref{tab:tot_acc_pairwise}), so some pairwise methods do not fully take advantage of pairwise queries.}

\vspace{-2mm}
\section{Conclusion}
\label{sec:conclusion}
In this paper, we showed how to incorporate pairwise queries from a model to improve selective classification. We establish conditions under which pairwise queries obtain lower population risk than using a confidence estimate inconsistent with the labels. Through extensive experiments on in-context learning in LLMs, where the next-token logits are not aligned with the true labeling functions, we showed that our pairwise query-based algorithms consistently perform better than all baselines, especially using the raw baseline using next-token logits. A promising direction for future work is finding the optimal selective classification algorithm using pairwise queries in terms of cost or $\alpha$.

\backmatter

\bmhead{Acknowledgments} Part of the work was performed when HV was an intern at Adobe Research. In addition to this, the work was supported by an Adobe Research Gift and National Science Foundation award 2217058.

\bibliographystyle{apalike}
\bibliography{references}

\begin{appendices}
\section{Additional Experimental Details}
\label{sec:add_exp}

\subsection{Datasets}
\label{sec:datasets}

\paragraph{Synthetic} For this dataset $n=1000$ and $x_i \overset{iid}{\sim} \mathrm{Unif}(\mathbb{S}^{d-1})$ with $d=100$. Then, we sample $w^\star\sim \mathrm{Unif}(\mathbb{S}^{d-1})$. We assume that the labels $y_i \sim Ber(p^\star(x))$ where $p^\star(x) = \frac{1}{1 + \exp(-\beta \cdot \ip{w^\star}{x})}$ with $\beta = 4$. Here, the link function is sigmoid, which is $L$-smooth with $L=0.272$. We use the fact that if $g(a) = \frac{1}{1+\exp(-\beta a)}$, then $g'(a) = \beta g(a)(1-g(a))$ and $g''(a) = \beta^2 g(a)(1-g(a))(1-2g(a))$ with the value of $a\in [-1,1]$ as $a$ is the inner product of two unit vectors. To create $w_{v,l}$ and $w_{v,c}$, we sample two unit vectors, $u_1, u_2$,  from $\mathbb{S}^{d-1}$ perpendicular to $w^\star$ and set $w_{v,l} = \sqrt{1-\epsilon^2} w^\star + \epsilon u_1$ and $w_{v,c} = \sqrt{1-4\epsilon^2}w^\star + 2\epsilon u_2$, with $\epsilon=0.1$. The value $\epsilon$ allows us to directly control the values of $\ip{w^\star}{w_{v,l}}$ and $\ip{w^\star}{w_{v,c}}$ with $\ip{w^\star}{w_{v,c}} <\ip{w^\star}{w_{v,l}}$. The last condition ensures that the confidence estimate from the linear model is not accurate. For pairwise queries on two features $x_i$ and $x_j$, we compute the the sign of $p_{v,l}(x_i) - p_{v,l}(x_j) = g(\ip{w_{v,l}}{x_i}) - g(\ip{w_{v,l}}{x_j})$. Further, the distance metric between any two features $x_i$ and $x_j$ is $\norm{x_i - x_j}_2^2 = 2 - 2\ip{x_i}{x_j}$. This setting does not exactly match either of the two theoretical settings in Theorems~\ref{thm:spherical} and ~\ref{thm:gauss}, as we use Spherical features with Smooth Link. However, our theoretical results also hold for this setting.

\paragraph{Spider and Bird}To ensure that we test our baselines on sufficiently complicated SQL queries, we filter out the databases from each of these datasets such that they contain large number of tables, foreign keys and samples. For spider, the thresholds for number of tables, foreign keys and samples is $15, 7$ and $100$ respectively, which results in $3$ databases with $370$ samples in total. For Bird, the corresponding thresholds are $20,20$ and $150$ and we end up with $3$ databases totaling $919$ samples. Note that in each case, we only filter and label the training sets. For both these datasets, we use $60\%$ of the samples as positive samples. For the remaining $40\%$ of the samples, we need to pair a question with an SQL that does not answer the question. For this purpose, we shuffle the questions and their correct SQLs of these $40\%$ samples to obtain a derangement, i.e., each question is paired with the correct SQL of a different question. As there are no duplicate questions in the dataset, each question in the negative label set is mapped to an incorrect SQL. 

\

\paragraph{BoolQ and VisOnlyQA} For BoolQ dataset, to ensure that we have sufficiently complicated questions, we only consider those (Passage, Question) pairs where the Passage contains more than $175$ words. After this filtering, we end up with $631$ samples, of which $61.2\%$ are positive labels. The VisOnlyQA dataset contains single answer and multiple answer questions, of which the single answer questions contain both True/False answers and multiple choice questions. We only keep the single answer True/False questions to obtain a binary classification task resulting in $470$ samples of which $50\%$ are positive labels. For both these datasets, we convert the True label to $1$ and False to $0$. For BoolQ, we use the training dataset, and for VisOnlyQA, we use the Real Eval dataset.

\subsection{Baselines}
\label{sec:add_baselines}

    \begin{figure*}[t!]
    \centering
    \begin{subfigure}[t]{0.5\textwidth}
    \centering
    \includegraphics[height=2in, width=0.8\textwidth]{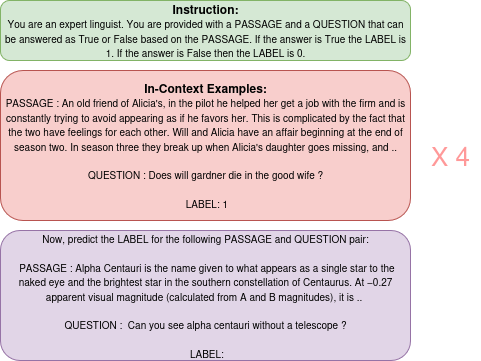}
    \end{subfigure}%
    \begin{subfigure}[t]{0.5\textwidth}
    \centering
    \includegraphics[height=2in, width=0.8\textwidth]{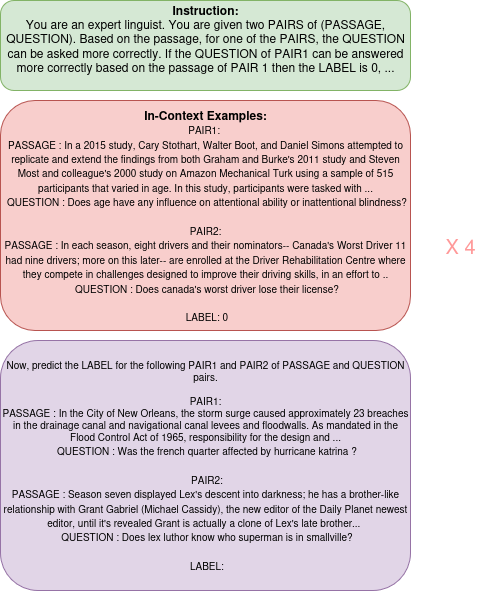}
    \end{subfigure}
    \caption{ICL Prompt for a direct query (left) and a pairwise query (right) from BoolQ dataset. The "X 4" in each image denotes that there are $4$ in-context examples of which only one has been presented. Prompt template derived from BoolQ~\citep[\href{ https://github.com/google-research-datasets/boolean-questions}{https://github.com/google-research-datasets/boolean-questions}]{clark2019boolq}, distributed under CC BY-SA 3.0.}
    \label{fig:prompt_boolq}
\end{figure*}

    \begin{figure*}[t!]
    \centering
    \begin{subfigure}[t]{0.5\textwidth}
    \centering
    \includegraphics[height=2in, width=0.8\textwidth]{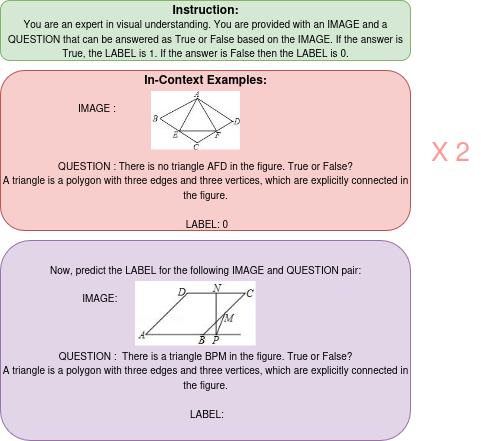}
    \end{subfigure}%
    \begin{subfigure}[t]{0.5\textwidth}
    \centering
    \includegraphics[height=2in, width=0.8\textwidth]{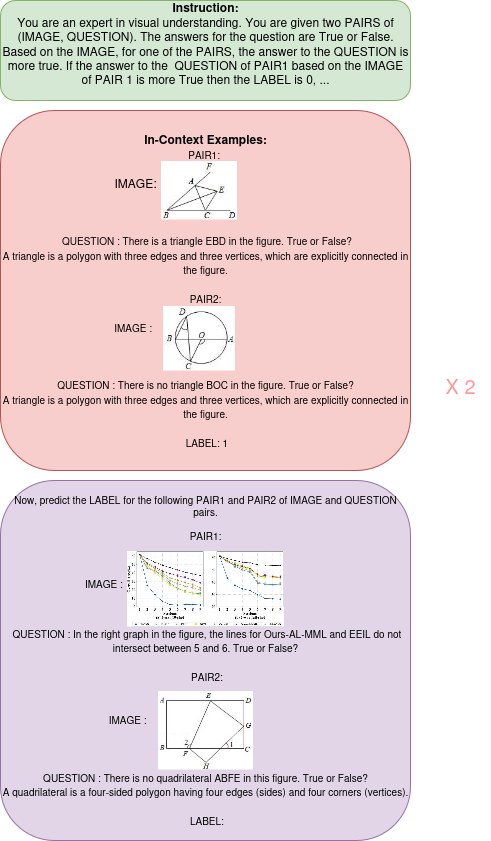}
    \end{subfigure}
    \caption{ICL Prompt for a direct query (left) and a pairwise query (right) from VisOnlyQA dataset. The "X 2" in each image denotes that there are $2$ in-context examples of which only one has been presented. Prompt template derived from BoolQ~\citep[\href{https://huggingface.co/datasets/ryokamoi/VisOnlyQA_Eval_Real_v1.1}{https://huggingface.co/datasets/ryokamoi/VisOnlyQA\_Eval\_Real\_v1.1}]{kamoi2025visonlyqa}, distributed under CC BY-SA 3.0.}
    \label{fig:prompt_visqa}
\end{figure*}

\textbf{Direct/Raw Baseline}
For each dataset, we remove an intial in-context example bank of $20$ samples. We keep adding the classifier-labeled examples to this in-context example bank. For Spider and Bird, for each in-context query, we supply the database schema followed by $4$ in-context examples of (NL, SQL, label) selected randomly from the in-context example bank. This random sampling assigns a weight of $1$ to each initial in-context example and $0.5$ to each example labeled by a classifier. See Fig~\ref{fig:prompt} for an example.
For BoolQ, we also use $4$ in-context examples of the form (Passage, Question, Label) as shown in Fig~\ref{fig:prompt_boolq}. For VisQA, we use $2$ in-context examples of the form (Image, Question, Label) for the prompt as shown in Fig~\ref{fig:prompt_visqa}.

\textbf{Calibrated baseline~\citep{han2023prototypical}}
 In this methods, a mixture of $2$ gaussians, one for each label, is fit to the raw confidences. Then, the calibrated label and calibrated confidence correspond to the mixture component and the mixture weights for each datapoint.

\textbf{Base model confidences}
In ~\citet[Figure~8]{openai2024gpt4}, the expected calibration error of base models is better than their instruct counterparts. However, instruct models are better at following instructions and thus, in-context learning~\citep{instruct,wei2022finetuned}. When the classifier is an instruct model with a corresponding base model, we use a direct query on the classifier, like Fig~\ref{fig:prompt},  to obtain its label. To obtain its confidence, we pass the full query sent to the classifier along with its label to the classifier's corresponding base model. Then, we perform a softmax on the logits for label tokens ($0$ and $1$) from the base model to obtain base model confidences. Note that the two possible methods using only base model confidences are obtained by applying ConfSel(Algoithm~\ref{alg:conf_sel}) with the labels being either raw or calibrated with the confidence obtained from the base models. The "best base" method reports the best performing of these methods.

\begin{figure*}[htbp]
\footnotesize
\vspace{-2mm}
    \centering
    \begin{subfigure}[t]{0.33\textwidth}
        \centering
        \includegraphics[height=1.17in]{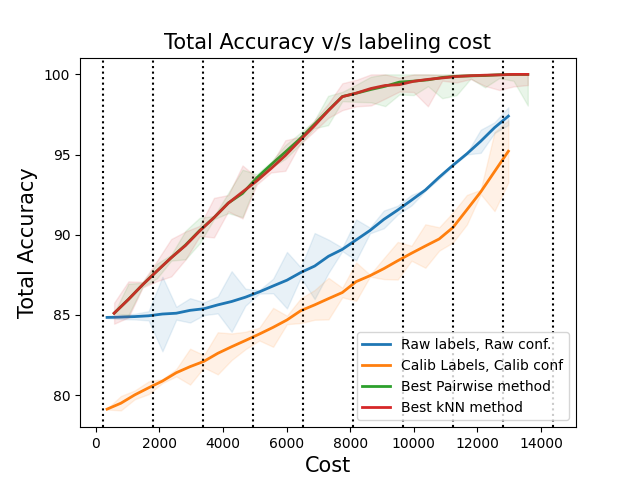}
        \caption{Synthetic, Linear}
    \end{subfigure}%
    \begin{subfigure}[t]{0.33\textwidth}
        \centering
        \includegraphics[height=1.17in]{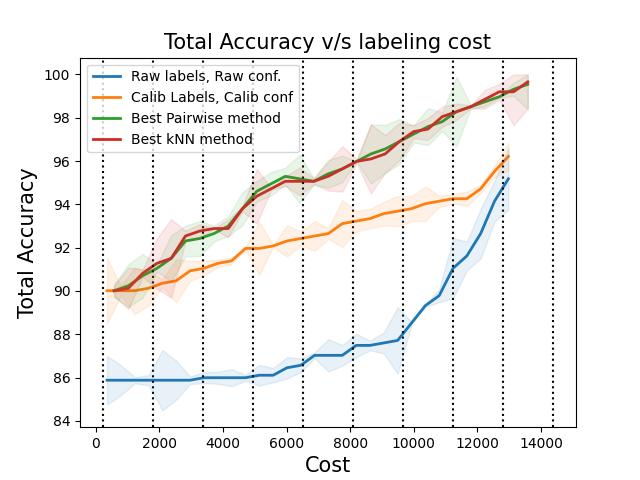}
        \caption{Bird, Llama3 SQLCoder}
    \end{subfigure}%
    \begin{subfigure}[t]{0.33\textwidth}
        \centering
        \includegraphics[height=1.17in]{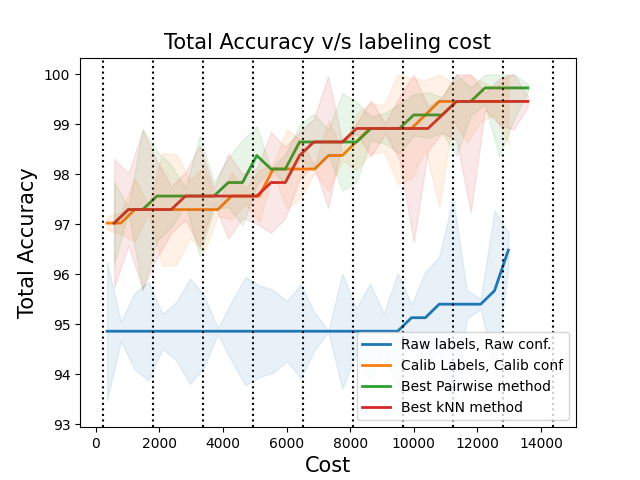}
        \caption{Spider, Llama3 SQLCoder}
    \end{subfigure}%
    \hfill
    \begin{subfigure}[t]{0.33\textwidth}
        \centering
        \includegraphics[height=1.17in]{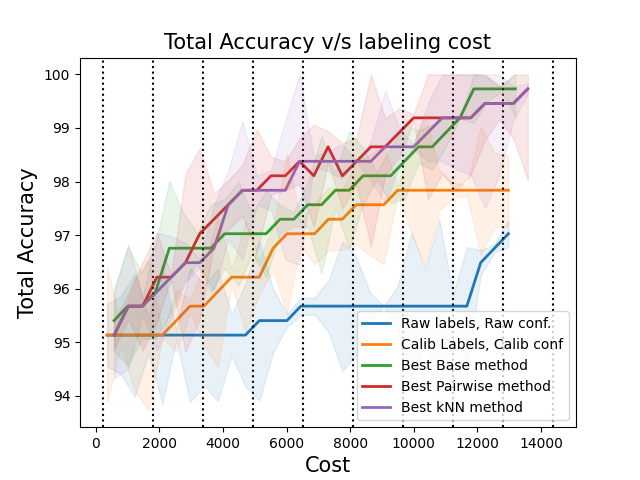}
        \caption{Spider, Gemma 3}
    \end{subfigure}%
    \begin{subfigure}[t]{0.33\textwidth}
        \centering
        \includegraphics[height=1.17in]{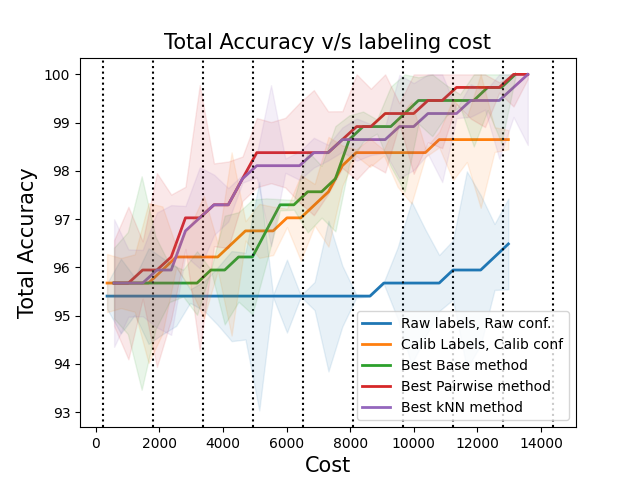}
        \caption{Spider, Llama 3.1}
    \end{subfigure}%
    \begin{subfigure}[t]{0.33\textwidth}
        \centering
        \includegraphics[height=1.17in]{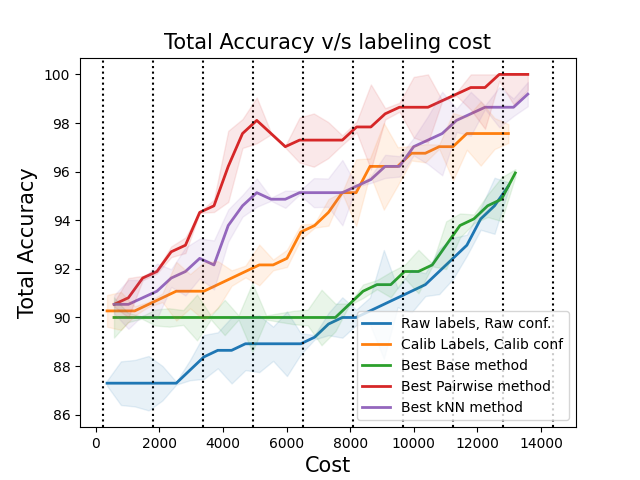}
        \caption{Spider, Qwen 3}
    \end{subfigure}%
    \hfill
    \begin{subfigure}[t]{0.33\textwidth}
        \centering
        \includegraphics[height=1.17in]{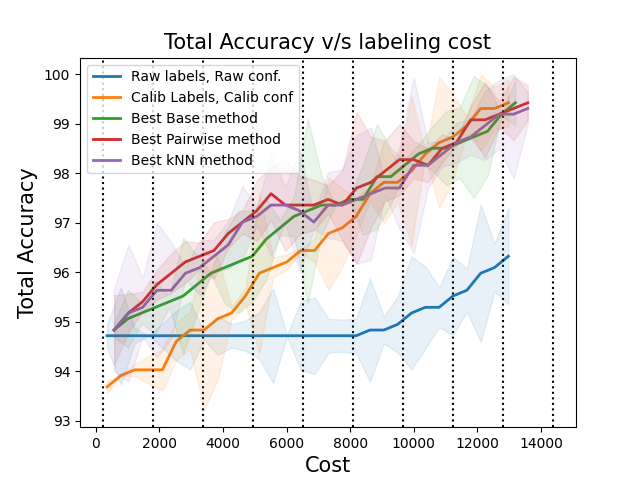}
        \caption{Bird, Gemma 3}
    \end{subfigure}%
    \begin{subfigure}[t]{0.33\textwidth}
        \centering
        \includegraphics[height=1.17in]{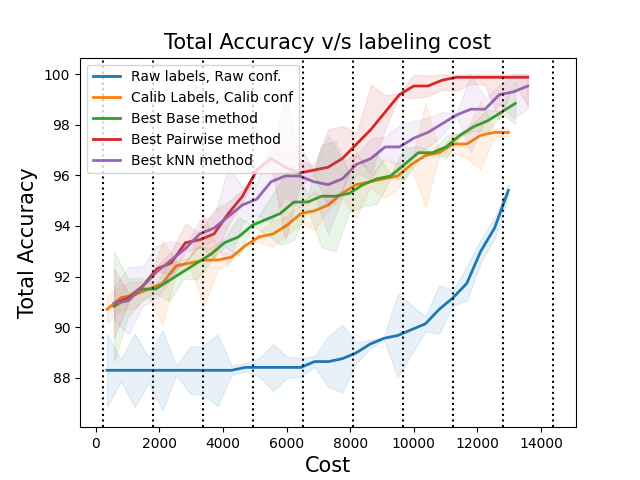}
        \caption{Bird, Llama 3.1}
    \end{subfigure}%
    \begin{subfigure}[t]{0.33\textwidth}
        \centering
        \includegraphics[height=1.17in]{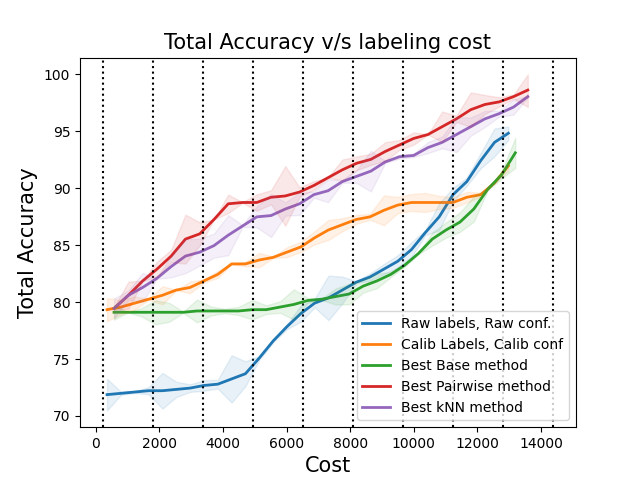}
        \caption{Bird, Qwen 3}
    \end{subfigure}%
    \hfill
    \begin{subfigure}[t]{0.33\textwidth}
        \centering
        \includegraphics[height=1.17in]{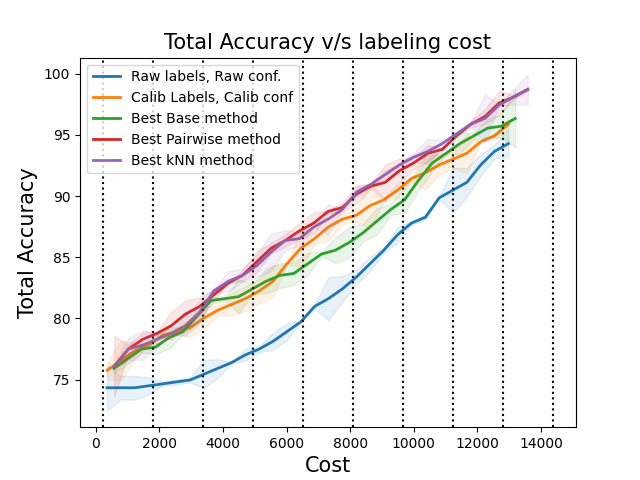}
        \caption{BoolQ, Gemma 3}
    \end{subfigure}%
    \begin{subfigure}[t]{0.33\textwidth}
        \centering
        \includegraphics[height=1.17in]{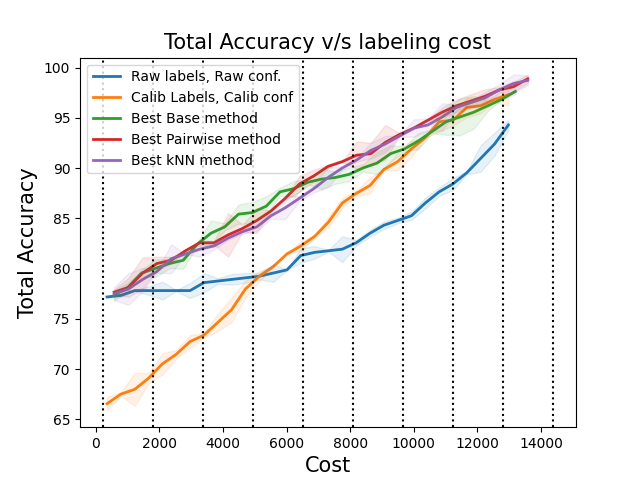}
        \caption{BoolQ, Llama 3.1}
    \end{subfigure}%
    \begin{subfigure}[t]{0.33\textwidth}
        \centering
        \includegraphics[height=1.17in]{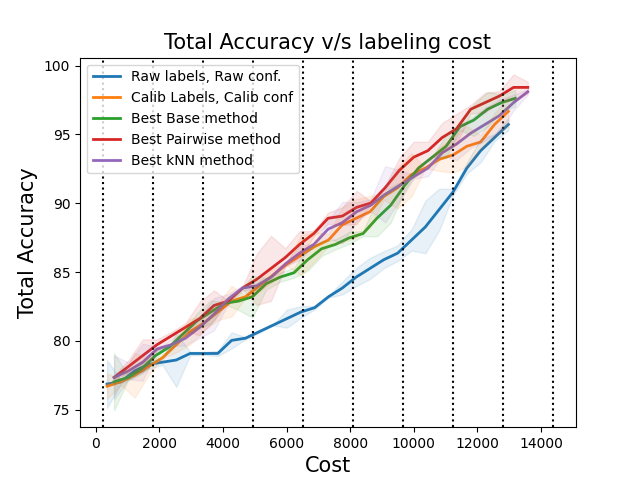}
        \caption{BoolQ, Qwen 3}
    \end{subfigure}%
    \hfill
    \begin{subfigure}[t]{0.33\textwidth}
        \centering
        \includegraphics[height=1.17in]{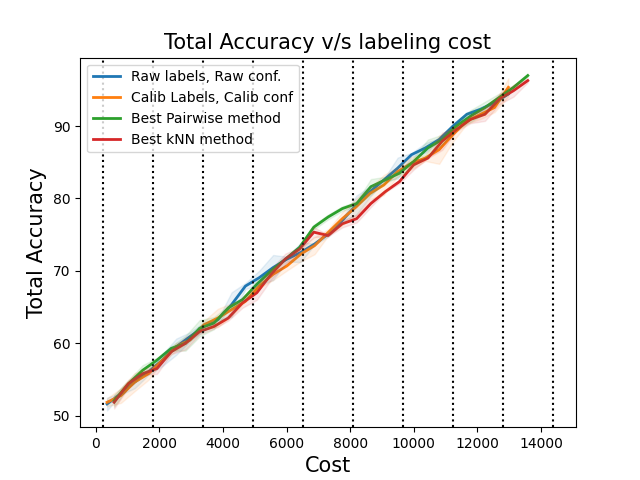}
        \caption{VisOnlyQA, Gemma 3}
    \end{subfigure}%
    \begin{subfigure}[t]{0.33\textwidth}
        \centering
        \includegraphics[height=1.17in]{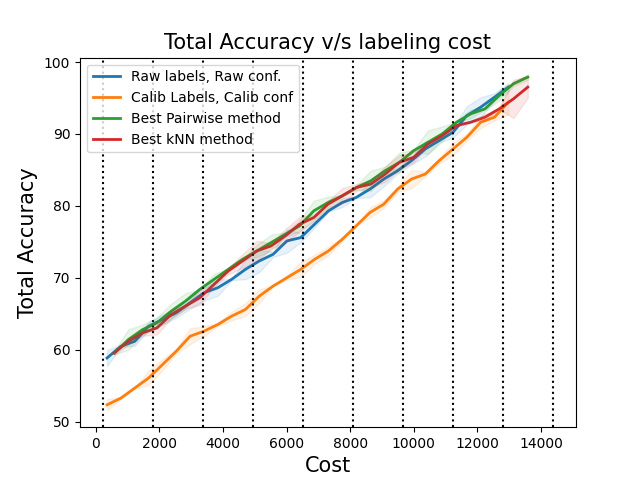}
        \caption{VisOnlyQA, Qwen 2.5 VL}
    \end{subfigure}%
    \begin{subfigure}[t]{0.33\textwidth}
        \centering
        \includegraphics[height=1.17in]{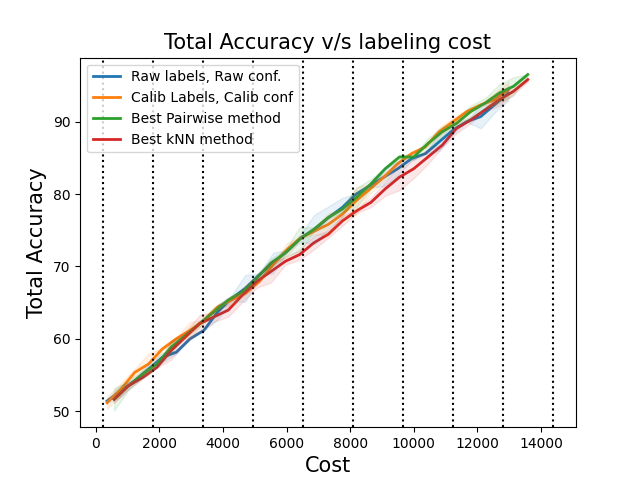}
        \caption{VisOnlyQA, GLM 4.1V}
    \end{subfigure}%

    \caption{Total Accuracy of classifier v/s cost for different datasets and classifiers. We only report the best-performing methods for each baseline (Pairwise, kNN, Base).}
    \label{fig:cost_all}
\end{figure*}

\textbf{Pairwise Queries}
For each pairwise query, we supply $4$ in-context examples of correct pairwise comparisons. These are created from the example bank by sampling an example labeled $1$ and another labeled $0$. Note that the example labeled $1$ is more correct than the example labeled $0$, which decides the label of the pairwise query. For the VisOnlyQA dataset, we use only $2$ in-context examples for each pairwise comparisons. The prompts for pairwise query for BoolQ and VisOnlyQA dataset are shown in Fig~\ref{fig:prompt_boolq} and ~\ref{fig:prompt_visqa}. Note that for each pairwise method in PairSel(Algorithm~\ref{alg:pair_sel}) other than PairSel-kNN, we have $2$ options for labels (raw or calibrated) and $2-3$ options for confidences (raw, calibrated and base). For certain classifiers (SQLCoder) and datasets (VisOnlyQA), base models and their corresponding confidences are not available. The "best pairwise" method is obtained by computing the best performing method among all the combinations of pairwise algorithm, label and confidence for each $\alpha$.

\textbf{PairSel-kNN}
For each dataset, we set $\lambda = 0.5$ as the relative weight of the NL/Image/Passage to the Question in Natural Language (NL).  To compute $\text{dist}_{\text{NL}}$, we compute embeddings for NL using `all-MiniLM-L6-V2' from Sentence Transformers~\citep{ni-etal-2022-sentence}. To compute $\text{dist}_{\text{SQL}}$ in Spider and Bird dataset, we compute the SQL embeddings using SQL Graph Neural Networks proposed in Appendix~\ref{sec:dist_metric}, and compute the cosine distance between embeddings. We set $k=10$ for k-Nearest Neighbors. Since we use the labels and confidences in PairSel, we again have $2$ options for labels (raw or calibrated) and $2-3$ options for confidences (raw, calibrated or base). The "best kNN" method computes the best performing method of PairSel-kNN among the different combinations of labels and confidences for each $\alpha$.

\textbf{Naive Classifier}
If $\beta\in [0.5,1)$ is the fraction of the dataset with the majority label, then, if $\alpha$ fraction of datapoints are sent for human labeling, the total accuracy of the naive classifier is $\beta * (1 - \alpha) + \alpha$. Note that $\beta =0.6$ for Spider and Bird datasets, while $\beta= 0.5, 0.507$ and $0.612$ for the Synthetic, VisOnlyQA and BoolQ datasets respectively.

\subsection{Metrics}
\label{sec:metrics}

Assuming that a direct query costs $1$ unit, a pairwise query costs $2$ units, and a human query costs $A\gg1$ units, we compute the cost for each algorithm in Table~\ref{tab:theory_cost}. By varying $\alpha$ from $0.01$ to $0.9$, with estimated $A=65$  we plot the total accuracy v/s cost for selective classification algorithms in Figures~\ref{fig:cost_all}.  In Table~\ref{tab:theory_cost}, we report the theoretical costs of all methods for any value of $A$. If $n$ is the total number of datapoints, then direct labeling costs $n$ units. Further, if we use base model confidences, then each query needs to be passed through the base model resulting in additional cost of $n$. For methods involving pairwise queries, we apply pairwise queries to only $n$ datapoints. Since we use pairwise queries to obtain a sorting, this incurs $n \log(n)$ cost. Combining base model confidences with pairwise queries incurs the highest cost. Note that using a distance metric in PairSel-kNN does not incur any additional cost apart from that due to sorting by pairwise queries and/or base model confidences.

To estimate $A$, we find the ratio of the cost of a direct query to that of an expert. For the direct query, we consider the most expensive classifier as GPT-4o, which costs $\$ 5 /1\text{M}$ tokens. In case there are additional instructions to the GPT-4o model, the total length of a direct query and its response is $\approx 2K$ tokens. Therefore, the cost of one direct query is atmost $\$0.02$. If the human is a product manager, with an average salary of $\$39/\text{hr}$~\citep{salary} and spends around $2$ minutes to verify a single datapoint, then the average human cost is $\$1.3$. Therefore, the ratio $A\approx 65$.

All our experiments except training SQL embeddings took around 2 weeks on one Nvidia A-10 GPU with $12$ cores and $20$ GB RAM. Training SQL embedding took 2 days for each seed on the same machine.

\begin{table}[t!]
    \centering
    \begin{tabular}{|c|c|c|c|}
        \toprule
        Method & \makecell{Cost of computing which \\
        points to label by human} & \makecell{Human \\
        Labeling\\ Cost} & Total Cost\\
        \midrule
        Direct  &  $n$ & $An\alpha$ & $n+An\alpha$\\
        Base Model & $2n$ & $An\alpha$ & $2n + An\alpha$ \\
        Pairwise Method & $n + n\log(n)$  & $An\alpha$ & $n + n\log(n) + An\alpha$\\
        Base Model + Pairwise & $2n + n\log(n)$ & $An\alpha$ & $2n + n\log(n) + An\alpha$\\
        \bottomrule
    \end{tabular}
    \caption{\small Labeling cost of each method assuming there are $n$ datapoints in total, of which $\alpha\in [0,1]$ fraction is sent for human labeling, where the cost of each classifier query is $1$ unit and of each human labeling is $A$ units. Note that using the distance metric doesn't incur any additional cost on top of pairwise queries or base model confidences.}
    \label{tab:theory_cost}
\end{table}

\subsection{Additional Experimental Results}

\subsubsection{Total accuracy vs. Cost of Labeling}
Note that the cost of base models or pairwise queries is negligible compared to human labeling cost due to the relatively large value of $A$ and small value of $n$ in Fig~\ref{fig:cost_all}. Therefore, the performance is mostly affected by the value of $\alpha$, due to which Fig~\ref{fig:cost_all} resembles Fig~\ref{fig:cov_all}, and the conclusions drawn from Fig~\ref{fig:cov_all} carry forward to Fig~\ref{fig:cost_all}. One key difference between cost and $\alpha$, is when the cost of pairwise queries or base models becomes non-negligible. Note that this is for larger values of $\alpha$. Conisder the cases of $\alpha$ close to $0.9$, the maximum value that we consider, we can see that raw and calibrated baselines end up at a smaller cost, while some additional and somewhat non-negligible cost is added to all other baselines due to pairwise queries or base model confidences.

\subsubsection{Choice of confidence and labels}
Note that for each of base, pairwise and kNN methods, we report only the best performing ones in Table~\ref{tab:tot_acc} and Fig~\ref{fig:cov_all}. For base and kNN methods this choice is between different labels and different confidences. The performance of these different combinations can be inferred from that of the raw and calibrated baselines. For instance, calibrated performs better than raw for most $\alpha$ values in Fig~\ref{fig:cov_all}, for all datasets and classifiers except Synthetic, (BoolQ, Llama3.1), and VisOnlyQA with Gemma 3 and Qwen 2.5VL. The performance of the base model confidence and kNN methods with calibrated labels and(or) confidences is also better than that with raw labels for only these datasets and classifiers.

\begin{figure*}[htbp]
\footnotesize
\vspace{-2mm}
    \centering
    \begin{subfigure}[t]{0.33\textwidth}
        \centering
        \includegraphics[height=1.17in]{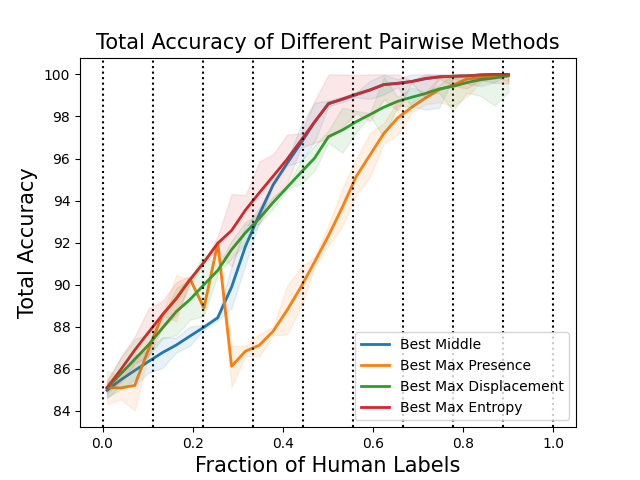}
        \caption{Synthetic, Linear}
    \end{subfigure}%
    \begin{subfigure}[t]{0.33\textwidth}
        \centering
        \includegraphics[height=1.17in]{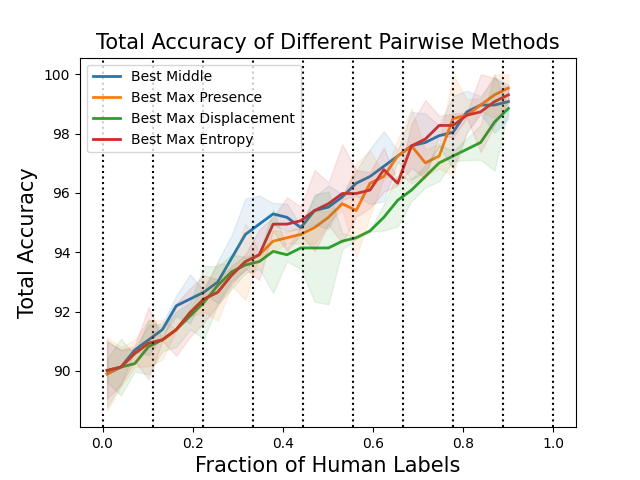}
        \caption{Bird, Llama3 SQLCoder}
    \end{subfigure}%
    \begin{subfigure}[t]{0.33\textwidth}
        \centering
        \includegraphics[height=1.17in]{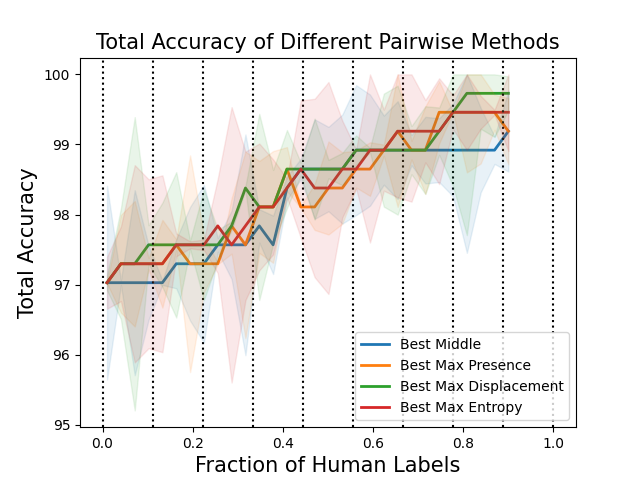}
        \caption{Spider, Llama3 SQLCoder}
    \end{subfigure}%
    \hfill
    \begin{subfigure}[t]{0.33\textwidth}
        \centering
        \includegraphics[height=1.17in]{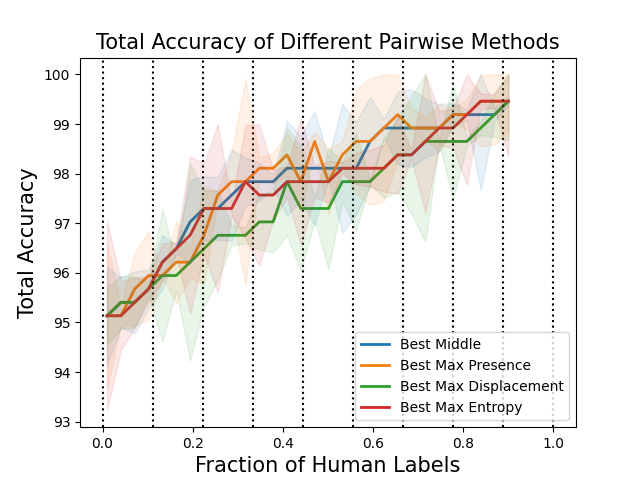}
        \caption{Spider, Gemma 3}
    \end{subfigure}%
    \begin{subfigure}[t]{0.33\textwidth}
        \centering
        \includegraphics[height=1.17in]{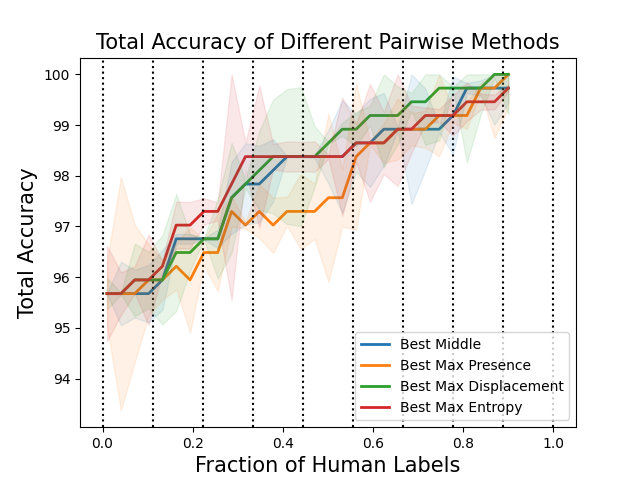}
        \caption{Spider, Llama 3.1}
    \end{subfigure}%
    \begin{subfigure}[t]{0.33\textwidth}
        \centering
        \includegraphics[height=1.17in]{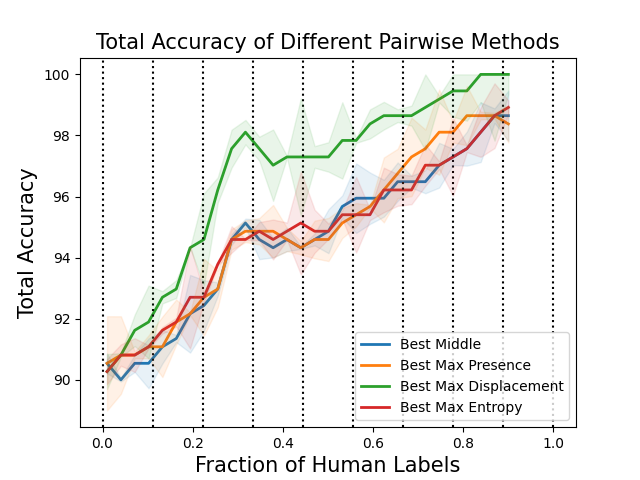}
        \caption{Spider, Qwen 3}
    \end{subfigure}%
    \hfill
    \begin{subfigure}[t]{0.33\textwidth}
        \centering
        \includegraphics[height=1.17in]{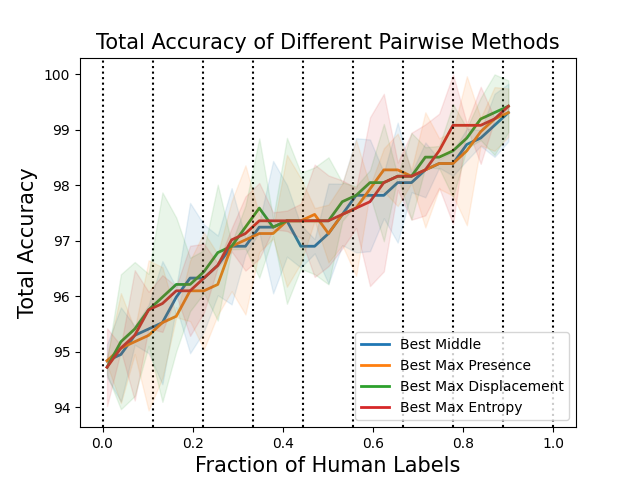}
        \caption{Bird, Gemma 3}
    \end{subfigure}%
    \begin{subfigure}[t]{0.33\textwidth}
        \centering
        \includegraphics[height=1.17in]{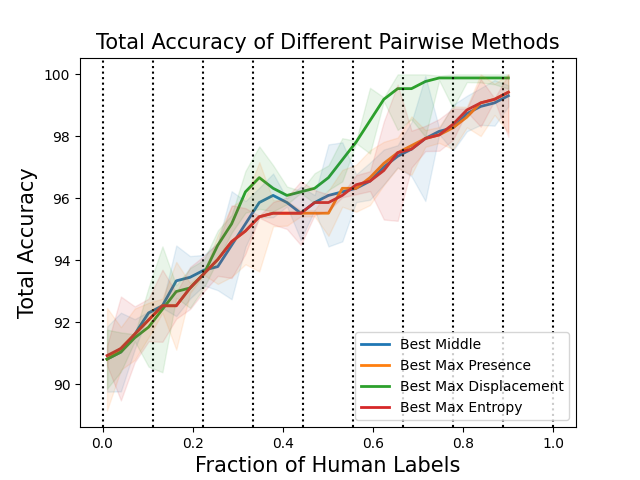}
        \caption{Bird, Llama 3.1}
    \end{subfigure}%
    \begin{subfigure}[t]{0.33\textwidth}
        \centering
        \includegraphics[height=1.17in]{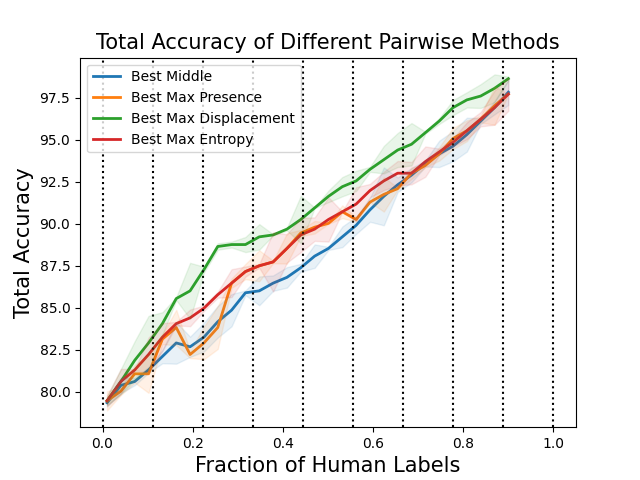}
        \caption{Bird, Qwen 3}
    \end{subfigure}%
    \hfill
    \begin{subfigure}[t]{0.33\textwidth}
        \centering
        \includegraphics[height=1.17in]{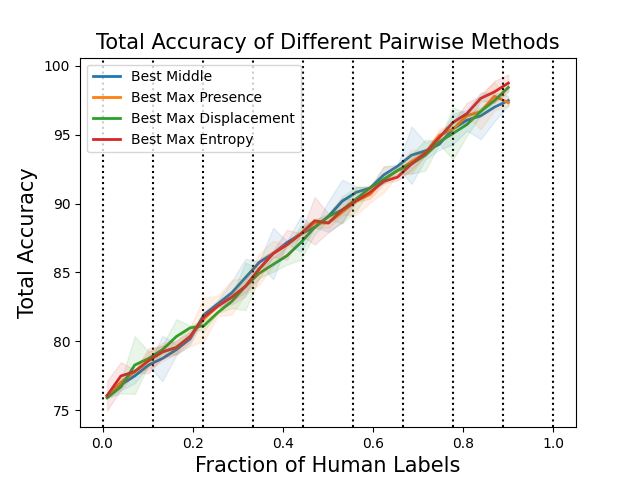}
        \caption{BoolQ, Gemma 3}
    \end{subfigure}%
    \begin{subfigure}[t]{0.33\textwidth}
        \centering
        \includegraphics[height=1.17in]{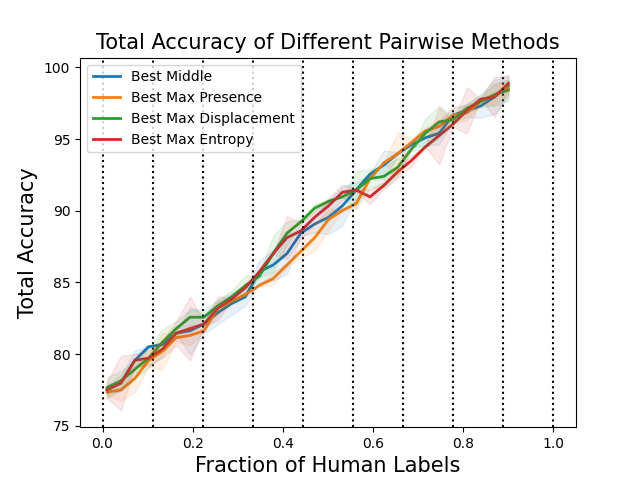}
        \caption{BoolQ, Llama 3.1}
    \end{subfigure}%
    \begin{subfigure}[t]{0.33\textwidth}
        \centering
        \includegraphics[height=1.17in]{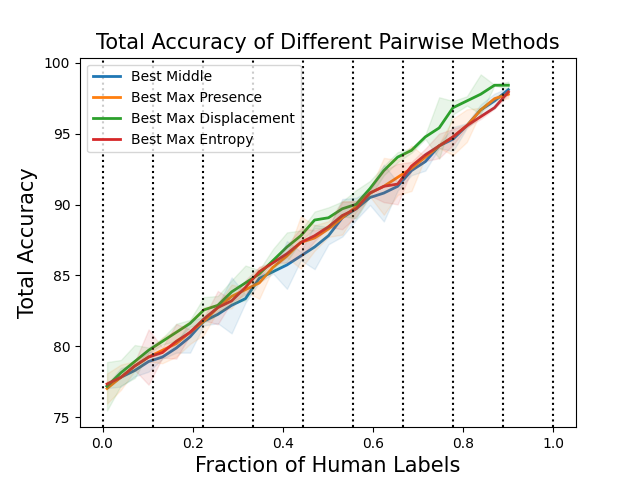}
        \caption{BoolQ, Qwen 3}
    \end{subfigure}%
    \hfill
    \begin{subfigure}[t]{0.33\textwidth}
        \centering
        \includegraphics[height=1.17in]{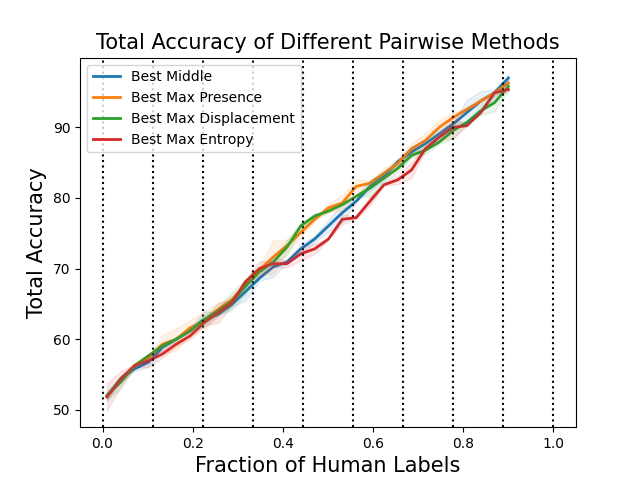}
        \caption{VisOnlyQA, Gemma 3}
    \end{subfigure}%
    \begin{subfigure}[t]{0.33\textwidth}
        \centering
        \includegraphics[height=1.17in]{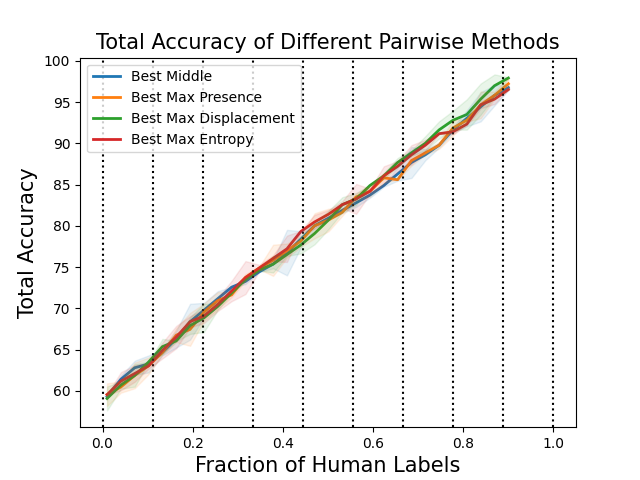}
        \caption{VisOnlyQA, Qwen 2.5 VL}
    \end{subfigure}%
    \begin{subfigure}[t]{0.33\textwidth}
        \centering
        \includegraphics[height=1.17in]{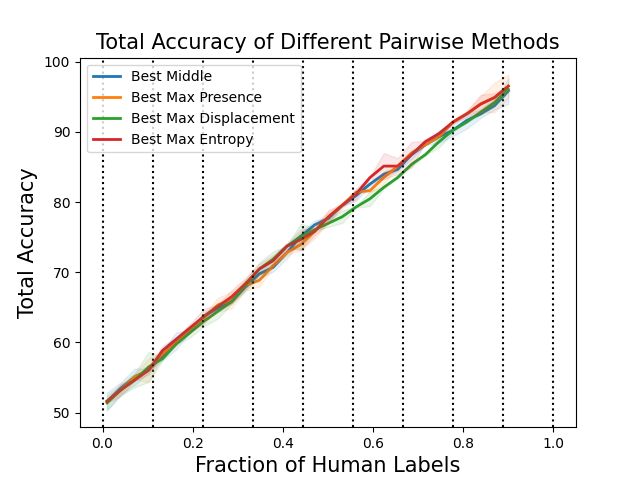}
        \caption{VisOnlyQA, GLM 4.1V}
    \end{subfigure}%

    \caption{\footnotesize Total Accuracy v/s fraction of human labels for different Pairwise methods for different datasets and classifiers. We only report the best performance for each pairwise method (Middle, Max Presence, Max Entropy, Max Displacement).}
    \label{fig:pair_comp}
    \vspace{-3mm}
\end{figure*}

\subsubsection{Parameter Sensitivity}
Note that none of the algorithms in ConfSel or PairSel have any parameters except PairSel-kNN, where we have $2$ parameters $k$, the number of nearest neighbors and $\lambda$, the relative importance of distance for question and passage/SQL/Image for BoolQ/Spider and Bird/VisOnlyQA datasets, respectively. We compare the total accuracy of  for different values of $ k=5, 20$ and $\lambda = 0.1, 1.0, 1.5$ in addition to the baseline of $\lambda=0.5, k=10$ for various combinations of datasets and classifiers.

For each of these combinations, we report the performance of best kNN method in Table~\ref{tab:param_sens}. Again, to measure the maximum sensitivity to parameters, we compute the best $\alpha \leq 0.5$, where the absolute difference between total accuracy of any hyperparameter combination and the baseline hyperparameter of $\lambda = 0.5, k=10$ is  maximum. We also measure the difference of each hyperparameter combination to the actual baseline hyperparameter choice of $\lambda = 0.5, k =10$.  Note that we do not include the Synthetic dataset in Table~\ref{tab:param_sens} as there is no notion of $\lambda$ defined for it.
\begin{table}[htbp]
    \centering
    \resizebox{\textwidth}{!}{
    \setlength{\tabcolsep}{2pt}
    \begin{tabular}{|c|cc|cccccc|}
        \toprule
        \textbf{Dataset} & \textbf{Classifier} & \makecell{\textbf{Best} $\alpha$} 
& \makecell{$\lambda = 0.1$\\ $k=10$} & \makecell{$\lambda=1.0$\\ $k=10$} & \makecell{$\lambda = 1.5$\\ $k=10$} & \makecell{$\lambda = 0.5$\\ $k=10$\\ \textbf{Baseline}} & \makecell{$\lambda = 0.5$\\ $k=5$} & \makecell{$\lambda = 0.5$\\ $k=20$}\\
        \midrule 
\multirow{4}{*}{Spider} & SQLCoder &0.32 & 97.57(+0) & 97.57(+0) & \textbf{98.11 (+ 0.54)} & 97.57& 97.57(+0)& 97.57(+0)\\
& Llama 3.1& 0.01 & \textbf{95.68(+0)}& \textbf{95.68(+0)} & \textbf{95.68(+0)} &\textbf{95.68} &\textbf{95.68(+0)} & \textbf{95.68(+0)}\\
& Gemma 3 &0.22 & \textbf{97.3(+0.54)} & 96.76(+0)&96.76(+0) &96.76 & 96.76(+0) &96.76(+0)\\
& Qwen 3 & 0.29 & 94.32(-0.27) & \textbf{94.59(+0)}&\textbf{94.59(+0)} &\textbf{94.59} &\textbf{94.59(+0)}&94.32 (-0.27)\\
\midrule 
\multirow{4}{*}{Bird} & SQCoder & 0.01& \textbf{90.01(+0)}&\textbf{90.01(+0)} &\textbf{90.01(+0)}&\textbf{90.01} & \textbf{90.01(+0)} &\textbf{90.01(+0)}\\
& Llama 3.1 & 0.35 &95.52(-0.23)&\textbf{95.75(+0)}&95.64(-0.11)&\textbf{95.75}& \textbf{95.75(+0)} & 95.64 (-0.11)\\ 
& Gemma 3 & 0.13  &\textbf{95.87(+0.23)} &95.64(+0) & 95.64(+0)&95.64&95.64(+0)&95.64, (+0)\\
& Qwen 3 & 0.01 & \textbf{79.45(+0)}&\textbf{79.45(+0)}&\textbf{79.45(+0)}&\textbf{79.45}& \textbf{79.45(+0)}&\textbf{79.45(+0)}\\
\midrule
\multirow{3}{*}{BoolQ}& Llama 3.1& 0.38&86.05(+0)&86.05(+0) & 86.05(+0)&86.05 & \textbf{86.21 (+0.16)}& 86.05(+0)\\ 
& Gemma 3 & 0.19  & 80.03(-0.48)&80.19(-0.32) & 80.19(-0.32) & \textbf{80.51}&80.19(-0.32) &  80.19 (-0.32)\\
& Qwen 3 & 0.01 & \textbf{77.34(+0)}& \textbf{77.34(+0)} & \textbf{77.34(+0)}& \textbf{77.34} & \textbf{77.34(+0)} & \textbf{77.34(+0)}\\
\midrule
\multirow{3}{*}{VisOnlyQA} & Qwen 2.5 VL & 0.19  &\textbf{67.67(+0.46)} & \textbf{67.67(+0.46)} & \textbf{67.67 (+0.46)} &67.21& \textbf{67.67 (+0.46)} & 67.21(+0)\\
& Gemma 3 & 0.01 &\textbf{51.86(+0)}&\textbf{51.86 (+0)}& \textbf{51.86 (+0)}& \textbf{51.86} &  \textbf{51.86 (+0)}& \textbf{51.86 (+0)}\\
& GLM 4.1V & 0.26 &  63.95 (+0)& 63.95(+0) & 63.95(+0)& 63.95& \textbf{65.12 (+1.17)} & 63.95(+0) \\
\bottomrule
    \end{tabular}
    }
    \caption{ Total accuracy of PairSel-kNN for several hyperparameters for different classifiers and datasets. For each row, the best $\alpha\in (0,0.5]$ is where the absolute difference between the performance any hyperparameter and the baseline $\lambda =0.5,k=10$ is the largest. For each row, the baseline is reported without brackets and for all other hyperparameters, their difference in performance with respect to the baseline is shown inside the brackets.  The highest total accuracy in each row is \textbf{bold}.}
    \label{tab:param_sens}
\end{table}
From Table~\ref{tab:param_sens}, we can see that PairSel-kNN is not sensitive to the choice of hyperparameters that we have taken, with either no difference in performance wrt baseline or very small difference. The only case where changing hyperparameters increases performance by a substantial amount compared to the boost in performance of the kNN method over the raw baseline in Table~\ref{tab:tot_acc} is for GLM 4.1V in the VisOnlyQA dataset, as their the kNN method itself does not improve performance by a lot. Further, the sensitivity to $k$ is much lower than that of $
\lambda$. Additionally, the low sensitivity to hyperparameters of PairSel-kNN, as well as its comparable performance to other pairwise methods, indicates that it is indeed the use of pairwise queries that provides the largest boost in performance.

\subsection{SQL Embeddings}
\label{sec:dist_metric}

\paragraph{Existing SQL Distance metrics}
Measuring the distance between SQLs is a key requirement for PairSel-kNN. If we can access a database for each SQL, we can execute the SQL and match the outputs of two SQLs. This is termed as execution accuracy~\citep{yu-etal-2018-spider}. While being accurate, this requires access to sample data for the query which might not be available due to privacy considerations. Extracting high quality sample data that obtains different outputs for different queries is also non-trivial~\citep{zhong-etal-2020-semantic}. Treating SQLs as strings, we can compute metrics based on string matching~\citep{crystal_bleu} or sentence embeddings~\citep{ni-etal-2022-sentence}, however these fail to incorporate the syntactic structure of SQLs. SQLs, and in general, scripts in most programming languages, can be represented as an abstract syntax tree (AST)~\citep{ast}, which is used by the compiler to execute the script. CodeBLEU~\citep{ren2020codebleumethodautomaticevaluation} is a string matching distance metric which incorporates the AST structure. Existing works incorporate information from the AST into neural architectures, but these are designed for different tasks : ~\citep{rat-sql} utilizes AST graph for schema linking, ~\citep{cao2023astormer} uses AST information in positional encodings to a transformer for NL2SQL with AST decoder and ~\citep{hg2ast} uses an AST-based decoder for NL2SQL.

\paragraph{Proposed SQL Embeddings}
We propose an embedding for SQL statements which uses the SQL's abstract syntax tree (AST). The AST of an SQL statement is a tree which mimics its execution by the compiler. Its nodes consist of SQL commands, table and column names and aliases. Its edges encode the logic and syntax of the SQL statement. In our case, we consider a directed AST with edges from the leaf nodes to parents. As this is a tree, it has a single root node with only incident edges, and a path from each leaf node to the root. Each SQL statement is parsed into its AST using sqlglot~\citep{sqlglot}. We further simplify this tree by resolving all ALIAS commands to their actual names.
In Figure~\ref{fig:ast}, we provide an example of a SQL command and its AST. Note that here the "SELECT" command is the root node.
We treat this AST graph as a Message Passsing Graph Neural Network~\citep{mpnn_intro}. In our GNN, there are two types of nodes -- a) SQL nodes, which contain SQL commands and operators and b) content nodes, which contain table and column names. Each node has a fixed node embedding vector obtained from a pre-trained model. For SQL nodes, we generate embeddings from a T5~\citep{ni-etal-2022-sentence} model fine-tuned on wikiSQL~\citep{t5_wikisql} dataset to encode SQL-specific information. Content nodes are arbitrary names, so we use a Sentence-T5 model~\citep{ni-etal-2022-sentence} to encode them. Each node's embedding vector has dimension $768$. To compute a single embedding from each graph, we perform message passing with topological propagation. In our topological propagation~\citep{top_sort}, each node is visited only after all nodes incident to it have been visited. We now show the message passing for a given node. For leaf nodes, the message is simply its node embeddings. For any non-leaf node $i$, the message $m_i$ is obtained by an attention-based combination of its node embeddings $v_i$ and the messages of its neighbors, $\{m_j\}_{j\in \mathrm{nghbr}(i)}$. Here, $\mathrm{nghbr}(i)$ is the set of neighbors of $i$. If $X_i\in \R^{(1 + \abs{\mathrm{nghbr}(i)})\times 768}$  is the stacked version of these vectors, i.e., $X_i = [v_i^\intercal, m_1^\intercal, m_2^\intercal, \ldots,]$, then $m_i = \mathrm{softmax}(\frac{Q X_i  (K X_i  )^\intercal}{\sqrt{768}})  V X_i $, where $Q, K, V\in \R^{768 \times 768}$ are the query, key and value attention matrices to be learned. Note that this aggregation scheme resembles that of graph attention networks~\citep{veličković2018graph}. Further, by topological propgation, we ensure that the message at root node has utilized combinations of all nodes in the tree. The message at root node $m_{root}$ is defined as the embedding of the SQL statement.
\begin{figure*}[htbp]
    \centering
    \includegraphics[width=0.6\textwidth]{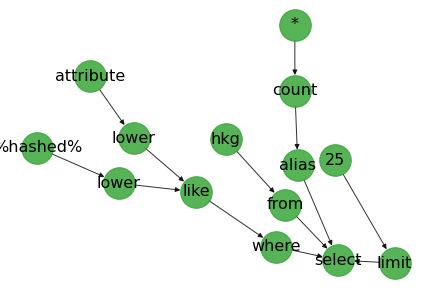}
    \caption{The Abstract Syntax Tree (AST) for the SQL statement : ``select count(*) as HashedEmails from hkg\_dim\_attribute where lower(attribute) like lower('\%hashed\%') limit 25"}
    \label{fig:ast}
\end{figure*}

Using this SQL embedding allows us to define a distance metric between different SQL statements and thus a distance metric for PairSel-kNN. We compare with a naive baseline of treating SQL as vanilla sentences and using sentence T5 model~\citep{ni-etal-2022-sentence} to obtain embeddings.
A concurrent work, FuncGMNEval ~\citep{zhan2024funcevalgmnevaluatingfunctionalcorrectness}, also utilizes SQL's syntactical structure for a GNN. Their input is a pair of SQL statements which are  converted to their ROTs (Relational Operational Trees), which represents SQL's execution plan based on relational algebra and are in general more representative of actual SQL execution than ASTs. Then, treating the ROTs of the two SQL statements as two different GNNs, ~\citep{zhan2024funcevalgmnevaluatingfunctionalcorrectness} train a GMN(Graph Matching Network)~\citep{pmlr-v97-li19d}. A GMN performs message passing on each node of the graph learning the node embeddings, edge embeddings and cross-graph attention weights to obtain a similarity score from their GMN.

\paragraph{Comparison to baselines for distance between SQLs.}
We train our method and funcGMNEval on Spider-Pair dataset provided by ~\citep{zhan2024funcevalgmnevaluatingfunctionalcorrectness}, which has pairs of SQL's which are labeled $0$ if they are similar and $1$ if they're dissimilar. We set the training time for both methods to be same, around ~1.5 days. To train our method on Spider-Pair dataset, we adapt the training procedure of FuncGMNEval, but obtain the embeddings for the two SQLs separately using our AST GNNs and use cosine similarity between these vectors as the similarity metric instead of FuncGMNEval's similarity score.

To compare different methods to compute distance metrics, we plot the AUC for the test dataset of Spider-Pair and the time taken to compute similarity of all pairs for the Bird dataset in Table~\ref{tab:dist_metric}. We compare our method, FuncGMNEval~\citep{zhan2024funcevalgmnevaluatingfunctionalcorrectness}, and vanilla Sentence Embeddings. Further, we compute the time taken for computing all distances for all pairs of Bird, as it is the largest dataset we use ($>900$ samples), and for PairSel-kNN, we need to compute distances between all pairs.

\begin{table*}  
    \centering
    \begin{tabular}{|c|cc|}
    \toprule
    Method & \makecell{All Pairs Inference\\
     Time on Bird} & \makecell{Test AUC on\\
      Spider-Pair}\\
    \midrule
    FuncGMNEval~\citep{zhan2024funcevalgmnevaluatingfunctionalcorrectness} 
    & 5 min & \textbf{0.801} \\
    Ours & 15 s & 0.798\\
    Sent Emb~\citep{ni-etal-2022-sentence} & \textbf{4s} & 0.639 \\
    \bottomrule
    \end{tabular}

    \caption{Test AUC on Spider-Pair and inference time for distance metric on all pairs of SQLs in Bird. Highest AUC and lowest inference time are \textbf{bold}.}
    \label{tab:dist_metric}
    \end{table*}

From Table~\ref{tab:dist_metric}, we can see that after similar training time, the performance of our method in terms of test AUC on the Spider-Pair dataset is slightly worse than the baseline FuncGMNEval. However, to compute the similarity between any two SQL statements, FuncGMNEval needs to perform a forward pass over the pair of SQL statements. In contrast, our method computes the embedding only once for each SQL statement separately, resulting in one forward pass per SQL statement. This makes our method much faster, by approximately an order of magnitude, when computing all pairs distance on the bird dataset
Further, the sentence embedding baseline is faster than our method and FuncGMNEval, but its test AUC is extremely low as it considers SQL statements as sentences. Overall, our method achieves a  balance between relative speed and performance. Thus, it can be utilized even in online settings, where the NL-SQL pair arrive one-by-one and we need to compute the distance between every new SQL statement and all existing SQL statements.

\section{Additional Theoretical Details}
\label{sec:add_theory}
We first compute $F_{\mathrm{clf}}$.
\begin{align*}
    f_{\mathrm{clf}}(n) &= \frac{1}{\abs{S}}\sum_{x_i \in S} \mathbf{1}\{\hat{y}_i \neq y_i^\star\}\\
    \E_{\hat{y}_i, y^\star_i}[f_{\mathrm{clf}}(n)] &= \frac{1}{\abs{S}} \sum_{x_i \in S} \E_{\hat{y}_i, y^\star}[\mathbf{1}\{\hat{y}_i \neq y_i^\star\}]\\
    &= \frac{1}{\abs{S}}\sum_{x_i \in S}(p_{\mathrm{clf},l}(x_i)(1 - p^\star(x_i))+ p^\star(x_i)(1 - p_{\mathrm{clf},l}(x_i)))\\
    &=\frac{1}{\abs{S}}\sum_{x_i \in S}(p_{\mathrm{clf},l}(x_i) + p^\star(x_i) - 2 p_{\mathrm{clf},l}(x_i)p^\star(x_i))\\
    F_{\mathrm{clf}} &= \E_{x\sim \gD_x}[\E_{\hat{y}_i, y_i^\star}[f_{\mathrm{clf}}(n)]] = \E_{x\sim \gD_x}[p_{\mathrm{clf},l}(x) + p^\star(x) - 2 p_{\mathrm{clf},l}(x)p^\star(x)]
\end{align*}
Similarly, we can compute $F_h$.
\begin{align*}
    f_h(n) \triangleq& \frac{1}{\abs{S}}\sum_{x_i\in S} \mathbf{1}\{\tilde{y}_i \neq y_i^\star\}\\
    \E_{\tilde{y}_i^\star, y_i^\star}(f_h(n)) =& \frac{1}{\abs{S}}\sum_{x\in S}(p_h(x) + p^\star(x) - p_h(x)p^\star(x))\\
    F_h =& \E_{x\sim \gD_x}[p_h(x) + p^\star(x) - 2p_h(x)p^\star(x)]
\end{align*}
Further, by Chernoff bound based on iid sampled datapoints $x_i\sim \gD_x$, with probability $1-\delta$, we have,
\begin{align}\label{eq:chernoff_f}
    \max\{\abs{f_{\mathrm{clf}}(n) - F_{\mathrm{clf}}},\,\, \abs{f_h(n) - F_h}\} = \mathcal{O}\left(\sqrt{\frac{2\log(2/\delta)}{n}}\right)
\end{align}

\subsection{Sufficient Conditions for Pairwise Queries to be Beneficial}
We need to show $F_{\mathrm{clf}, c}(\alpha) \geq F_{\mathrm{clf}, p}(\alpha)$ to show that pairwise queries are advantageous. Here, we establish certain sufficient conditions independent of the actual models for $p_{\mathrm{clf},c}, p^\star, p_{\mathrm{clf},l}$ and $p_h$ and the data distribution $\gD_x$.

We define $h_{c}, h_{p}:\gX \to [0,1]$ as the following.
\begin{align*}
h_{c}(x) &\triangleq \Pr_{x'\sim \gD_x}[\abs{2p_{\mathrm{clf},c}(x') - 1} \geq \abs{2p_{\mathrm{clf},c}(x) - 1}]\\
 h_{p}(x)& \triangleq \sqrt{\Pr_{x'\sim \gD_x}[p_{\mathrm{clf},l}(x')\leq p_{\mathrm{clf},l}(x)]\cdot \Pr_{x'\sim \gD_x}[p_{\mathrm{clf},l}(x')\geq p_{\mathrm{clf},l}(x)]}   
\end{align*}

\begin{proposition}[Sufficient Conditions for Improvement]\label{rem:suff}
Under the following conditions, $F_{\mathrm{clf},p}(\alpha) \leq F_{\mathrm{clf},c}(\alpha)\forall \alpha$
\begin{enumerate}
    \item $\forall x\in \gX$,\,\, if $p^\star(x) \geq \frac{1}{2},\,\, p_{h}(x) \geq p_{\mathrm{clf},l}(x)$, otherwise $p_h(x) \leq p_{\mathrm{clf},l}(x)$.
    \item $\forall x\in \gX,\,\, h_c(x) \leq h_{p}(x)$
\end{enumerate}
\end{proposition}

The first condition implies that at each datapoint, the human is a better labeler than the classifier, and the second relates the accuracy of $p_{\mathrm{clf},c}$ to the actual prediction of the labeler. Note that if $p_{\mathrm{clf},c}(x) = p_{\mathrm{clf},l}(x), \forall x\in \gX$, then the second condition can never be satisfied.
The consequence of this assumption is that $(2p^\star(x) - 1)(p_h(x) - p_{\mathrm{clf},l}(x))\geq 0$ and $(h_{p}(x))^{(1-\alpha)n} - (h_{c}(x))^{(1-\alpha)n}\geq 0$ for all $x\in \gX$, therefore, their product is always positive. Note that these conditions are very strict and are not satisfied by the models used in Assumption~\ref{assump:lin_model} and Theorems~\ref{thm:spherical} and \ref{thm:gauss}.

\section{Proofs}
\label{sec:proofs}

\subsection{Proofs of Lemmas~\ref{lem:conf_err} and \ref{lem:pair_err}}
\label{sec:lemma_proofs}

\paragraph{Proof of Lemma~\ref{lem:conf_err}}
Note that if a datapoint $x_i\in S_{\alpha, c}, i\in [n]$, for all datapoints $j\in S\setminus S_{\alpha, c}$, we should have $\abs{2p_{\mathrm{clf},c}(x_j) - 1} \geq \abs{2 p_{\mathrm{clf},c}(x_i) - 1}$. Note that if $x_i \in S_{\alpha,c}$, then we use the label $\hat{y}_i\sim p_{\mathrm{clf},l}(x_i)$, otherwise, we use the label $\tilde{y}_i \sim p_{h}(x_i)$. Therefore, 
\begin{align*}
   F_{\mathrm{clf}} -  F_{\mathrm{clf},c}(\alpha)&= \frac{1}{n}\sum_{i=1}^n \E_{x_i \sim \gD_x}\bigg[\vone\{x_i\in S_{\alpha, c}\} \cdot  (2p^\star(x_i) - 1)(p_h(x_i) - p_{\mathrm{clf},l}(x_i)\bigg]
\end{align*}
Note that each term of the expectation is equal as $x_i$ are iid. Further, $\Pr[x_i\in S_{\alpha, c}]  = (\Pr_{x}[\abs{2p_{\mathrm{clf},c}(x) - 1} \geq \abs{2p_{\mathrm{clf},c}(x_i) - 1}])^{(1-\alpha)n}$, as the other $(1-\alpha)n$ points are sampled independently from $\gD_x$.

\paragraph{Proof of Lemma~\ref{lem:pair_err}}
    The proof follows that of Lemma~\ref{lem:conf_err} with $S_{\alpha, c}$ replaced by $S_{\alpha, p}$. Note that an element $x_i\in S_{\alpha, p}$ if and only if there are two sets of points $S_1, S_2$, where $S_1 \cap S_2 = \phi, S_1 \cup S_2 = S \setminus S_{\alpha, p}$ and $\abs{S_1} = \abs{S_2} = \frac{(1-\alpha)n}{2}$, such that  for each point $x_j\in S_1$ $p_{\mathrm{clf},l}(x_j) \geq p_{\mathrm{clf},l}(x_i)$ and for each point $x_j\in S_2, p_{\mathrm{clf},l}(x_j) \leq p_{\mathrm{clf},l}(x_i)$. $S_1$ and $S_2$ correspond to the sets $\{x_{(i)}, i \in [0, \frac{(1-\alpha)n}{2}]\}$ and $\{x_{(i)}, i \in [\frac{n(1+\alpha)}{2}, n]\}$, the two parts of the sorted values on either side of the selected "middle" portion. Therefore, $$\Pr[x_i\in S_{\alpha, p}] = \left(\sqrt{\Pr_{x'\sim \gD_x}[p_{\mathrm{clf},l}(x')\leq p_{\mathrm{clf},l}(x)]\cdot \Pr_{x'\sim \gD_x}[p_{\mathrm{clf},l}(x')\geq p_{\mathrm{clf},l}(x)]}\right)^{(1-\alpha)n},$$ as each datapoint is sampled iid from $\gD_x$.

\subsection{Proofs of Theorems~\ref{thm:spherical} and \ref{thm:gauss}}
\label{sec:thm_proofs}
\subsubsection{Preliminaries}
\label{sec:preliminaries}
In this section, we simplify the computation of $F_{\mathrm{clf}} - F_{\mathrm{clf},c}(\alpha)$ and $F_{\mathrm{clf}} - F_{\mathrm{clf},p}(\alpha)$ by computing the terms  $\Pr_{x'\sim \gD_x}[\abs{2p_{\mathrm{clf},c}(x')-1} \geq \abs{2p_{\mathrm{clf},c}(x)-1}]$ and $\\\sqrt{\Pr_{x'\sim \gD_x}[p_{\mathrm{clf},l}(x')\leq p_{\mathrm{clf},l}(x)]\cdot \Pr_{x'\sim \gD_x}[p_{\mathrm{clf},l}(x')\geq p_{\mathrm{clf},l}(x)]}$ in Lemmas~\ref{lem:conf_err} and ~\ref{lem:pair_err} respectively.

\begin{lemma}\label{lem:simpl}
Under Assumption~\ref{assump:lin_model} and access to a Pairwise Query Oracle, if $p_{\mathrm{clf},c}(x')$ and $p_{\mathrm{clf},l}(x')$ are continuous distributions symmetric around $\frac{1}{2}$ for $x'\sim \gD_x$, $\forall x\in \gX$, we have,
\begin{align*}
    &    \Pr_{x'\sim \gD_x}[\abs{2p_{\mathrm{clf},c}(x')-1}\geq \abs{2p_{\mathrm{clf},c}(x)-1}] = \Pr_{x'}[\abs{\ip{w_{\mathrm{clf},c}}{\varphi(x')}} \geq \abs{\ip{w_{\mathrm{clf},c}}{\varphi(x)}}]\\
    &\sqrt{\Pr_{x'\sim \gD_x}[p_{\mathrm{clf},l}(x')\leq p_{\mathrm{clf},l}(x)]\cdot \Pr_{x'\sim \gD_x}[p_{\mathrm{clf},l}(x')\geq p_{\mathrm{clf},l}(x)]} \geq \Pr_{x'}[\ip{w_{\mathrm{clf},l}}{\varphi(x')} \geq \abs{\ip{w_{\mathrm{clf},l}}{\varphi(x')}}]
\end{align*}
\end{lemma}
\begin{proof}
Consider the first inequality. Either $p_{\mathrm{clf},c}(x)\geq \frac{1}{2}$ or $p_{\mathrm{clf},c} <\frac{1}{2}$. If $p_{\mathrm{clf},c} \geq \frac{1}{2}$, we need $\abs{2p_{\mathrm{clf},c}(x') -1} \geq 2p_{\mathrm{clf},c}(x) - 1$ which implies $p_{\mathrm{clf},c}(x') \geq p_{\mathrm{clf},c}(x)$ or $p_{\mathrm{clf},c}(x') \leq 1 - p_{\mathrm{clf},c}(x)$. If $p_{\mathrm{clf},c}(x) < \frac{1}{2}$, we need $\abs{2p_{\mathrm{clf},c}(x') -1} \geq 1 - 2p_{\mathrm{clf},c}(x)$, which implies $p_{\mathrm{clf},c}(x') \geq 1 - p_{\mathrm{clf},c}(x)$ or $p_{\mathrm{clf},c}(x') \leq p_{\mathrm{clf},c}(x)$. 
\begin{align*}
    &\Pr_{x'\sim \gD_x}[\abs{2p_{\mathrm{clf},c}(x')-1}\geq \abs{2p_{\mathrm{clf},c}(x)-1}] \\
    &= \vone\{p_{\mathrm{clf},c}(x)\geq \frac{1}{2}\}(\Pr_{x'}[p_{\mathrm{clf},c}(x') \geq p_{\mathrm{clf},c}(x)] + \Pr_{x'}[p_{\mathrm{clf},c}(x') \leq 1 - p_{\mathrm{clf},c}(x)])\\
    &\quad +\vone\{p_{\mathrm{clf},c}(x) < \frac{1}{2}\}(\Pr_{x'}[p_{\mathrm{clf},c}(x') \leq p_{\mathrm{clf},c}(x)] + \Pr_{x'}[p_{\mathrm{clf},c}(x') \geq 1 - p_{\mathrm{clf},c}(x)])
\end{align*}
If $p_{\mathrm{clf},c}(x)$ is symmetric about $\frac{1}{2}$, then $\Pr_{x'}[p_{\mathrm{clf},c}(x') \geq a] = \Pr_{x'}[p_{\mathrm{clf},c}(x') \leq 1-a]$ for any $a\in (0,1)$. This implies the following equation.
\begin{align*}
    &\Pr_{x'\sim \gD_x}[\abs{2p_{\mathrm{clf},c}(x')-1}\geq \abs{2p_{\mathrm{clf},c}(x)-1}]\\
    &= \vone\{p_{\mathrm{clf},c}(x) \geq \frac{1}{2}\} 2\Pr_{x'}[p_{\mathrm{clf},c}(x') \geq p_{\mathrm{clf},c}(x)] +\vone\{p_{\mathrm{clf},c}(x) < \frac{1}{2}\}2\Pr_{x'}[p_{\mathrm{clf},c}(x') \leq p_{\mathrm{clf},c}(x)]
\end{align*}
From Assumption~\ref{assump:lin_model}, if $\ip{w_{\mathrm{clf},c}}{\varphi(x)} \leq 0$, $g(\ip{w_{\mathrm{clf},c}}{\varphi(x)}) \leq \frac{1}{2}$ and if $\ip{w_{\mathrm{clf},c}}{\varphi(x)} \geq 0$, $g(\ip{w_{\mathrm{clf},c}}{\varphi(x)}) \geq \frac{1}{2}$. Therefore, we have the following condition. 
\begin{align*}
    \Pr_{x'\sim \gD_x}[\abs{2p_{\mathrm{clf},c}(x')-1}\geq \abs{2p_{\mathrm{clf},c}(x)-1}] = \Pr_{x'}[\abs{\ip{w_{\mathrm{clf},c}}{\varphi(x')}} \geq \abs{\ip{w_{\mathrm{clf},c}}{\varphi(x)}}]
\end{align*}

Now, consider the second equation in the Lemma. For continuous distributions $p_{\mathrm{clf},l}(x'), x'\sim \gD_x$, we have, 
\begin{align*}
    \Pr_{x'}[p_{\mathrm{clf},l}(x') \leq p_{\mathrm{clf},l}(x)] = 1 - \Pr_{x'}[p_{\mathrm{clf},l}(x') \geq p_{\mathrm{clf},l}(x)]
\end{align*}
Further, if $p_{\mathrm{clf},l}(x')$ is symmetric about $\frac{1}{2}$, then, 
\begin{align*}
p_{\mathrm{clf},l}(x) \geq\frac{1}{2} \quad &\implies \Pr_{x'}[p_{\mathrm{clf},l}(x')\geq p_{\mathrm{clf},l}(x)] = \frac{1}{2} - \Pr_{x'}[p_{\mathrm{clf},l}(x') \in[\nicefrac{1}{2}, p_{\mathrm{clf},l}(x)]], \\
&\implies \Pr_{x'}[p_{\mathrm{clf},l}(x')\leq p_{\mathrm{clf},l}(x)] = \frac{1}{2} + \Pr_{x'}[p_{\mathrm{clf},l}(x') \in[\nicefrac{1}{2}, p_{\mathrm{clf},l}(x)]], \\
p_{\mathrm{clf},l}(x) < \frac{1}{2} \quad &\implies \Pr_{x'}[p_{\mathrm{clf},l}(x')\geq p_{\mathrm{clf},l}(x)] = \frac{1}{2} + \Pr_{x'}[p_{\mathrm{clf},l}(x') \in[p_{\mathrm{clf},l}(x), \nicefrac{1}{2}]], \\
&\implies \Pr_{x'}[p_{\mathrm{clf},l}(x')\leq p_{\mathrm{clf},l}(x)] = \frac{1}{2} - \Pr_{x'}[p_{\mathrm{clf},l}(x') \in[p_{\mathrm{clf},l}(x)],\nicefrac{1}{2}] \\
\end{align*}
Let $P(x) \triangleq \Pr_{x'}[p_{\mathrm{clf},l}(x') \in[\nicefrac{1}{2}, \max\{p_{\mathrm{clf},l}(x), 1-p_{\mathrm{clf},l}(x)]\}]$, then,

\begin{align*}
    &\sqrt{\Pr_{x'\sim \gD_x}[p_{\mathrm{clf},l}(x')\leq p_{\mathrm{clf},l}(x)]\cdot \Pr_{x'\sim \gD_x}[p_{\mathrm{clf},l}(x')\geq p_{\mathrm{clf},l}(x)]}\\
    &= \sqrt{\left(\frac{1}{2} - P(x)\right)\left(\frac{1}{2}- P(x)\right)}=\sqrt{\frac{1}{4} - (P(x))^2}\\
    &\geq \frac{1}{2} - P(x) = \Pr_{x'}[p_{\mathrm{clf},l}(x') \geq \max\{p_{\mathrm{clf},l}(x), 1- p_{\mathrm{clf},l}(x)]
\end{align*}
From Assumption~\ref{assump:lin_model}, using the symmetry of $g$, we obtain, $\max\{p_{\mathrm{clf},l}(x), 1 - p_{\mathrm{clf},l}(x)\} = g(\abs{\ip{w_{\mathrm{clf},l}}{\varphi(x)}})$. Therefore,
\begin{align*}
    \Pr_{x'}[p_{\mathrm{clf},l}(x') \geq \max\{p_{\mathrm{clf},l}(x), 1- p_{\mathrm{clf},l}(x)] = \Pr_{x'}[\ip{w_{\mathrm{clf},l}}{\varphi(x')} \geq \abs{\ip{w_{\mathrm{clf},l}}{\varphi(x')}}]
\end{align*}

\end{proof}

\subsubsection{Proof of Theorem~\ref{thm:spherical}}
\label{sec:spherical_proof}
Let $u=\varphi(x), \forall x\in \gX$, then $u\sim \mathrm{Unif}(\mathbb{S}^{d-1})$. 

Let $u = \rho w_{\mathrm{clf},c} + \sqrt{1-\rho^2}u_{\perp} = \lambda w_{\mathrm{clf},l} + \sqrt{1-\lambda^2}v_{\perp}$, where $\ip{u_{\perp}, w_{\mathrm{clf},c}} = 0$ and $\ip{v_{\perp}}{w_{\mathrm{clf},l}} = 0$. For $x\sim \mathrm{Unif}(\mathbb{S}^{d-1})$, we need to find the distibutions on $\rho,\lambda, u_{\perp}$ and $v_{\perp}$ respectively. 

Note that $u_{\perp}$ is sampled uniformly from the $d-1$-dimensional unit sphere of $d$-dimensional unit vectors that are perpendicular to $w_{\mathrm{clf},c}$, while $v_{\perp}$ is sampled from a similar distribution of vectors perpendicular to $w_{\mathrm{clf},l}$. 

As for $\rho$ and $\lambda$, note that these belong to the same distribution, as the projection of a vector uniformly sampled from the unit sphere on a fixed vector, which is $w_{\mathrm{clf},c}$ and $w_{\mathrm{clf},l}$ for $\rho$ and $\lambda$ respectively.
Further, if $A_{d}(r)$ is the surface area of the $d$-dimensional unit sphere of radius $r$. Then, $\forall a,b\in [-1,1]$ with $a\leq b$, we have, 
\begin{align}
    \Pr[\rho \in [a, b]] = \Pr[\lambda  \in [a, b]] = \int_{a}^b\frac{A_{d-1}(\sqrt{1-a^2})}{A_d(1)} da = \frac{A_{d-1}(1)}{A_{d}(1)} \int_{a}^b(\sqrt{1 - a^2})^{d-3} da \label{eq:prob_sphere}
\end{align}
Note that the surface area of sphere which has a projection $a$ on a fixed vector is simply the surface area of the $d-1$ dimensional sphere of radius $\sqrt{1-a^2}$.

We obtain the following conditions.
\begin{equation}\label{eq:sphere_main}
\begin{aligned}
    & \rho \perp u_{\perp} \text{ and } \lambda  \perp v_{\perp},\quad\E[\lambda] = \E[\rho] = 0, \quad \E[\lambda^2] = \E[\rho^2] = \frac{1}{d},\quad \E[u_{\perp}] = \E[v_{\perp}] = 0\\ 
    &\quad\E[u_{\perp} u_{\perp}^\top] = \frac{1}{d-1}(I_d - w_{\mathrm{clf},c}w_{\mathrm{clf},c}^\top),\quad \E[v_{\perp} v_{\perp}^\top] = \frac{1}{d-1}(I_d - w_{\mathrm{clf},l}w_{\mathrm{clf},l}^\top)
\end{aligned}
\end{equation}
We decompose the difference term $(2p^\star(x)-1)(p_h(x) -p_{\mathrm{clf},l}(x))$
\begin{align}
    (2p^\star(x)-1)(p_h(x) -p_{\mathrm{clf},l}(x)) &= \ip{w^\star}{u}\ip{u}{w_{h} - w_{\mathrm{clf},l}} = (w^\star)^\top uu^\top (w_h - w_{\mathrm{clf},l})\label{eq:sphere_err}
\end{align}

Now, we can compute $F_{\mathrm{clf}} - F_{\mathrm{clf},c}(\alpha)$ and $F_{\mathrm{clf}} - F_{\mathrm{clf},p}(\alpha)$. As we need a positive lower bound on $F_{\mathrm{clf},c}(\alpha) - F_{\mathrm{clf},p}(\alpha)$ to show that $F_{\mathrm{clf},c}(\alpha) \geq F_{\mathrm{clf},p}(\alpha)$, we compute an upper bound on $F_{\mathrm{clf}} - F_{\mathrm{clf},c}(\alpha)$ and a lower bound on $F_{\mathrm{clf}} - F_{\mathrm{clf},p}(\alpha)$.

\paragraph{Upper bound on ConfSel.}
Note that the distribution of $p_{\mathrm{clf},l}(x)$ and $p_{\mathrm{clf},c}(x)$ is symmetric around $\frac{1}{2}$ for spherical features and linear link function. So, we use Lemma~\ref{lem:simpl} and ~\eqref{eq:prob_sphere} to bound $F_{\mathrm{clf}} - F_{\mathrm{clf},c}(\alpha)$.
\begin{align*}
&\Pr_{x'\sim \gD_x}[\abs{2p_{\mathrm{clf},c}(x')-1}\geq \abs{2p_{\mathrm{clf},c}(x)-1}] \\
&= \Pr_{x'}[\abs{\ip{w_{\mathrm{clf},c}}{\varphi(x')}} \geq \abs{\ip{w_{\mathrm{clf},c}}{\varphi(x)}}]=C_1\int_{\abs{\rho}}^1 (\sqrt{1-a^2})^{d-3} da 
\end{align*}
where $C_1 = \frac{2A_{d-1}(1)}{A_d(1)}$.

To compute $F_{\mathrm{clf}} - F_{\mathrm{clf},c}(\alpha)$, we can further simplify ~\eqref{eq:sphere_err}.
\begin{align*}
    uu^\top = \rho^2 w_{\mathrm{clf},c}w_{\mathrm{clf},c}^\top  + (1-\rho^2) u_{\perp}u_{\perp}^\top + \rho\sqrt{1-\rho^2}(w_{\mathrm{clf},c} u_{\perp}^\top + u_{\perp}w_{\mathrm{clf},c}^\top)
\end{align*}
Using ~\eqref{eq:sphere_main} and \eqref{eq:sphere_err}, we obtain
\begin{align*}
F_{\mathrm{clf}} - F_{\mathrm{clf},c}(\alpha) &= \E_x\left[\left(C_1\int_{\abs{\rho}}^1 (\sqrt{1-a^2})^{d-3} da\right)^{(1-\alpha)n}(w^\star)^\top uu^\top (w_h - w_{\mathrm{clf},l})\right]\\
&=(w^\star)^\top\E_x\left[\left(C_1\int_{\abs{\rho}}^1 (\sqrt{1-a^2})^{d-3} da\right)^{(1-\alpha)n} uu^\top\right](w_h - w_{\mathrm{clf},l})
\end{align*}
Therefore, we only need to compute $\E_x\left[\left(C_1\int_{\abs{\rho}}^1 (\sqrt{1-a^2})^{d-3} da\right)^{(1-\alpha)n} uu^\top\right]$.
\begin{align*}
    &\E_x\left[\left(C_1\int_{\abs{\rho}}^1 (\sqrt{1-a^2})^{d-3} da\right)^{(1-\alpha)n} uu^\top\right]\\ 
    &= \E_x\Bigg[\left(C_1\int_{\abs{\rho}}^1 (\sqrt{1-a^2})^{d-3} da\right)^{(1-\alpha)n}\\
    &\quad \cdot(\rho^2 w_{\mathrm{clf},c}w_{\mathrm{clf},c}^\top  + (1-\rho^2) u_{\perp}u_{\perp}^\top + \rho\sqrt{1-\rho^2}(w_{\mathrm{clf},c} u_{\perp}^\top + u_{\perp}w_{\mathrm{clf},c}^\top)) \Bigg]\\
    & = \underset{\numcircledtikz {1}}{\underbrace{\E_{\rho}\left[\left(C_1\int_{\abs{\rho}}^1 (\sqrt{1-a^2})^{d-3} da\right)^{(1-\alpha)n}\rho^2\right] w_{\mathrm{clf},c}w_{\mathrm{clf},c}^\top}} \\
    &\quad+ \underset{\numcircledtikz {2}}{\underbrace{\E_{\rho}\left[\left(C_1\int_{\abs{\rho}}^1 (\sqrt{1-a^2})^{d-3} da\right)^{(1-\alpha)n}(1-\rho^2)\right]\E_{u_{\perp}}[u_{\perp}u_{\perp}^\top]}} \\
    &\quad+ \underset{\numcircledtikz {3}}{\underbrace{\E_{\rho}\left[\left(C_1\int_{\abs{\rho}}^1 (\sqrt{1-a^2})^{d-3} da\right)^{(1-\alpha)n}\rho\sqrt{1-\rho^2}\right](w_{\mathrm{clf},c}\E_{u_{\perp}}[u_{\perp}^\top] + \E_{u_{\perp}}[u_{\perp}]w_{\mathrm{clf},c}^\top)}}
\end{align*}

In the last step, we use the independence of $\rho$ and $u_{\perp}$. From ~\eqref{eq:sphere_main}, the term $\numcircledtikz {3}$ is $0$. Further, $\E_{u_{\perp}}[u_{\perp} u_{\perp}^\top] = \frac{1}{d-1}(I_d - w_{\mathrm{clf},c} w_{\mathrm{clf},c}^\top)$.

We use $E_0 \triangleq   \E_{\rho}[ (h(\rho))^{(1-\alpha)n}]$ and $E_2 \triangleq =     \E_{\rho}[ (h(\rho))^{(1-\alpha)n}\rho^2]$ where $h(\rho) = \int_{\abs{\rho}}^1 (\sqrt{1-a^2})^{d-3} da$. Further, we set $W_{\mathrm{clf},c} = w_{\mathrm{clf},c} w_{\mathrm{clf},c}^\top$
Then, we can compute the required non-zero terms in above expression.

\begin{align}\label{eq:sphere_conf_final}
    \numcircledtikz {1} &= C_1^{(1-\alpha)n} E_2 W_{\mathrm{clf},c}, \quad \numcircledtikz {2} = \frac{C_1^{(1-\alpha)n}(E_0 - E_2)}{d-1}(I_d - W_{\mathrm{clf},c}) 
\end{align}

\paragraph{Lower bound for PairSel-Middle.}
The lower bound on $F_{\mathrm{clf}} - F_{\mathrm{clf},p}(\alpha)$ proceeds in a similar fashion to the upper bound on $F_{\mathrm{clf}} - F_{\mathrm{clf},c}(\alpha)$. Note that $p_{\mathrm{clf},l}(x)$ is also symmetric around $\frac{1}{2}$ under spherical features and linear link function. We use Lemma~\ref{lem:simpl} and ~\eqref{eq:prob_sphere} to bound the probability term.
\begin{align*}
    \sqrt{\Pr_{x'\sim \gD_x}[p_{\mathrm{clf},l}(x')\leq p_{\mathrm{clf},l}(x)]\cdot \Pr_{x'\sim \gD_x}[p_{\mathrm{clf},l}(x')\geq p_{\mathrm{clf},l}(x)]} &\geq \Pr_{x'\sim \gD_x}[\ip{w_{\mathrm{clf},l}}{\varphi(x')} \geq \abs{\ip{w_{\mathrm{clf},l}}{\varphi(x')}}]\\
    &=C_2 \int_{\abs{\lambda}}^1 (\sqrt{1-a^2})^{d-3}ds
\end{align*}
where $C_2 = \frac{A_{d-1}(1)}{A_d(1)} = \frac{C_1}{2}$.
Note that this expression is exactly half that of the corresponding term for confidence queries.
Further, we can use a similar decomposition of $u u^\top$ in terms of $\lambda$ and $v_{\perp}$ to obtain,
\begin{align*}
    F_{\mathrm{clf}} - F_{\mathrm{clf},p}(\alpha) =(w^\star)^\top\E_x\left[\left(C_2\int_{\abs{\lambda}}^1 (\sqrt{1-a^2})^{d-3} da\right)^{(1-\alpha)n} uu^\top\right](w_h - w_{\mathrm{clf},l})
\end{align*}
Now, following the decomposition of $u u^\top$ in terms of $\lambda$ and $v_{\perp}$, we obtain,
\begin{align*}
    &\E_x\left[\left(C_2\int_{\abs{\lambda}}^1 (\sqrt{1-a^2})^{d-3} da\right)^{(1-\alpha)n} uu^\top\right]\\
    & = \underset{\numcircledtikz {4}}{\underbrace{\E_{\lambda}\left[\left(C_2\int_{\abs{\lambda}}^1 (\sqrt{1-a^2})^{d-3} da\right)^{(1-\alpha)n}\lambda^2\right] w_{\mathrm{clf},l}w_{\mathrm{clf},l}^\top}} \\
    &\quad + \underset{\numcircledtikz {5}}{\underbrace{\E_{\lambda}\left[\left(C_2\int_{\abs{\lambda}}^1 (\sqrt{1-a^2})^{d-3} da\right)^{(1-\alpha)n}(1-\lambda^2)\right]\E_{\mathrm{clf}_{\perp}}[v_{\perp}v_{\perp}^\top]}} \\
    &\quad+ \underset{\numcircledtikz {6}}{\underbrace{\E_{\lambda}\left[\left(C_2\int_{\abs{\lambda}}^1 (\sqrt{1-a^2})^{d-3} da\right)^{(1-\alpha)n}\lambda \sqrt{1-\lambda ^2}\right](w_{\mathrm{clf},l}\E_{\mathrm{clf}_{\perp}}[v_{\perp}^\top] + \E_{\mathrm{clf}_{\perp}}[v_{\perp}]w_{\mathrm{clf},l}^\top)}}
\end{align*}
The term $\numcircledtikz {6}$ is $0$ from ~\eqref{eq:sphere_main}. Further, the terms $\numcircledtikz {4}$ and $\numcircledtikz {5}$ can be computed in terms of $E_0$ and $E_2$. We set $W_{\mathrm{clf},l} = w_{\mathrm{clf},l}w_{\mathrm{clf},l}^\top$.
\begin{align}\label{eq:sphere_pair_final}
    \numcircledtikz {4} = C_2^{(1-\alpha)n} E_2 W_{\mathrm{clf},l}, \quad \numcircledtikz {5} = \frac{C_2^{(1-\alpha)n}(E_0 - E_2)}{d-1}(I_d - W_{\mathrm{clf},l}) 
\end{align}

\paragraph{Difference of errors}
We can now compute the term $F_{\mathrm{clf},c}(\alpha) - F_{\mathrm{clf},p}(\alpha)$.
\begin{align*}
F_{\mathrm{clf},c}(\alpha) - F_{\mathrm{clf},p}(\alpha) &= (F_{\mathrm{clf}} - F_{\mathrm{clf},p}(\alpha)) - (F_{\mathrm{clf}} - F_{\mathrm{clf},c}(\alpha)) \\
&= (w^\star)^\top( \numcircledtikz {4} + \numcircledtikz {5} - \numcircledtikz {1} - \numcircledtikz {2})(w_h - w_{\mathrm{clf},l})\\
\end{align*}
We first compute the difference $\numcircledtikz {4} + \numcircledtikz {5} - \numcircledtikz {1} - \numcircledtikz {2}$.
\begin{align*}
    \numcircledtikz {4} + \numcircledtikz {5} - \numcircledtikz {1} - \numcircledtikz {2} &=    C_2^{(1-\alpha)n} E_2 W_{\mathrm{clf},l} + \frac{C_2^{(1-\alpha)n}(E_0 - E_2)}{d-1}(I_d - W_{\mathrm{clf},l}) \\
    \quad &- C_1^{(1-\alpha)n} E_2 W_{\mathrm{clf},c} - \frac{C_1^{(1-\alpha)n}(E_0 - E_2)}{d-1}(I_d - W_{\mathrm{clf},c})\\
    &= \frac{C_2^{(1-\alpha)n}}{d-1}\bigg((E_0 - dE_2) (2^{(1-\alpha)n - W_{\mathrm{clf},l}} W_{\mathrm{clf},c}) \\
    &\quad - (2^{(1-\alpha)n} - 1) (E_0 - E_2) I_d\bigg)
\end{align*}
Plugging this back, we obtain.
\begin{align*}
    &F_{\mathrm{clf},c}(\alpha) - F_{\mathrm{clf},p}(\alpha) \\
    &\geq  \frac{C_2^{(1-\alpha)n}}{d-1}\bigg(( E_0 - d E_2) (2^{(1-\alpha)n} \ip{w^\star}{w_{\mathrm{clf},c}}(\ip{w_{\mathrm{clf},c}}{w_h} - \ip{w_{\mathrm{clf},c}}{w_{\mathrm{clf},l}}) \\
    &\quad-\ip{w^\star}{w_{\mathrm{clf},l}}(\ip{w_h}{w_{\mathrm{clf},l}} - 1)) - (2^{(1-\alpha)n} - 1)(E_0 - E_2)(\ip{w^\star}{w_h} - \ip{w^\star}{w_{\mathrm{clf},l}})\bigg)
\end{align*}
We require $w_{\mathrm{clf},c}$ such that the RHS is $\geq 0$. This implies,
\begin{align}\label{eq:sphere_penul}
    &\ip{w^\star}{w_{\mathrm{clf},c}}(\ip{w_{\mathrm{clf},c}}{w_h} - \ip{w_{\mathrm{clf},c}}{w_{\mathrm{clf},l}})\nonumber\\
    &\geq 2^{-(1-\alpha)n}(\ip{w^\star}{w_{\mathrm{clf},l}}(\ip{w_h}{w_{\mathrm{clf},l}} - 1))  \nonumber\\
    &\quad + \frac{E_0 - E_2}{E_0 - d E_2}(1 - 2^{-(1-\alpha)n})(\ip{w^\star}{w_h} - \ip{w^\star}{w_{\mathrm{clf},l}})
\end{align}
We now compute the terms $E_0$ and an upper bound on $E_2$ which satisfies our requirements.

\begin{lemma}\label{lem:sphere_conf}
    For $\rho$ distributed according to ~\eqref{eq:prob_sphere}, 
    \begin{align*}
    E_0 &= \Theta\left(d^{(1-\alpha)n/2}((1-\alpha) n)^{-1}\right)\\
    E_2 &\leq   \mathcal{O}\left(d^{-3/2}((1-\alpha)n)^{-2}\right)\\
    \end{align*}
\end{lemma}
\begin{proof}
First, we compute $E_0$ using ~\eqref{eq:prob_sphere}.
\begin{align*}
    E_0 &= C_2 \int_{-1}^1 (h(\rho))^{(1-\alpha)n} (\sqrt{1-\rho^2})^{d-3} d\rho  
\end{align*}
As the term inside the integral is even in $\rho$, as $h(\rho)$ is even in $\rho$, we restrict the integral from $\rho=0$ to $1$. Further, $h'(\rho) = -(\sqrt{1-\rho^2})^{d-3} d\rho$. Therefore, substituting $t = h(\rho)$, we obtain,
\begin{align*}
    E_0 = -C_1 \int_{h(0)}^{h(1)} t^{(1-\alpha)n} dt = \frac{C_1}{(1-\alpha)n+1} (h(0))^{(1-\alpha)n+1}
\end{align*}
Note that $h(1) = 0$. Further, $h(0) = \int_0^1 (\sqrt{1-a^2})^{d-3} da = \frac{1}{C_1}\Pr_{x}[\abs{\rho} \in [0,1]] = \frac{1}{C_1}$ from ~\eqref{eq:prob_sphere}.
Therefore, 
\begin{align*}
    E_0 = \frac{C_1}{(1-\alpha)n +1} \frac{1}{C_1^{(1-\alpha)n+1)}} = \frac{1}{C_1^{(1-\alpha)n}(1-\alpha)n+1)} = \Theta\left(\frac{d^{(1-\alpha)n/2}}{(1-\alpha) n}\right)
\end{align*}
We use the fact that $C_1 = \Theta(d^{-1/2})$.

To compute $E_2$, we need to find an upper bound on $h(\rho)$.
\begin{align*}
    E_2 &= C_2 \int_{-1}^1 (h(\rho))^{(1-\alpha)n}(\sqrt{1-\rho^2})^{d-3} \rho^2 d\rho
\end{align*}
We first use the fact that both $h(\rho)$ and $\rho^2$ are even to restrict the integral from $\rho=0$ to $\rho=1$.
By using integration by parts, assuming $\rho^2$ is the first term and the rest forms the second term, we obtain,
\begin{align*}
    E_2 &= C_1 \left[- \frac{(h(\rho))^{(1-\alpha)n+1} \rho^2}{(1-\alpha)n+1}\right]_0^{1} + \frac{2C_1}{(1-\alpha)n+1} \int_0^1 (h(\rho))^{(1-\alpha)n+1} \rho d\rho 
\end{align*}
Here, the first term is $0$ as $h(1) = 0$ and at $\rho=0$ $\rho^2=0$. Note that the second term cannot be integrated directly, so we obtain an upper bound on it.
To do this, we need an upper bound on $h(\rho)$ for $\rho\in [0,1]$.
\begin{align*}
    h(\rho) = \int_\rho^1 (\sqrt{1-a^2})^{d-3} da \leq  (\sqrt{1-\rho^2})^{d-3} \int_\rho^1 da = (\sqrt{1-\rho^2})^{d-3} (1-\rho) \leq (\sqrt{1-\rho^2})^{d-1}
\end{align*}
For $\rho\in [0,1]$, we use $\rho^2 \leq \rho$.
With this upper bound, we can bound the value of $E_2$.
\begin{align*}
    E_2 &= \frac{2C_1}{(1-\alpha)n+1} \int_0^1 (h(\rho))^{(1-\alpha)n+1} \rho d\rho \leq \frac{2C_1}{(1-\alpha)n+1} \int_0^1 (\sqrt{1-\rho^2})^{(d-1)((1-\alpha)n+1)} \rho d\rho 
\end{align*}
We now substitute $t = 1 - \rho^2$ with $dt= -2\rho d\rho$.
\begin{align*}
    E_2 &\leq \frac{C_1}{(1-\alpha)n+1} \int_0^1 t^{(d-1)((1-\alpha)n+1)/2} dt = \frac{C_1}{((d-1)((1-\alpha)n+1)+2)(1-\alpha)n+1)} 
\end{align*}
Note that $C_1 = \Theta(d^{-1/2})$. Therefore,
\begin{align*}
    E_2\leq \mathcal{O}\left(\frac{1}{d^{3/2}((1-\alpha)n)^2}\right)
\end{align*}
\end{proof}

Using these bounds on $E_0$ and $E_2$, we can make obtain an interpretable upper bound for the second term in  ~\eqref{eq:sphere_penul}.
\begin{align*}
    \frac{E_0 - E_2}{E_0 - d E_2} = 1 + \frac{(d-1)E_2}{E_0 - d E_2} = 1 +(1-1/d)(E_0/d E_2  - 1)^{-1}
\end{align*}
If we plug in the value of $E_0$ and the upper bound on $E_2$, we obtain,
\begin{align*}
    \frac{E_0 - E_2}{E_0 - d E_2} \leq 1 + \Theta\left(d^{-((1-\alpha)n + 1)/2}((1-\alpha)n)^{-1}\right)
\end{align*}
Plugging this into ~\eqref{eq:sphere_penul} proves Theorem~\ref{thm:spherical}

\subsubsection{Proof of Theorem~\ref{thm:gauss}}
\label{sec:gauss_proof}
The proof for Gaussian features and $L$-smooth $g$ follows that of Theorem~\ref{thm:spherical} in Appendix~\ref{sec:spherical_proof}.
We set $u = \varphi(x), \forall x\in \gX$. Then, $u\sim \gN(0, I_d)$. Let $u = \rho w_{\mathrm{clf},c} + u_{\perp} = \lambda w_{\mathrm{clf},l} + v_{\perp}$ where $\ip{w_{\mathrm{clf},c}}{u_{\perp}} = \ip{w_{\mathrm{clf},l}}{v_{\perp}}=0$. As $u$ is gaussian, $\rho, \lambda, u_{\perp}$ and $v_{\perp}$ are also gaussian. 

These random variables have the following distribution.
\begin{equation}\label{eq:gauss_main}
    \begin{aligned}
        &\rho  \perp u_{\perp} \quad \text{and} \quad \lambda \perp v_{\perp},\\
        &\rho = \ip{w_{\mathrm{clf},c}}{u}\sim \gN(0, 1), \quad \lambda = \ip{w_{\mathrm{clf},l}}{u} \sim \gN(0,1),\\
        & \E[u_{\perp} u_{\perp}^\top] + \E[\rho^2]w_{\mathrm{clf},c}w_{\mathrm{clf},c}^\top = \E[v_{\perp} v_{\perp}^\top] + \E[\lambda^2]w_{\mathrm{clf},l}w_{\mathrm{clf},l}^\top = I_d\\
        & u_{\perp} = u - \rho w_{\mathrm{clf},c} \sim \gN(0, I_d - w_{\mathrm{clf},c}w_{\mathrm{clf},c}^\top), \\
        & v_{\perp} = u - \lambda w_{\mathrm{clf},l} \sim \gN(0, I_d - w_{\mathrm{clf},l}w_{\mathrm{clf},l}^\top),
    \end{aligned}
\end{equation}

For a $L$-smooth differentiable $g$, we need to use Taylor's expansion for the difference term $(2p^\star(x) - 1)(p_h(x) - p_{\mathrm{clf},l}(x))$. We use $g'(0)=1$ and $g(0) = \frac{1}{2}$ to write down the Taylor's expansion for $p^\star(x), p_h(x)$ and $p_{\mathrm{clf},l}(x)$ around $0$.

\begin{equation}\label{eq:taylor}
    \begin{aligned}
        p^\star(x) &=g(0) + g'(0)\ip{w^\star}{u} + r^\star(\ip{w^\star}{u})) = \frac{1}{2} + \ip{w^\star}{u} + r^\star(\ip{w^\star}{u}))\\
        p_h(x) &=g(0) + g'(0)\ip{w_h}{u} + r_h(\ip{w_h}{u}))= \frac{1}{2} + \ip{w_h}{u} + r_h(\ip{w_h}{u}))\\
        p_{\mathrm{clf},l}(x) &=g(0) + g'(0)\ip{w_{\mathrm{clf},l}}{u} + r_{\mathrm{clf},l}(\ip{w_{\mathrm{clf},l}}{u}) = \frac{1}{2} + \ip{w_{\mathrm{clf},l}}{u} + r_{\mathrm{clf},l}(\ip{w_{\mathrm{clf},l}}{u})
    \end{aligned}
\end{equation}
By $L$-smoothness of $g$, we have,
\begin{align}\label{eq:l_smooth}
    \forall a\in \R, \quad \abs{r^\star(a)},\abs{r_h(a)}, \abs{r_{\mathrm{clf},l}(a)} \leq \frac{L}{2}\abs{a}^2 
\end{align}
Therefore,
\begin{equation}\label{eq:gauss_err_main}
\begin{aligned}
    &(2p^\star(x) - 1)(p_{h}(x) - p_{\mathrm{clf},l}(x)) \\
    &= (\ip{w^\star}{u} + r^\star(\ip{w^\star}{u}))(\ip{w_{h} - w_{\mathrm{clf},l}}{u} + r_h(\ip{w_{h}}{u}) - r_{\mathrm{clf},l}(\ip{w_{\mathrm{clf},l}}{u}))\\
    &= (w^\star) uu^\top (w_h - w_{\mathrm{clf},l}) + r(u)\\
    \text{where}\quad r(u) &=  \ip{w^\star}{u}(r_h(\ip{w_{h}}{u}) - r_{\mathrm{clf},l}(\ip{w_{\mathrm{clf},l}}{u})) + \ip{w_{h} - w_{\mathrm{clf},l}}{u} r^\star(\ip{w^\star}{u})
\end{aligned}
\end{equation}

We compute a bound on $r(u)$.
\begin{align}
    \abs{r(u)} \leq&  \abs{\ip{w^\star}{u}}(\abs{r_h(\ip{w_{h}}{u})} + \abs{r_{\mathrm{clf},l}(\ip{w_{\mathrm{clf},l}}{u})}) + (\abs{\ip{w_{h}}{u}} +\abs{\ip{w_{\mathrm{clf},l}}{u}})\abs{ r^\star(\ip{w^\star}{u})}\nonumber\\
    \leq&  \frac{L}{2}\abs{\ip{w^\star}{u}}(\abs{\ip{w_{h}}{u}}^2 + \abs{\ip{w_{\mathrm{clf},l}}{u}}^2) + \frac{L}{2}(\abs{\ip{w_{h}}{u}} +\abs{\ip{w_{\mathrm{clf},l}}{u}})\abs{\ip{w^\star}{u}}^2 \nonumber\\
\leq&  \frac{L}{2}\abs{\ip{w^\star}{u}}\left((\abs{\ip{w_{h}}{u}}^2 + \abs{\ip{w_{\mathrm{clf},l}}{u}}^2) + (\abs{\ip{w_{h}}{u}} +\abs{\ip{w_{\mathrm{clf},l}}{u}})\abs{\ip{w^\star}{u}}\right) \nonumber\\
\leq&\frac{L}{2}\abs{\ip{w^\star}{u}}\left((\abs{\ip{w_{h}}{u}} + \abs{\ip{w_{\mathrm{clf},l}}{u}})^2 + (\abs{\ip{w_{h}}{u}} +\abs{\ip{w_{\mathrm{clf},l}}{u}})\abs{\ip{w^\star}{u}}\right) \nonumber\\
=& \frac{L}{2}\abs{\ip{w^\star}{u}}(\abs{\ip{w_{h}}{u}} + \abs{\ip{w_{\mathrm{clf},l}}{u}})(\abs{\ip{w_{h}}{u}} +\abs{\ip{w_{\mathrm{clf},l}}{u}}+\abs{\ip{w^\star}{u}})\label{eq:g_final}
\end{align}
We use triangle inequality in the first step and $L$-smoothness of $g$ in the second step. In the fourth step, we use $a^2 + b^2  \leq (a+b)^2$ for $a,b\geq 0$.

We compute lower bound on $F_{\mathrm{clf}} - F_{\mathrm{clf},p}(\alpha)$ and upper bound on $F_{\mathrm{clf}} - F_{\mathrm{clf},c}(\alpha)$ to compute a lower bound on the difference $F_{\mathrm{clf},c}(\alpha) - F_{\mathrm{clf},p}(\alpha)$.

\textbf{Upper bound on ConfSel.} The distributions of $p_{\mathrm{clf},l}(x)$  is symmetric around $\frac{1}{2}$ from Assumption~\ref{assump:lin_model} for gaussian features. From Lemma~\ref{lem:conf_err}, we obtain,
\begin{align*}
&\Pr_{x'\sim \gD_x}[\abs{2p_{\mathrm{clf},c}(x')-1}\geq \abs{2p_{\mathrm{clf},c}(x)-1}] \\
&= \Pr_{x'}[\abs{\ip{w_{\mathrm{clf},c}}{\varphi(x')}} \geq \abs{\ip{w_{\mathrm{clf},c}}{\varphi(x)}}]=\Pr_{\rho'\sim \gN(0,1)}[\abs{\rho'}\geq \abs{\rho}] = 2(1 - \Phi(\abs{\rho}))
\end{align*}
We substitute $u' = \varphi(x') = \rho' w_{\mathrm{clf},c} + u_{\perp}'$, where $\ip{w_{\mathrm{clf},c}}{u_{\perp}'}=0$ which matches the decomposition of $u$. Additionally,  $\Phi:\R\to [0,1]$  the cdf of unit normal distribution.

We set $\widetilde{h}(\rho) \triangleq 1 - \Phi(\abs{\rho}), W_{\mathrm{clf},c} = w_{\mathrm{clf},c}w_{\mathrm{clf},c}^\top$.
From Eq~\eqref{eq:gauss_err_main}, we bound $F_{\mathrm{clf}} - F_{\mathrm{clf},c}(\alpha)$.
\begin{align*}
    &F_{\mathrm{clf}} - F_{\mathrm{clf},c}(\alpha) = \E_u[ (2\widetilde{h}(\rho))^{(1-\alpha)n} ((w^\star)^\top(uu^\top)(w_{h} - w_{\mathrm{clf},l}) + r(u))]\\
    &= (w^\star)^\top\left(\underset{\numcircledtikz {7}}{\underbrace{\E_{\rho}[(2\widetilde{h}(\rho))^{(1-\alpha)n}\rho^2]W_{\mathrm{clf},c}}} + \underset{\numcircledtikz {8}}{\underbrace{\E_{\rho}[(2\widetilde{h}(\rho))^{(1-\alpha)n}]\E_{u_{\perp}}[u_{\perp}u_{\perp}^\top]}} \right)(w_{h} - w_{\mathrm{clf},l})\\
    &\quad  + \underset{\numcircledtikz {9}}{\underbrace{\E_{\rho}[(2\widetilde{h}(\rho))^{(1-\alpha)n}\rho](w^\star)^\top(w_{\mathrm{clf},c}\E_{u_{\perp}}[u_{\perp}^\top] + \E_{u_{\perp}}[u_{\perp}] w_{\mathrm{clf},c}^\top )(w_h - w_{\mathrm{clf},l})}} \\
    &\quad + \underset{\numcircledtikz {10}}{\underbrace{\E_{u}[(2\widetilde{h}(\rho))^{(1-\alpha)n}r(u)]}}
\end{align*}
Note that $\numcircledtikz{9}=0$ by symmetry of $u_{\perp}$. Let $E_{a,b} \triangleq \E_{\rho\sim \gN(0,1)}[(\widetilde{h}(\rho))^{b}\abs{\rho}^a]$ for all $a\in \mathbb{N}\cup\{0\}$. 

\begin{align}\label{eq:gauss_conf_final}
    \numcircledtikz{7} = 2^{(1-\alpha)n} E_{2,(1-\alpha)n} W_{\mathrm{clf},c},\quad \numcircledtikz{8} = 2^{(1-\alpha)n} E_{0,(1-\alpha)n}(I_d - W_{\mathrm{clf},c})
\end{align}

We compute $\numcircledtikz{10}$ separately using ~\eqref{eq:g_final}.

\begin{align*}
    &\numcircledtikz{10} =  \E_{u}[(2\widetilde{h}(\rho))^{(1-\alpha)n}r(u)] \leq \E_{u}[(2\widetilde{h}(\rho))^{(1-\alpha)n}\abs{r(u)}]\\
    &\leq\frac{L}{2}\E_{u}[(2\widetilde{h}(\rho))^{(1-\alpha)n} \sqrt{\abs{\ip{w^\star}{u}}^2(\abs{\ip{w_{h}}{u}} + \abs{\ip{w_{\mathrm{clf},l}}{u}})^2(\abs{\ip{w_{h}}{u}} +\abs{\ip{w_{\mathrm{clf},l}}{u}}+\abs{\ip{w^\star}{u}})^2}]\\
    &\leq\frac{\sqrt{6}L}{2}\E_{u}[(2\widetilde{h}(\rho))^{(1-\alpha)n} \sqrt{\ip{w^\star}{u}^2(\ip{w_{h}}{u}^2 + \ip{w_{\mathrm{clf},l}}{u}^2)(\ip{w_{h}}{u}^2 +\ip{w_{\mathrm{clf},l}}{u}^2+\ip{w^\star}{u}^2)}]\\
    &\leq\frac{\sqrt{6}L}{2}\E_{u}[(2\widetilde{h}(\rho))^{(1-\alpha)n} \sqrt{\ip{w^\star}{u}^2 \ip{w_{h}}{u}^4} + \sqrt{\ip{w^\star}{u}^4\ip{w_{h}}{u}^2} + \sqrt{\ip{w^\star}{u}^2 \ip{w_{\mathrm{clf},l}}{u}^4} \\
    &+ \sqrt{\ip{w^\star}{u}^4 \ip{w_{\mathrm{clf},l}}{u}^2} +\sqrt{2\ip{w_{h}}{u}^2 \ip{w^\star}{u}^2 \ip{w_{\mathrm{clf},l}}{u}^2}]\\
\end{align*}
For the second inequality, we use $L$-smoothness. We use $(\sum_{i=1}^l a_i)^2 \leq l \sum_{i=1}^l a_i^2$ to obtain the third inequality. In the fifth inequality, we use $\sum_{i=1}^l \sqrt{a_i} \leq \sum_{i=1}^l \sqrt{a_i}, \forall a_i\geq 0, i\in [l]$.

Note that all the terms in this expression are of the form $\ip{w_1}{u}^2\ip{w_2}{u}^2$ or $2\ip{w_1}{u}^2\ip{w_2}{u}^2\ip{w_3}{u}^2$, where $w_1, w_2$ and $w_3$ take values in $\{w^\star, w_{\mathrm{clf},l}, w_h\}$. We use these simplified expressions to compute $\numcircledtikz{10}$.

\begin{lemma}[Intermediate remainder Terms]\label{lem:rem_term_simp}
Consider $4$ unit norm vectors, $w_0, w_1, w_2, w_3\in \R^d$. If $u\sim \gN(0,I_d)$ and $\rho = \ip{u}{w_0},\forall u$, then for some constants $C_3,C_4>0$, we have,
    \begin{align*}
        \E_{u}[(\widetilde{h}(\rho))^{(1-\alpha)n}\sqrt{\ip{w_1}{u}^4 \ip{w_2}{u}^2} &\leq   C_3\sum_{k=0}^3 E_{k,(1-\alpha)n}d^{(3-k)/2} \\
        \E_{u}[(\widetilde{h}(\rho))^{(1-\alpha)n}\sqrt{\ip{w_1}{u}^2 \ip{w_2}{u}^2 \ip{w_3}{u}^2}] &\leq  C_4\sum_{k=0}^3 E_{k,(1-\alpha)n}d^{(3-k)/2}\\
        \text{where } E_{a,b}& = \E_{\rho\sim \gN(0,1)}[(\widetilde{h}(\rho))^b \abs{\rho}^a], \quad \forall a,b\in \mathbb{N}\cup\{0\}
    \end{align*}
\end{lemma}
\begin{proof}
    Let $u = \rho w_0 + u_{\perp}$ where $u_{\perp} \perp w_0$. Then, $\rho\sim \gN(0,1)$ and $u_{\perp} \sim \gN(0, I_d - w_0 w_0^\top)$ from ~\eqref{eq:gauss_main}. We compute the first expression. 

    \begin{align*}
        &\ip{w_1}{u}^4 \ip{w_2}{u}^2 = (\rho\ip{w_1}{w_0} + \ip{w_1}{u_{\perp}})^4(\rho\ip{w_2}{w_0} + \ip{w_2}{u_{\perp}})^2\\
        &=(\rho^4\ip{w_1}{w_0}^4 + 3\rho^3 \ip{w_1}{w_0}^3 \ip{w_1}{u_{\perp}} + 6 \rho^2 \ip{w_1}{w_0}^2 \ip{w_1}{u_{\perp}}^2 + 3 \rho \ip{w_1}{w_0}\ip{w_1}{u_{\perp}}^3 + \ip{w_1}{u_{\perp}}^4)\cdot \\
        &\times (\rho^2\ip{w_2}{w_0}^2 + 2\rho \ip{w_2}{w_0}\ip{w_2}{u_{\perp}} + \ip{w_2}{u_{\perp}}^2)
    \end{align*}
    We bound all inner product terms of $\ip{w_i}{w_j}\leq 1$, for $i,j\in [4]$. Further, let $\ip{w_1}{u_{\perp}} = \zeta$. Then, for some constant $C_3'>0$, using $\sqrt{\sum_{i=1}^l a_i}\leq \sum_{i=1}^l \sqrt{\abs{a_i}}$ we have,
    \begin{align*}
        \E_{u}[(\widetilde{h}(\rho))^{(1-\alpha)n}\sqrt{\ip{w_1}{u}^4 \ip{w_2}{u}^2}]
        &\leq C_3'\sum_{k=0}^3\E_{\rho}[\widetilde{h}(\rho)\abs{\rho}^{k}]\E_{\zeta}[\abs{\zeta}^{3-k}]
    \end{align*}
    Note that $\zeta\sim \gN(0, d-1)$, therefore $\abs{\zeta}$ is a folded normal distribution. As we need a bound on its first three moments for the above summation, we use its moment-generating function from ~\citep{Tsagris_2014}.
    \begin{align*}
        \E_{\zeta}[\abs{\zeta}^k] = \Theta_{k}(\sqrt{d}^{k})
    \end{align*}
    Plugging this in, we get the first expression.

    For the second expression, 
    \begin{align*}
            &\ip{w_1}{u}^2 \ip{w_2}{u}^2 \ip{w_3}{u}^2 = (\rho\ip{w_1}{w_0} + \ip{w_1}{u_{\perp}})^2(\rho\ip{w_2}{w_0} + \ip{w_2}{u_{\perp}})^2(\rho\ip{w_3}{w_0} + \ip{w_3}{u_{\perp}})^2        
    \end{align*}
    Again, bounding the inner product terms by $1$, we obtain, for some constant $C_4'>0$
    \begin{align*}
        \E_{u}[(\widetilde{h}(\rho))^{(1-\alpha)n}\sqrt{\ip{w_1}{u}^2 \ip{w_2}{u}^2 \ip{w_3}{u}^2} \leq C_4'\sum_{k=0}^3\E_{\rho}[\widetilde{h}(\rho)\abs{\rho}^{k}]\E_{\zeta}[\abs{\zeta}^{3-k}] 
    \end{align*}
    We again use the bound on the moment generating functions of $\abs{\zeta}$.
\end{proof}
Using the above lemma, for some constant $C_5>0$, we can bound $\numcircledtikz{10}$
as,
\begin{align}\label{eq:gauss_term_10}
    \numcircledtikz{10} \leq 2^{(1-\alpha)n}C_5 L \sum_{k=0}^3 E_{k,(1-\alpha)n}d^{(3-k)/2}
\end{align}

\paragraph{Lower bound for PairSel-Middle.} By symmetry of $p_{\mathrm{clf},l}(x)$ around $\frac{1}{2}$ from Assumption~\ref{assump:lin_model} for gaussian features, and we apply Lemma~\ref{lem:simpl}.
\begin{align*}
        \sqrt{\Pr_{x'\sim \gD_x}[p_{\mathrm{clf},l}(x')\leq p_{\mathrm{clf},l}(x)]\cdot \Pr_{x'\sim \gD_x}[p_{\mathrm{clf},l}(x')\geq p_{\mathrm{clf},l}(x)]} &\geq \Pr_{x'}[\ip{w_{\mathrm{clf},l}}{\varphi(x')} \geq \abs{\ip{w_{\mathrm{clf},l}}{\varphi(x')}}]\\
    &= \Pr_{\lambda \sim \gN(0,1)}[\lambda' \geq \abs{\lambda}] = 1 - \Phi(\abs{\lambda})
\end{align*}
We use a similar parametrization of $x'$ into $\lambda' w_{\mathrm{clf},l} + v_{\perp}'$ where $\ip{v_{\perp}'}{w_{\mathrm{clf},l}}=0$. We now find a lower bound for $F_{\mathrm{clf}} - F_{\mathrm{clf},p}(\alpha)$ using the $\lambda, v_{\perp}$ parametrization and ~\eqref{eq:gauss_err_main}. Further, $W_{\mathrm{clf},l} = w_{\mathrm{clf},l}w_{\mathrm{clf},l}^\top$

\begin{align*}
    &F_{\mathrm{clf}} - F_{\mathrm{clf},p}(\alpha) = \E_u[ (\widetilde{h}(\lambda))^{(1-\alpha)n} ((w^\star)^\top(uu^\top)(w_{h} - w_{\mathrm{clf},l}) + r(u))]\\
    &= (w^\star)^\top\left(\underset{\numcircledtikz {11}}{\underbrace{\E_{\lambda}[(\widetilde{h}(\lambda))^{(1-\alpha)n}\lambda^2]W_{\mathrm{clf},l}}} + \underset{\numcircledtikz {12}}{\underbrace{\E_{\lambda}[(\widetilde{h}(\lambda))^{(1-\alpha)n}]\E_{\mathrm{clf}_{\perp}}[v_{\perp}v_{\perp}^\top]}} \right)(w_{h} - w_{\mathrm{clf},l})\\
    &\quad  + \underset{\numcircledtikz {13}}{\underbrace{\E_{\lambda}[(\widetilde{h}(\lambda))^{(1-\alpha)n}\lambda](w^\star)^\top(w_{\mathrm{clf},c}\E_{\mathrm{clf}_{\perp}}[v_{\perp}^\top] + \E_{\mathrm{clf}_{\perp}}[v_{\perp}] w_{\mathrm{clf},c}^\top )(w_h - w_{\mathrm{clf},l})}} \\
    &\quad + \underset{\numcircledtikz {14}}{\underbrace{\E_{u}[(\widetilde{h}(\lambda))^{(1-\alpha)n}r(u)]}}
\end{align*}
Note that $\numcircledtikz{13}=0$ by symmetry of $v_{\perp}$. From ~\eqref{eq:gauss_main}, $\lambda$ and $\rho$ have the same distribution. Then, 

\begin{align}\label{eq:gauss_pair_final}
    \numcircledtikz{11} =  E_{2,(1-\alpha)n} W_{\mathrm{clf},l},\quad \numcircledtikz{12} = E_{0,(1-\alpha)n}(I_d - W_{\mathrm{clf},l})
\end{align}
To compute $\numcircledtikz{14}$, we use the same technique as $\numcircledtikz{10}$.

\begin{align*}
        &\numcircledtikz{14} =  \E_{u}[(\widetilde{h}(\lambda))^{(1-\alpha)n}r(u)] \geq -\E_{u}[(\widetilde{h}(\lambda))^{(1-\alpha)n}\abs{r(u)}]\\
\end{align*}
Since the bounds in Lemma~\ref{lem:rem_term_simp} are independent of $w_0, w_1,w_2,w_3$, we can use them here as well. This gives us the bound 
\begin{align}
   \numcircledtikz{14} \geq - C_5 L\sum_{k=0}^3 E_{k,(1-\alpha)n}d^{(3-k)/2}
 \label{eq:gauss_term_14}
\end{align}

\paragraph{Difference of errors}
We can now compute the term $F_{\mathrm{clf},c}(\alpha) - F_{\mathrm{clf},p}(\alpha)$.
\begin{align*}
& F_{\mathrm{clf},c}(\alpha) - F_{\mathrm{clf},p}(\alpha) = (F_{\mathrm{clf}} - F_{\mathrm{clf},p}(\alpha)) - (F_{\mathrm{clf}} - F_{\mathrm{clf},c}(\alpha)) \\
&= (w^\star)^\top( \numcircledtikz {11} + \numcircledtikz {12} - \numcircledtikz{7} - \numcircledtikz {8})(w_h - w_{\mathrm{clf},l}) + \numcircledtikz{14} - \numcircledtikz{10}\\
\end{align*}
We first compute the difference $\numcircledtikz {11} + \numcircledtikz {12} - \numcircledtikz {7} - \numcircledtikz {8}$.
\begin{align*}
    \numcircledtikz {11} + \numcircledtikz {12} - \numcircledtikz {7} - \numcircledtikz {8} &=     E_{2,(1-\alpha)n} W_{\mathrm{clf},l} + E_{0,(1-\alpha)n}(I_d - W_{\mathrm{clf},l})\\
    &\quad - 2^{(1-\alpha)n}(E_{2,(1-\alpha)n} W_{\mathrm{clf},c} + E_{0,(1-\alpha)n}(I_d - W_{\mathrm{clf},c})) \\
    &= (E_{0,(1-\alpha)n} - E_{2,(1-\alpha)n})(2^{(1-\alpha)n} W_{\mathrm{clf},c} - W_{\mathrm{clf},l}) \\
    &\quad+E_{0,(1-\alpha)n} I_d(1 - 2^{(1-\alpha)n})
\end{align*}

We now compute the difference $\numcircledtikz{14} - \numcircledtikz{10}$ using ~\eqref{eq:gauss_term_10} and ~\eqref{eq:gauss_term_14}.

\begin{align*}
    \numcircledtikz{14} - \numcircledtikz{10} \geq -(1 + 2^{(1-\alpha)n}) C_5 L\sum_{k=0}^3 E_{k,(1-\alpha)n}d^{(3-k)/2}
\end{align*}

Plugging this back, we obtain.
\begin{align*}
    F_{\mathrm{clf},c}(\alpha) - F_{\mathrm{clf},p}(\alpha) \geq &  (E_{0,(1-\alpha)n} - E_{2,(1-\alpha)n})(2^{(1-\alpha)n} \ip{w^\star}{w_{\mathrm{clf},c}}\\
    &\quad \cdot(\ip{w_{\mathrm{clf},c}}{w_h} - \ip{w_{\mathrm{clf},c}}{w_{\mathrm{clf},l}}) -\ip{w^\star}{w_{\mathrm{clf},l}}(\ip{w_h}{w_{\mathrm{clf},l}} - 1))\\
    &\quad - (2^{(1-\alpha)n} - 1)E_{0,(1-\alpha)n}(\ip{w^\star}{w_h} - \ip{w^\star}{w_{\mathrm{clf},l}})\\
    &\quad -(1 + 2^{(1-\alpha)n})C_5 L\sum_{k=0}^3 E_{k,(1-\alpha)n}d^{(3-k)/2}
\end{align*}

We require $w_{\mathrm{clf},c}$ such that the RHS is $\geq 0$. This implies,
\begin{align}\label{eq:gauss_penul}
    &\ip{w^\star}{w_{\mathrm{clf},c}}(\ip{w_{\mathrm{clf},c}}{w_h} - \ip{w_{\mathrm{clf},c}}{w_{\mathrm{clf},l}}) \\
    &\geq\, 2^{-(1-\alpha)n}(\ip{w^\star}{w_{\mathrm{clf},l}}(\ip{w_h}{w_{\mathrm{clf},l}} - 1))  \nonumber\\
    &\quad + \frac{(1 - 2^{-(1-\alpha)n})E_{0,(1-\alpha)n}}{E_{0,(1-\alpha)n}-E_{2,(1-\alpha)n}}(\ip{w^\star}{w_h} - \ip{w^\star}{w_{\mathrm{clf},l}})\nonumber\\
    &\quad + \frac{(1 + 2^{-(1-\alpha)n})C_5 L}{E_{0,(1-\alpha)n} - E_{2,(1-\alpha)n}}C_5 L\sum_{k=0}^3 E_{k,(1-\alpha)n}d^{(3-k)/2} 
\end{align}

We compute the term $E_{0,b}$ and find upper bounds on $E_{a,b}$ for $a,b\in \mathbb{N}$ with $a$ being even to obtain a tight and interpretable upper bound on the RHS.
\begin{lemma}\label{lem:gauss_conf}
    If $E_{a,b} = \E_{\rho\sim \gN(0,1)}[ (1 - \Phi(\abs{\rho}))^b \abs{\rho}^a]$, then
    \begin{align*}
&\forall b\in \mathbb{N}, \quad E_{0,b} = 2^{-b}(b+1)^{-1}\\
    &\forall a,b \in \mathbb{N}, b\gg a, \quad E_{0,b} - E_{a,b} \geq \omega(1) E_{0,b}, 
    \end{align*}
\end{lemma}
\begin{proof}
We first compute $E_{0,b}$. We use $\phi(\rho)$ to denote the probability density function of a unit normal distribution at $\rho$. Note that $\Phi(\rho) = \int_{-\infty}^{\rho} \phi(a) da$.
\begin{align*}
    E_{0,b} &= \int_{-\infty}^\infty (1 -  \Phi(\abs{\rho}))^{b} \phi(\rho) d\rho
\end{align*}
As the term inside the integral is even in $\rho$, we restrict the integral from $\rho=0$ to $\infty$. If we set $t = 1-\Phi(\rho)$, then $dt = -\Phi'(\rho)d\rho = -\phi(\rho)d\rho$.
\begin{align*}
    E_{0,b} = -2\int_{1 - \Phi(0)}^{0}  t^{b} dt = 2\int_{0}^{\frac{1}{2}} t^b dt = \frac{2}{2^{b+1}(b+1)} = 2^{-b}(b+1)^{-1}
\end{align*}
By symmetry of the normal distribution, $\Phi(0) = \frac{1}{2}$.

We now compute $E_{a,b}$ for $a\in \mathbb{N}$.  
\begin{align*}
    E_{a,b} = \int_{-\infty}^\infty (1-\Phi(\abs{\rho}))^b \abs{\rho}^{a} \phi(\rho) d\rho = 2\int_{0}^\infty (1-\Phi(\rho))^b \rho^{a} \phi(\rho) d\rho
\end{align*}
As $1-\Phi(\abs{\rho})$ and $\rho^{2a}$ are both even, we can restrict the integral from $\rho=0$ to $\rho=\infty$. To bound this integral in terms of  $E_{0,b}$, we split the integration into two parts -- $[0,\delta]$ and $[\delta,\infty)$ for some $\delta\in[0,\infty)$. Ideally, this $\delta$ is very close to $0$.
Let $G_1 \triangleq 2\int_{0}^\delta(1 - \Phi(\rho))^b \rho^{a} \phi(\rho)d\rho$ and $G_2\triangleq = 2\int_\delta^{\infty}(1 - \Phi(\rho))^b \rho^{a} \phi(\rho)d\rho$. We need an upper bound on both these terms.

\paragraph{Bound on $G_1$}
\begin{align}
    G_1 &= 2\int_{0}^\delta(1 - \Phi(\rho))^b \rho^{a} \phi(\rho)d\rho \leq 2(1 - \Phi(0))^b \phi(\rho) \int_{0}^\delta \rho^{a} d\rho = 2^{-b+1/2} \delta^{a+1} (a+1)^{-1}\label{eq:g_1}
\end{align}
We use the fact that both $1-\Phi(\rho)$ and $\phi(\rho)$ are maximized at $\rho=0$.

\paragraph{Bound on $G_2$}
\begin{align}
    G_2 &= 2\int_{\delta}^{\infty}(1-\Phi(\rho))^b \rho^a \phi(\rho)d\rho \leq \frac{2}{(\sqrt{2\pi})^{a}}\int_{\delta}^{\infty}(1-\Phi(\rho))^{b-a} e^{-a\rho^2/2}  \phi(\rho)d\rho\nonumber\\ 
    &\leq \frac{2}{(\sqrt{2\pi})^{a}}\int_{\delta}^{\infty}(1-\Phi(\rho))^{b-a}  \phi(\rho)d\rho =  -\frac{2}{(\sqrt{2\pi})^{a}}\int_{1-\Phi(\delta)}^{0}t^{b-a}  dt = 2 (\sqrt{2\pi})^{-a} (1-\Phi(\delta))^{b-a}\label{eq:g_2}
\end{align}
For $(1-\Phi(\rho))^{a}$, we use the gaussian tail inequality, $1-\Phi(\rho) = \frac{\mathrm{erfc}(\rho/\sqrt{2})}{2} \leq \frac{e^{-\rho^2/2}}{\sqrt{2\pi}\rho}, \forall \rho >0$. Here $\mathrm{erfc}$ refers to the complementary error function. 
This is possible as $b \gg a$. Then, we bound $e^{-a\rho^2/2}\leq 1$ and set $t=1-\Phi(\rho)$ to integrate out the remaining components.

From ~\eqref{eq:g_1} and ~\eqref{eq:g_2}, we can find an appropriate $\delta$ such that for some constant $\alpha\in (0,1)$ and every value of $a,b$ 
\begin{align*}
    G_1 \leq 2^{-b+1/2} \delta^{a+1} (a+1)^{-1} \leq (1-\alpha)E_{0,b}/2,\; G_2\leq 2 (\sqrt{2\pi})^{-a} (1-\Phi(\delta))^{b-a} \leq (1-\alpha)E_{0,b}/2
\end{align*}
Plugging in the value of $E_{0,b}$, we need 
\begin{align*}
    \delta \leq \left(\frac{(1-\alpha)(a+1)}{2\sqrt{2}(b+1)}\right)^{1/(a+1)},\; \delta \geq \Phi^{-1}\left(1 -\frac{1}{2} \left(\frac{1-\alpha}{4}\left(\sqrt{\frac{\pi}{2}}\right)^{a}\right)^{1/(b-a)}\right)
\end{align*}
Note that we can always find $\alpha$ such that both these inequalities are simultaneously satisfied. Further, equating the two bounds on $\delta$ gives us a unique value of $\delta$ for fixed $\alpha$.

\end{proof}

From this Lemma, we can simplify the terms of $\frac{E_{0,(1-\alpha)n}}{E_{0,(\alpha)n} - E_{2,(\alpha)n}} \leq 1 + o(1)$ and the coefficients of terms of $L$ are bounded by $\mathcal{O}(1)$.
Plugging this into ~\eqref{eq:gauss_penul} proves Theorem~\ref{thm:gauss} with $g'(0)=1$. To work with $g'(0)\neq 1$, we replace $g$ by $\tilde{g}(a) = \frac{g(a)}{g'(0)},\forall a\in \R$ in the above proof. The smoothness constant for $\tilde{g}$ is now $\frac{L}{g'(0)}$, which gives us the correct statement of Theorem~\ref{thm:gauss}. 

\end{appendices}

\end{document}